\documentclass[twoside,11pt]{article}

\usepackage{amsmath, amssymb, amsfonts, amsthm, mathtools, mathrsfs, bbm} 
\usepackage{graphicx} 
\usepackage{hyperref}
\usepackage{geometry} 
\usepackage{setspace} 
\usepackage{caption} 
\usepackage[numbers, sort]{natbib}
\usepackage{lineno} 

\usepackage{lipsum}
\usepackage{titlesec}
\usepackage{fancyhdr} 
\usepackage{array}
\usepackage{xcolor}
\usepackage{enumitem}
\usepackage{bbm}
\usepackage{tikz} 
\usetikzlibrary{calc, arrows.meta, intersections, patterns, positioning, shapes.misc, fadings, through,decorations.pathreplacing, trees, shapes.geometric,math,arrows}
\usepackage{tikz-cd}
\usepackage{amscd}
\usepackage{pgfplots}
\pgfplotsset{compat=1.18}
\usepgfplotslibrary{fillbetween}
\usepackage{url}
\usepackage{aliascnt}
\usepackage{subcaption}

\newcommand{\supp}[1]{\mathrm{supp}\!\left(#1\right)}
\newcommand{\suppinline}[1]{\mathrm{supp}(#1)}

\newcommand{\spanR}[1]{\mathrm{span}_{\mathbb{R}}\!\left(#1\right)}
\newcommand{\spanRinline}[1]{\mathrm{span}_{\mathbb{R}}(#1)}

\newcommand{\dimR}[1]{\mathrm{dim}_{\mathbb{R}}\!\left(#1\right)}
\newcommand{\dimRinline}[1]{\mathrm{dim}_{\mathbb{R}}(#1)}
\newcommand{\sepcap}[1]{\mathcal{SC}\left(#1\right)}
\newcommand{\spark}[1]{\mathrm{spark}(#1)}
\newcommand{\lenin}[1]{\lvert #1 \rvert}

\newcommand{\kr}[1]{\mathrm{kr}(#1)}

\newcommand{\rank}[1]{\mathrm{rank}(#1)}
\newcommand{\sepcaps}[2]{\mathcal{SC}^{#1}\!\left(#2\right)}
\newcommand{\sepcapsinline}[2]{\mathcal{SC}^{#1}(#2)}
\newcommand{\innerprod}[2]{\left\langle #1, #2 \right\rangle}
\newcommand{\innerprodinline}[2]{\langle #1, #2 \rangle}
\newcommand{\Mod}[1]{\ (\mathrm{mod}\ #1)}
\newcommand\restr[2]{{
  \left.\kern-\nulldelimiterspace 
  #1 
  \vphantom{\big|} 
  \right|_{#2} 
  }}
\newcommand\restrinline[2]{{
  \,\kern-\nulldelimiterspace 
  #1 
  |_{#2} 
  }}

\newcommand{\tallphantom}{\vphantom{\dimR{\displaystyle\sum_{j}}}}
\newcommand{\longphantom}{\hphantom{\sum_{\substack{t=\bar r_{/1}+1 \\ (t-\bar r_{/1}) \text{ is odd}}}^{n-1} \lvert \mathcal E_{t}(\{H_1\cap H_k\}_{k=2}^n) \rvert \leq}}
\newtheorem{theorem}{Theorem}[section]

\newaliascnt{lemma}{theorem}
\newtheorem{lemma}[lemma]{Lemma}
\aliascntresetthe{lemma}

\newaliascnt{proposition}{theorem}
\newtheorem{proposition}[proposition]{Proposition}
\aliascntresetthe{proposition}

\newaliascnt{assumption}{theorem}
\newtheorem{assumption}[assumption]{Assumption}
\aliascntresetthe{assumption}

\newaliascnt{definition}{theorem}
\theoremstyle{definition}
\newtheorem{definition}[definition]{Definition}
\aliascntresetthe{definition}

\newaliascnt{remark}{theorem}
\theoremstyle{remark}
\newtheorem{remark}[remark]{Remark}
\aliascntresetthe{remark}

\numberwithin{equation}{section}
\counterwithin{figure}{section}

\titleformat{\section}[block]{\large\bfseries}{\thesection}{1em}{}
\titleformat{\subsection}[block]{\normalsize\bfseries}{\thesubsection}{1em}{}
\titleformat{\subsubsection}[block]{\normalsize\bfseries}{\thesubsubsection}{1em}{}

 \renewcommand{\paragraph}[1]{
     \textbf{#1.} 
 }

\fancypagestyle{plain}{
    \fancyhf{}

}

\newcommand{\mytitle}{Function-Counting Theory for Low-Dimensional Data Structures}
\title{\mytitle}
\author{
    Konstantin H\"aberle\\
    ETH Zurich\\
    \texttt{haeberlk@ethz.ch}
    \and
    Helmut B\"olcskei\\
    ETH Zurich\\
    \texttt{hboelcskei@ethz.ch}
}
\date{}

\usepackage[capitalize,nameinlink,compress]{cleveref}
\crefname{equation}{}{}
\crefname{subsection}{Subsection}{Subsections}
\crefname{remark}{Remark}{Remarks}
\crefname{assumption}{Assumption}{Assumptions}

\begin{document}

\maketitle

\begin{abstract}
The success of deep learning models in classification and regression is widely attributed to the low-dimensional structure that real-world data tend to exhibit, despite their high-dimensional representation.
This work attempts to provide a mathematical framework for binary classification on low-dimensional data, building on Cover's (1965) function-counting theory. With our framework, we aim to address the question of how the low-dimensional structure of the data affects the classification capabilities of learning models.
Cover's theory relies on a general position assumption that blinds it to the underlying data structure. We refine this assumption to account for the low-dimensionality of the data and derive dichotomy counts that reflect the data structure. We further extend Cover's separation capacity and problem of generalization to the low-dimensional setting, enabling the impact of the underlying data structure on both to be analyzed.
\\[1em]
\end{abstract}

\section{Introduction}
Function-counting theory, as initiated by Cover \cite{cover1965geometrical}, stands as a pivotal cornerstone in learning theory, providing a framework for the analysis of the classification capabilities of learning models such as neural networks. A core result in this theory is the so-called \emph{function-counting theorem} \cite{schlafli1950theorie,winder1961single,wendel1962problem,cover1965geometrical}. It quantifies the number of binary classification functions, or \emph{dichotomies}, that can be realized by a learning model when data points are assumed to be in general position (with respect to the learning model). 

A key motivation for extending this theory is that real-world data typically exhibits low-dimensional structure \cite{fefferman2016testing}. Existing function-counting results \cite{cover1965geometrical,kowalczyk1994separating,haberle2026scattering}, however, discard this structure entirely: Under the general position assumption, the number of realizable dichotomies depends only on the number of data points and the ambient feature dimension, leaving the underlying geometry of the dataset unconsidered.
Furthermore, the general position assumption becomes increasingly restrictive, and indeed usually fails, when data concentrates on low-dimensional structures.   

In this paper, we present an extension of the existing function-counting framework sensitive to the low-dimensional structure of datasets. The central question we aim to address with our framework is: How does the low-dimensional structure of the dataset affect the classification capabilities of learning models? We report three main contributions. 

\begin{enumerate}[label=(\roman*)]
    \item By refining the general position assumption to the given data structure, we derive function-counting results for a broad class of $s$-dimensional sets that includes sparse signals and rectifiable sets, with $s$ a positive integer much smaller than the ambient dimension. More precisely, instead of imposing general position on the dataset as a whole, we decompose it into components on which the linear spanning dimension is constant across all subsets of positive $s$-dimensional measure, and impose general position within each separately.
    Our dichotomy counts depend on the following quantities invisible to the classical function-counting theorem: the intrinsic dimension $s$, the geometry of the individual components, and their relative geometric configuration. For sparse signals, the constituent components are $s$-dimensional linear subspaces, and we show that the dichotomy count is governed by $s$ and independent of the ambient dimension.  
    For rectifiable sets, the components are in general nonlinear, and we establish that the number of realizable dichotomies reflects the geometric richness of the components arising from this nonlinearity: The dichotomy count is higher for sets spread across more directions.

    \item Based on the function-counting theorem, Cover introduced the notion of separation capacity \cite{kowalczyk1994separating,haberle2026scattering}, a quantity closely related to the Vapnik--Chervonenkis (VC) dimension \cite{vapnik1971uniform}. While VC dimension is an existential threshold on the cardinality of the dataset, separation capacity is a universal one, indicating where the majority of dichotomies becomes unrealizable for almost all datasets. Under the general position assumption, Cover established that the separation capacity is twice the ambient feature dimension.   
    We extend Cover's separation capacity to $s$-dimensional sets, and derive explicit characterizations thereof. Notably, by leveraging our dichotomy counts, we show that the separation capacity equals twice the smallest linear spanning dimension across the constituent components of the above decomposition. Our extension thus allows identifying the properties of the low-dimensional structure of the dataset that govern the separation capacity, sharpening Cover's theory, which captures only the ambient feature dimension. 

    \item Finally, we investigate Cover's problem of generalization for low-dimensional sets. While Cover's treatment relies on the general position assumption, and is thus blind to the underlying data structure, our framework allows us to analyze how the low-dimensional structure of the dataset influences a learning model's ability to generalize ambiguously. Here and throughout, generalization refers to whether a dichotomy realized by a learning model uniquely determines the label of a new point, or leaves both assignments compatible with the realized dichotomy. We derive an exact expression for the probability of ambiguous generalization and establish a connection to our extended notion of separation capacity.
\end{enumerate}

The paper is organized as follows. In \Cref{sec:fct-counting}, we derive function-counting bounds which hold true for arbitrary datasets and are fundamental for our extension of separation capacity. \Cref{sec:sep-struct-data} addresses the separation of points on low-dimensional datasets. In \Cref{sec:sep-cap}, we generalize the notion of separation capacity to encompass low-dimensional data structures. \Cref{sec:generalization-learning} is devoted to the problem of generalization. The notation used throughout this paper is summarized in \cref{app:notation}. \Cref{app:Cover,app:Hausdorff} review key results from function-counting theory and basic properties of the Hausdorff measure, respectively.    

\section{Function-counting bounds}\label{sec:fct-counting}
Let $E \subseteq \mathbb R^M$ be an arbitrary subset of the pattern space $(\mathbb R^M, \innerprod{\cdot}{\cdot})$ with the standard inner product $\innerprod{f}{g} = g^\mathsf{T}f$, $f,g \in \mathbb R^M$, $M\in \mathbb N$. In this setting, patterns correspond to raw input data, such as, e.g., images, audio signals, or videos, represented as elements of $\mathbb{R}^M$. The set $E$ may serve as a formal model for real-world data. It is often assumed that data lie on sets of low-dimensional structure. Intuitively, this assumption may be motivated by the following observations. High-dimensional real-world data typically exhibit some redundancy and are often correlated. Furthermore, especially in the context of classification, data of the same class often show invariance or equivariance with respect to certain transformations or deformations. For example, in the MNIST dataset \cite{lecun1998mnist}, data belonging to the same class are invariant under translations and small deformations, see \cref{fig:mnist}. A popular hypothesis for the low-dimensional structure of data is that of manifolds; concretely, in our notation, the assumption that $E$ is a submanifold of $\mathbb R^M$. This hypothesis is also referred to as the \emph{manifold hypothesis}. While a detailed technical treatment of testing model choices $E$ lies beyond the scope of this paper, for comprehensive discussions on the manifold hypothesis and related work, we refer the reader to \cite{narayanan2010sample,bengio2013representation,fefferman2016testing,carlsson2009topology}.

\begin{figure}
\centering
\begin{resizebox}{\columnwidth}{!}{
    \begin{tikzpicture}
        \node[anchor=south west, inner sep=0] at (0,0) {\includegraphics[scale=3]{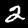}};
        \node[anchor=south west, inner sep=0] at (8,0) {\includegraphics[scale=3]{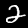}};
        \node[anchor=south west, inner sep=0] at (4,0) {\includegraphics[scale=3]{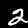}};
    \end{tikzpicture}}
\end{resizebox}
    \vspace{-5mm}
    \caption{The class `2' of the MNIST dataset \cite{lecun1998mnist} is invariant with respect to translations (middle) and small deformations (right).}
    \label{fig:mnist}
\end{figure}

Building upon Cover's framework \cite{cover1965geometrical} for quantifying the classification capabilities of a map $\Phi \colon E \to \mathbb R^{M'}$, we analyze the number of $\Phi$-separable dichotomies, $C_F$, of an arbitrary $N$-point set $F = \{f_1,\ldots,f_N\} \subseteq E$, where $N,M' \in \mathbb N$. 
A review of Cover's framework together with some key results of function-counting theory is provided in \cref{app:Cover}. 
   
This section develops novel lower and upper bounds on $C_F$ that generalize beyond the standard assumption of $\Phi$-general position. Such a generalization is essential for our setting, where the arbitrary nature of $E$, $F$, and $\Phi$ likely introduces degeneracies. These bounds are expressed in terms of the counting function $C(\cdot,\cdot)$, defined in \cref{eq:counting-fct-cover}, as well as the quantities\footnote{Recall that $\kr{\{\Phi(f_k)\}_{k=1}^N}$ denotes the Kruskal rank, i.e.,  the largest integer $s$ such that every subset of $s$ elements of $\{\Phi(f_k)\}_{k=1}^N$ is linearly independent.} $r \coloneqq \rank{\{\Phi(f_k)\}_{k=1}^N}$ and $s \coloneqq \kr{\{\Phi(f_k)\}_{k=1}^N}$, which measure the degeneracy of $E$, $F$, and $\Phi$.  

\begin{theorem}\label{thm:lower-upper-bound-C_F}
    The following holds: 
    \begin{enumerate}[label=(\alph*)]
        \item\label{it:C_F-a} If $s=0$, then $C_F=0$.
        \item\label{it:C_F-b} If $0<s=r$, then $C_F = C(N,s)$. 
        \item\label{it:C_F-c} If $0<s<r$, then 
                \begin{align}\label{eq:lower-upper-bound-C_F}
                    C(N,s) < C_F < C(N,r).
                \end{align} 
    \end{enumerate}
    
\end{theorem}
\begin{remark}
     If $s<r$, then $F$ is in particular not in $\Phi$-general position, and the upper bound reads
     $C_F < C(N,r) \leq C(N,M')$. That is, fewer than $C(N,M')$ dichotomies of $F$ are $\Phi$-separable, as expected.
\end{remark}
The proof of \cref{thm:lower-upper-bound-C_F} utilizes the next two lemmata. To this end, let us introduce the following notation. For each $k \in \{1,\ldots,N\}$, let $H_k \coloneqq\{\varphi_k\}^\perp$, where $\varphi_k \coloneqq \Phi(f_k)$. For $t \in \{0,\ldots,N\}$, denote by $\mathcal{E}_t$ and $\mathcal O_t$ the sets of all even- and odd-degenerate hyperplane arrangements consisting of $t$ hyperplanes, respectively. That is,
    \begin{align*}
        \mathcal{E}_t &\coloneqq \left\{\{H_k\}_{k \in \mathcal{K}} \colon \mathcal{K} \subseteq \{1,\ldots,N\} \text{, $\lvert \mathcal{K} \rvert = t$, and } \{H_k\}_{k \in \mathcal{K}} \text{ is even-degenerate}\right\}
        \intertext{and}
        \mathcal{O}_t &\coloneqq \left\{\{H_k\}_{k \in \mathcal{K}} \colon \mathcal{K} \subseteq \{1,\ldots,N\} \text{, $\lvert \mathcal{K} \rvert = t$, and } \{H_k\}_{k \in \mathcal{K}} \text{ is odd-degenerate}\right\}.
    \end{align*}
\begin{lemma}\label{L:C_F-lb}
    If $2\leq s<r$, then
    \begin{align}\label{eq:lower-bound-C_F-a}
        \sum_{\substack{t=s+2 \\ (t-s) \text{ is even}}}^N \lvert \mathcal O_t \rvert &< \sum_{\substack{t=s+1 \\ (t-s) \text{ is odd}}}^N \lvert \mathcal E_t \rvert. 
    \end{align}
\end{lemma}
\begin{proof}
    See \Cref{subsec:C_F-lb}.
\end{proof}

\begin{lemma}\label{L:C_F-ub}
    If $2\leq s<r$, then 
    \begin{align}\label{eq:ub-2}
        \sum_{\substack{t=r+1 \\ (t-r) \text{ is odd}}}^N \lvert \mathcal E_t \rvert < \sum_{t=s+1}^r \lvert \mathcal O_t \rvert + \sum_{\substack{t=r+2 \\ (t-r) \text{ is even}}}^N \lvert \mathcal O_t \rvert.
    \end{align}
\end{lemma}
\begin{proof}
    See \Cref{subsec:C_F-ub}.
\end{proof}

\begin{proof}[Proof of \cref{thm:lower-upper-bound-C_F}]
    \begin{enumerate}[label=(\alph*)]
        \item The claim is immediate by \cref{def:phi-sep}. 
        \item Let $\pi\colon \spanRinline{\Phi(F)} \to \mathbb R^s$ be the linear map to the space of expansion coefficients with respect to $\{\Phi(f_k)\}_{k=1}^s$. Then, $F$ is in $(\pi \circ \Phi)$-general position. Moreover, every dichotomy of $F$ is $\Phi$-separable if and only if it is $(\pi \circ \Phi)$-separable. Indeed, for every $w \in \mathbb R^{M'}$ and $k\in \{1,\ldots,N\}$, we have\footnote{Here, we write inner products with subscripts to indicate the space in which they are defined.}
    \begin{align*}
        \innerprod{w}{\Phi(f_k)}_{\mathbb R^{M'}} &= \innerprod{w'+w''}{\Phi(f_k)}_{\mathbb R^{M'}} \\
        &=\innerprod{w'}{\Phi(f_k)}_{\mathbb R^{M'}} \\
        &= \innerprod{\pi(w')}{\pi(\Phi(f_k))}_{\mathbb R^{s}}, 
    \end{align*}
    where $w'$ and $w''$ denote the orthogonal projections of $w$ onto $\spanRinline{\Phi(F)}$ and $\Phi(F)^\perp$, respectively. 
    Application of \cref{thm:fct-counting} thus yields 
    \begin{align}\label{eq:rs-equality}
        C_F = C(N,s), 
    \end{align}
    as desired.
    \item The proof proceeds in two steps. We first show that if \cref{eq:lower-upper-bound-C_F} holds for $2\leq s<r$, then \cref{eq:lower-upper-bound-C_F} necessarily extends to $s=1$. We then provide the proof for the case $2\leq s<r$.
    \begin{description}
        \item[\underline{Step (c.1): Extension to $s=1$.}] Suppose that \cref{eq:lower-upper-bound-C_F} holds for $2\leq s<r$. 
    Assume without loss of generality that $N' \in \{2,\ldots,N-1\}$ is such that $\{H_k\}_{k=1}^{N'}$ contains no duplicates (i.e., $s'\coloneqq\kr{\{\varphi_k\}_{k=1}^{N'}} \geq 2$) and such that $r'\coloneqq\rank{\{\varphi_k\}_{k=1}^{N'}} =r$. Denoting by $F'$ the corresponding $N'$-point set, we then have $C_F=C_{F'}$. As $2\leq s' \leq r'$, it holds by assumption that 
    \begin{align*}
        C(N',s') \leq C_{F'} \leq C(N',r'),
    \end{align*}
    but $C(N',r') < C(N,r)$ and $C(N',s') > C(N,s)=2$. This establishes the extension of \cref{eq:lower-upper-bound-C_F} from $2\leq s<r$ to $0<s<r$.
    \item[\underline{Step (c.2): The case $2\leq s<r$.}]
    We first show the lower bound in \cref{eq:lower-upper-bound-C_F} and then the upper bound in \cref{eq:lower-upper-bound-C_F}. 
    \begin{description}
        \item[\textit{Step (c.2.1): Lower bound.}] Note that $\varphi_k\neq 0$, $k\in \{1,\ldots,N\}$, as $s>0$ by assumption, so that application of \cref{thm:winder}, together with \cref{rem:number-even-odd-deg}, yields
    \begin{align}\label{eq:lu-b-C_F1}
        C_F = 2^N - 2 \lvert \mathcal O\rvert.
    \end{align}
    Using \cref{eq:fct-count-hyperplanes} and the identity $\sum_{t=0}^N \binom{N}{t} = 2^N$, we can write $C(N,s)$ as\footnote{We use the convention $\sum_{t \in \emptyset} \xi_t =0$ for any $\{\xi_t\}_t$.}
    \begin{align}\label{eq:lu-b-C_F2}
        C(N,s) = 2^N - 2\sum_{\substack{t=s+1 \\ (t-s) \text{ is odd}}}^N \binom{N}{t}.
    \end{align}
    Thus, it follows from \cref{eq:lu-b-C_F1,eq:lu-b-C_F2} that the lower bound in \cref{eq:lower-upper-bound-C_F} is equivalent to
    \begin{align}\label{eq:lu-b-C_F3}
        \lvert \mathcal O \rvert < \sum_{\substack{t=s+1 \\ (t-s) \text{ is odd}}}^N \binom{N}{t}.
    \end{align}
    We have $\mathcal E = \bigcup_{t=0}^N \mathcal E_t$ and $\mathcal O = \bigcup_{t=0}^N \mathcal O_t$. As every subset of $\{\Phi(f_k)\}_{k=1}^N$ of cardinality $s$ is linearly independent, we have
    \begin{align*}
        \dimR{\cap_{k \in \mathcal{K}} H_k} = M'-\dimR{\spanR{\{\varphi_k\}_{k \in \mathcal K}}} = M' - \lvert \mathcal K \rvert,
    \end{align*}
    whenever $\mathcal{K} \subseteq \{1,\ldots,N\}$ with $\lvert \mathcal K \rvert \leq t$.
    In particular, this implies $\lvert \mathcal{O}_t \rvert = 0$ for $t \leq s$, and hence $\lvert \mathcal O\rvert = \sum_{t=s+1}^N \lvert \mathcal O_t \rvert$.
    Using the identities $\lvert \mathcal{E}_t \rvert + \lvert \mathcal O_t \rvert = \binom{N}{t}$ and $\lvert \mathcal O\rvert = \sum_{t=s+1}^N \lvert \mathcal O_t \rvert$, one can deduce that \cref{eq:lu-b-C_F3} holds if and only if
    \begin{align*}
        \sum_{\substack{t=s+2 \\ (t-s) \text{ is even}}}^N \lvert \mathcal O_t \rvert &< \sum_{\substack{t=s+1 \\ (t-s) \text{ is odd}}}^N \lvert \mathcal E_t \rvert. 
    \end{align*}
    Application of \cref{L:C_F-lb} then completes the proof of Step (c.2.1). 
    \item[\textit{Step (c.2.2): Upper bound.}]
    As
    \begin{align*}
        C(N,r) = 2^N - 2\sum_{\substack{t=r+1 \\ (t-r) \text{ is odd}}}^N \binom{N}{t},
    \end{align*}
    and $C_F = 2^N-2\lvert \mathcal{O}\rvert$, the upper bound in \cref{eq:lower-upper-bound-C_F} is equivalent to
    \begin{align}\label{eq:ub-1}
         \sum_{\substack{t=r+1 \\ (t-r) \text{ is odd}}}^N \binom{N}{t} < \lvert \mathcal O \rvert.
    \end{align}
    With the identities $\lvert \mathcal E_t \rvert + \lvert \mathcal O_t \rvert = \binom{N}{t}$ and $\lvert \mathcal O\rvert = \sum_{t=s+1}^N \lvert \mathcal O_t \rvert$, \cref{eq:ub-1} reads
    \begin{align*}
        \sum_{\substack{t=r+1 \\ (t-r) \text{ is odd}}}^N \lvert \mathcal E_t \rvert < \sum_{t=s+1}^r \lvert \mathcal O_t \rvert + \sum_{\substack{t=r+2 \\ (t-r) \text{ is even}}}^N \lvert \mathcal O_t \rvert,
    \end{align*}
    \sloppy
    and the proof of Step (c.2.2) is complete upon application of \cref{L:C_F-ub}.    
    \end{description}
    \end{description}
    \end{enumerate}

\end{proof}

\subsection{Proof of \texorpdfstring{\cref{L:C_F-lb}}{Lemma}}\label{subsec:C_F-lb}
\begin{proof}
    To show \eqref{eq:lower-bound-C_F-a}, it will be convenient to establish at the same time the complementary inequality
    \begin{align}\label{eq:lower-bound-C_F-b}
        \sum_{\substack{t=s+1 \\ (t-s) \text{ is odd}}}^N \lvert \mathcal O_t \rvert &< \sum_{\substack{t=s \\ (t-s) \text{ is even}}}^N \lvert \mathcal E_t \rvert.
    \end{align}
    
    We shall prove \cref{eq:lower-bound-C_F-a,eq:lower-bound-C_F-b} simultaneously by induction on the number of hyperplanes $N$. The induction step employs a deletion--contraction argument, a Tutte--Grothendieck method standard in the study of hyperplane arrangements, see, e.g., \cite{zaslavsky1975facing}. 
    
    As $2\leq s<r$, we consider, for the base case, $N=4$, the hyperplanes $\{H_k\}_{k=1}^4$ with $H_k = \{\varphi_k\}^\perp$, where $\{H_k\}_{k=1}^3$ are in general position and $H_4$ is such that, e.g., $\varphi_4 =\varphi_1+\varphi_2$. In this setting, we have $s=2$, $r=3$, and
    \begin{gather*}
    \begin{array}{cccccc}
        \lvert \mathcal{E}_0 \rvert = 1, & 
        \lvert \mathcal{E}_1 \rvert = 4, &
        \lvert \mathcal{E}_2 \rvert = 6, &
        \lvert \mathcal{E}_3 \rvert = 3, &
        \lvert \mathcal{E}_4 \rvert = 0, \\[4pt]
        \lvert \mathcal{O}_0 \rvert = 0, &
        \lvert \mathcal{O}_1 \rvert = 0, &
        \lvert \mathcal{O}_2 \rvert = 0, &
        \lvert \mathcal{O}_3 \rvert = 1, &
        \lvert \mathcal{O}_4 \rvert = 1.
    \end{array}
    \end{gather*}
    Thus, \cref{eq:lower-bound-C_F-a} is realized as $\lvert \mathcal{O}_4 \rvert <\lvert \mathcal{E}_3 \rvert$, while \cref{eq:lower-bound-C_F-b} holds in the form $\lvert \mathcal{O}_3 \rvert \leq \lvert \mathcal{E}_2 \rvert + \lvert \mathcal{E}_4 \rvert$, as required. 

    Now suppose that \cref{eq:lower-bound-C_F-a,eq:lower-bound-C_F-b} is true for an arbitrary set of $N-1$ hyperplanes, denoted by $\{\tilde H_k\}_{k=1}^{N-1}$, i.e.,\footnote{To highlight the dependency of $\lvert \mathcal{E}_t \rvert$ and $\lvert \mathcal{O}_t \rvert$ on the set of hyperplanes under consideration, namely, $\{\tilde H_k\}_{k=1}^{N-1}$, we included it in parenthesis.}
    \begin{align}\label{eq:ind-hyp-lower-bound-a}
        \sum_{\substack{t=\tilde s+2 \\ (t-\tilde s) \text{ is even}}}^{N-1} \lvert \mathcal O_t(\{\tilde H_k\}_{k=1}^{N-1}) \rvert &< \sum_{\substack{t=\tilde s+1 \\ (t-\tilde s) \text{ is odd}}}^{N-1} \lvert \mathcal E_t(\{\tilde H_k\}_{k=1}^{N-1}) \rvert
    \end{align}
    and
    \begin{align}\label{eq:ind-hyp-lower-bound-b}
        \sum_{\substack{t=\tilde s+1 \\ (t-\tilde s) \text{ is odd}}}^{N-1} \lvert \mathcal O_t(\{\tilde H_k\}_{k=1}^{N-1}) \rvert &< \sum_{\substack{t=\tilde s \\ (t-\tilde s) \text{ is even}}}^{N-1} \lvert \mathcal E_t(\{\tilde H_k\}_{k=1}^{N-1}) \rvert,
    \end{align}
    where $\tilde s \coloneqq \kr{\{\tilde \varphi_k\}_{k=1}^{N-1}}$ and $\tilde r \coloneqq \rank{\{\tilde \varphi_k\}_{k=1}^{N-1}}$ satisfying $2\leq \tilde s <\tilde r$. Here, $\{\tilde \varphi_k\}_{k=1}^{N-1}$ are such that $\tilde H_k = \{\tilde\varphi_k\}^\perp$, for all $k \in \{1,\ldots,N-1\}$.
    To establish the induction step, we require the identities
    \begin{align}
        \lvert \mathcal E_t(\{H_k\}_{k=1}^N)\rvert = \lvert \mathcal E_t(\{H_k\}_{k=2}^N)\rvert + \lvert \mathcal E_{t-1}(\{H_1 \cap H_k\}_{k=2}^N)\rvert \label{eq:del-contr-even}
        \intertext{and}
        \lvert \mathcal O_t(\{H_k\}_{k=1}^N)\rvert = \lvert \mathcal O_t(\{H_k\}_{k=2}^N)\rvert + \lvert \mathcal O_{t-1}(\{H_1 \cap H_k\}_{k=2}^N)\rvert. \label{eq:del-contr-odd}
    \end{align}
    The \emph{deletion} $\{H_k\}_{k=2}^N$ describes the arrangement of all hyperplanes except for $H_1$ in the same ambient space $\mathbb R^{M'}$. The \emph{contraction} $\{H_1 \cap H_k\}_{k=2}^N$ refers to the hyperplane arrangement in the new $(M'-1)$-dimensional ambient space $H_1$. Note that, as $s\geq 2$ (i.e., $\{H_k\}_{k=1}^N$ are distinct), $\{H_1 \cap H_k\}_{k=2}^N$ are indeed $(M'-2)$-dimensional hyperplanes in $H_1$. Denoting by $P_{H_1}$ the orthogonal projection onto $H_1$, we can write\footnote{Here, $\perp$ denotes the orthogonal complement in the ambient space $H_1$.} $H_1 \cap H_k = \{P_{H_1}\varphi_k\}^\perp$. To see that \cref{eq:del-contr-even} holds, let $\{H_k\}_{k \in \mathcal K} \in \mathcal E_t(\{H_k\}_{k=1}^N)$ be even-degenerate, where $\mathcal{K} \subseteq \{1,\ldots,N\}$ with $\lvert \mathcal{K} \rvert =t$. If $1 \notin \mathcal{K}$, then $\{H_k\}_{k \in \mathcal K} \in \mathcal E_t(\{H_k\}_{k=2}^N)$. If $1 \in \mathcal{K}$, then $\{H_1 \cap H_k\}_{k \in \mathcal{K}\setminus\{1\}}$ is even-degenerate in the space $H_1$. Indeed,
    \begin{align*}
        \dimR{\bigcap_{k \in \mathcal{K}\setminus\{1\}}(H_1\cap H_k)} &= \dimR{\bigcap_{k \in \mathcal{K}}H_k} \\
        &\equiv M'-t &&\Mod{2} \\
        &\equiv (M'-1)-(t-1) &&\Mod{2},
    \end{align*}
    where we used in the second line that $\{H_k\}_{k \in \mathcal{K}}$ is even-degenerate in $\mathbb R^{M'}$.
    Thus, \cref{eq:del-contr-even} holds with ``$\leq$''. Conversely, if $\{H_k\}_{k \in \mathcal K} \in  \mathcal E_t(\{H_k\}_{k=2}^N)$, where $\mathcal{K} \subseteq \{2,\ldots,N\}$ with $\lvert \mathcal{K} \rvert =t$, then $\{H_k\}_{k \in \mathcal K} \in  \mathcal E_t(\{H_k\}_{k=1}^N)$. Furthermore, if $\{H_1\cap H_k\}_{k \in \mathcal K} \in  \mathcal E_{t-1}(\{H_1 \cap H_k\}_{k=2}^N)$, where $\mathcal{K} \subseteq \{2,\ldots,N\}$ with $\lvert \mathcal{K} \rvert =t-1$, then $\{H_k\}_{k \in \mathcal{K} \cup \{1\}} \in \mathcal E_t(\{H_k\}_{k=1}^N)$ because
    \begin{align*}
        \dimR{\bigcap_{k \in \mathcal{K}\cup\{1\}}H_k} &= \dimR{\bigcap_{k \in \mathcal{K}}(H_1\cap H_k)} \\
        &\equiv (M'-1)-(t-1) &&\Mod{2} \\
        &\equiv M'-t &&\Mod{2},
    \end{align*}
    where the second line holds because $\{H_1\cap H_k\}_{k \in \mathcal K}$ is even-degenerate in the $(M'-1)$-dimensional space $H_1$. 
    Consequently, \cref{eq:del-contr-even} remains valid with ``$\geq$'', and the identity follows. 
    Likewise, one can show that \cref{eq:del-contr-odd} is true. 
    
    By definition of the Kruskal rank, there exists a $\mathcal{K}\subseteq \{1,\ldots,N\}$ with $\lvert \mathcal{K} \rvert = s+1$ such that $\{\varphi_k\}_{k \in \mathcal{K}}$ is linearly dependent. Without loss of generality assume that $1 \in \mathcal{K}$. Note that, upon deleting $H_1$ from $\{H_k\}_{k=1}^N$ and contracting $\{H_k\}_{k=1}^N$ onto $H_1$, in the resulting hyperplane arrangements $\{H_k\}_{k=2}^N$ and $\{H_1 \cap H_k\}_{k=2}^N$, the corresponding quantities\footnote{The subscripts $\setminus 1$ and $/1$ refer to deletion and contraction of $H_1$, respectively.} $r_{\setminus1} \coloneqq \rank{\{\varphi_k\}_{k =2}^N}$, $s_{\setminus 1} \coloneqq\kr{\{\varphi_k\}_{k =2}^N}$ and $r_{/1} \coloneqq \rank{\{P_{H_1}\varphi_k\}_{k =2}^N}$, $s_{/1} \coloneqq\kr{\{P_{H_1}\varphi_k\}_{k =2}^N}$, respectively, may, in general, change. Specifically, as $\{\varphi_k\}_{k \in \mathcal{K}}$ is linearly dependent and $1 \in \mathcal{K}$, it is immediate that $r_{\setminus1} =r $ and moreover, since $r<s$, we have $s_{\setminus1}\geq s$. Now by the rank--nullity theorem, we have for every $\mathcal{K}^{/1} \subseteq \{2,\ldots,N\}$,
    \begin{align*}
        &\!\!\dimR{\spanR{\{P_{H_1}\varphi_k\}_{k \in \mathcal{K}^{/1}}}} \\ &= \dimR{\spanR{\{\varphi_k\}_{k \in \mathcal{K}^{/1}}}} - \dimR{H_1^{\perp} \cap \spanR{\{\varphi_k\}_{k \in \mathcal{K}^{/1}}}} \\
        &= \dimR{\spanR{\{\varphi_k\}_{k \in \mathcal{K}^{/1}}}} - \dimR{\spanR{\{\varphi_1\}} \cap \spanR{\{\varphi_k\}_{k \in \mathcal{K}^{/1}}}} \\ 
        &= \dimR{\spanR{\{\varphi_k\}_{k \in \mathcal{K}^{/1} \cup \{1\}}}} -1. 
    \end{align*}
    Setting $\mathcal{K}^{/1} = \{2,\ldots,N\}$ and $\mathcal{K}^{/1} = \mathcal{K} \setminus \{1\}$, we can deduce that $r_{/1} = r-1$ and $s_{/1} =s-1$, respectively.   
    
    To complete the induction step, we consider two cases $s \equiv s_{\setminus1} \Mod{2}$ and $s \not\equiv s_{\setminus1} \Mod{2}$. 

    \underline{Case 1: $s \equiv s_{\setminus1} \Mod{2}$.} We compute
    \allowdisplaybreaks
    \begin{align}
        \sum_{\substack{t=s+2 \\ (t-s) \text{ is even}}}^N \lvert \mathcal O_t(\{H_k\}_{k=1}^N) \rvert
        &= \sum_{\substack{t=s+2 \\ (t-s) \text{ is even}}}^N \lvert \mathcal O_t(\{H_k\}_{k=2}^N) \rvert + \sum_{\substack{t=s+2 \\ (t-s) \text{ is even}}}^N \lvert \mathcal O_{t-1}(\{H_1 \cap H_k\}_{k=2}^N) \rvert \label{eq:lower-bound-induction-step1}\\
        &= \sum_{\substack{t=s+2 \\ (t-s) \text{ is even}}}^{N-1} \lvert \mathcal O_t(\{H_k\}_{k=2}^N) \rvert + \sum_{\substack{t=(s-1)+2 \\ (t-(s-1)) \text{ is even}}}^{N-1} \lvert \mathcal O_{t}(\{H_1 \cap H_k\}_{k=2}^N) \rvert \nonumber\\
        &= \sum_{\substack{t=s_{\setminus 1}+2 \\ (t-s_{\setminus 1}) \text{ is even}}}^{N-1} \lvert \mathcal O_t(\{H_k\}_{k=2}^N) \rvert + \sum_{\substack{t=s_{/1}+2 \\ (t-s_{/1}) \text{ is even}}}^{N-1} \lvert \mathcal O_{t}(\{H_1 \cap H_k\}_{k=2}^N) \rvert \label{eq:lower-bound-induction-step1.1}\\
        &\leq \sum_{\substack{t=s_{\setminus 1}+1 \\ (t-s_{\setminus 1}) \text{ is odd}}}^{N-1} \lvert \mathcal E_t(\{H_k\}_{k=2}^N) \rvert + \sum_{\substack{t=s_{/1}+1 \\ (t-s_{/1}) \text{ is odd}}}^{N-1} \lvert \mathcal E_t(\{H_1 \cap H_k\}_{k=2}^N) \rvert \label{eq:lower-bound-induction-step2}\\
        &= \sum_{\substack{t=s_{\setminus 1}+1 \\ (t-s_{\setminus 1}) \text{ is odd}}}^{N-1} \lvert \mathcal E_t(\{H_k\}_{k=2}^N) \rvert + \sum_{\substack{t=s+1 \\ (t-s) \text{ is odd}}}^{N-1} \lvert \mathcal E_{t-1}(\{H_1 \cap H_k\}_{k=2}^N) \rvert \nonumber\\
        &\leq \sum_{\substack{t=s+1 \\ (t-s) \text{ is odd}}}^{N} \lvert \mathcal E_t(\{H_k\}_{k=2}^N) \rvert + \sum_{\substack{t=s+1 \\ (t-s) \text{ is odd}}}^{N} \lvert \mathcal E_{t-1}(\{H_1 \cap H_k\}_{k=2}^N) \rvert \label{eq:lower-bound-induction-step2.1}\\
        &= \sum_{\substack{t=s+1 \\ (t-s) \text{ is odd}}}^N \lvert \mathcal E_t (\{H_k\}_{k=1}^N) \rvert, \label{eq:lower-bound-induction-step3}
    \end{align}
    where \cref{eq:lower-bound-induction-step1} is by \cref{eq:del-contr-odd}. In \cref{eq:lower-bound-induction-step1.1}, we used that $\lvert \mathcal O_t(\{H_k\}_{k=2}^N) \rvert = 0$ for $t\leq s_{\setminus1}$ and $s \equiv s_{\setminus 1} \Mod{2}$. The first term in \cref{eq:lower-bound-induction-step2} follows from the induction hypothesis \cref{eq:ind-hyp-lower-bound-a} applied with $\{\tilde H_k\}_{k=1}^{N-1} = \{H_k\}_{k=2}^{N-1}$, $\tilde s = s_{\setminus 1}$, and $\tilde r = r_{\setminus 1}$ whenever $s_{\setminus 1}<r_{\setminus 1}$. If $s_{\setminus 1}=r_{\setminus 1}$, we know from \cref{eq:rs-equality} that  
    \begin{align}\label{eq:rs-equality-a}
        \sum_{\substack{t=s_{\setminus 1}+2 \\ (t-s_{\setminus 1}) \text{ is even}}}^{N-1} \lvert \mathcal O_t(\{H_k\}_{k=2}^N) \rvert = \sum_{\substack{t=s_{\setminus 1}+1 \\ (t-s_{\setminus 1}) \text{ is odd}}}^{N-1} \lvert \mathcal E_t(\{H_k\}_{k=2}^N) \rvert.
    \end{align}
    Now, for the second term in \cref{eq:lower-bound-induction-step2} the induction hypothesis \cref{eq:ind-hyp-lower-bound-a} is employed with $\{\tilde H_k\}_{k=1}^{N-1} = \{H_1 \cap H_k\}_{k=2}^{N-1}$, $\tilde s = s_{/ 1} = s-1$, and $\tilde r = r_{/1}=r-1$ whenever $s_{/1} \geq 2$. For the case $s_{/1} = 1$, i.e., $\{H_1 \cap H_k\}_{k=2}^N$ are not distinct, we argue as follows. Suppose that $\{H_1 \cap H_k\}_{k=2}^N$ are ordered such that $\{H_1 \cap H_k\}_{k=2}^n$ are distinct and each of $\{H_1 \cap H_k\}_{k=n+1}^N$ is identical to one of the first $(n-1)$, for some $n \in \{2,\ldots,N-1\}$. Then, by assumption, for $H_1 \cap H_{n+1}$, there is an $\ell \in \{2,\ldots,n\}$ such that $H_1 \cap H_{n+1} = H_1 \cap H_{\ell}$. Considering subsets of $\{H_1 \cap H_k\}_{k=2}^N$ containing $H_1 \cap H_{n+1}$ but none of $\{H_1\cap H_{n+2}\}_{k=n+2}^N$, for each such subset containing $H_1 \cap H_\ell$ there is one identical lacking $H_1 \cap H_\ell$. Then, if one subset of such a pair belongs to $\bigcup^{N-1}_{t=s_{/1}+2, \, (t-s_{/1}) \text{ even}} \mathcal{O}_{t}(\{H_1 \cap H_k\}_{k=2}^N)$, the other one of this pair is in $\bigcup^{N-1}_{t=s_{/1}+1, \, (t-s_{/1}) \text{ odd}} \mathcal{E}_{t}(\{H_1 \cap H_k\}_{k=2}^N)$. Using this cancellation and repeating this argument with each of $\{H_1 \cap H_{k}\}_{k=n+2}^N$, we arrive at the equivalence
    \begin{align}\label{eq:removing-duplicates}
    \begin{split}
        &\left(\sum_{\substack{t=s_{/1}+2 \\ (t-s_{/1}) \text{ is even}}}^{N-1} \lvert \mathcal O_{t}(\{H_1 \cap H_k\}_{k=2}^N) \leq \sum_{\substack{t=s_{/1}+1 \\ (t-s_{/1}) \text{ is odd}}}^{N-1} \lvert \mathcal E_{t-1}(\{H_1 \cap H_k\}_{k=2}^N) \rvert\right) \\\iff &\left(\sum_{\substack{t=\bar s_{/1}+2 \\ (t-\bar s_{/1}) \text{ is even}}}^{n-1} \lvert \mathcal O_{t}(\{H_1 \cap H_k\}_{k=2}^n) \leq \sum_{\substack{t=\bar s_{/1}+1 \\ (t-\bar s_{/1}) \text{ is odd}}}^{n-1} \lvert \mathcal E_{t-1}(\{H_1 \cap H_k\}_{k=2}^n) \rvert\right), 
    \end{split}
    \end{align}
    where $\bar s_{/1} \coloneqq \kr{\{P_{H_1}\varphi_k\}_{k=2}^n} \geq 2$. Note that $\bar r_{/1} \coloneqq \rank{\{P_{H_1}\varphi_k\}_{k=2}^n} = r_{/1}$. Now to establish the second term in \cref{eq:lower-bound-induction-step2}, we apply the induction hypothesis \cref{eq:ind-hyp-lower-bound-a} with $\{\tilde H_k\}_{k=1}^{N-1} = \{H_1 \cap H_k\}_{k=2}^{n}$, $\tilde s = \bar s_{/ 1}$, and $\tilde r = \bar r_{/1}$ if $\bar s_{/1} <\bar r_{/1}$; and otherwise, if $\bar s_{/1} =\bar r_{/1}$, we apply \cref{eq:rs-equality-a}.      
    Finally, \cref{eq:lower-bound-induction-step2.1} holds as $s_{\setminus 1} \geq s$ and \cref{eq:lower-bound-induction-step3} follows from \cref{eq:del-contr-even}. Upon noting that the inequality in \cref{eq:lower-bound-induction-step2} is strict if $s_{\setminus 1} =s$, and that the inequality in \cref{eq:lower-bound-induction-step2.1} is strict if $s_{\setminus 1} >s$, we obtain
    \begin{align*}
        \sum_{\substack{t=s+2 \\ (t-s) \text{ is even}}}^N \lvert \mathcal O_t(\{H_k\}_{k=1}^N) \rvert < \sum_{\substack{t=s+1 \\ (t-s) \text{ is odd}}}^N \lvert \mathcal E_t(\{H_k\}_{k=1}^N) \rvert.
    \end{align*}
    
    Similarly, we have
    \begin{align}
        \sum_{\substack{t=s+1 \\ (t-s) \text{ is odd}}}^N \lvert \mathcal O_t(\{H_k\}_{k=1}^N) \rvert 
        &= \sum_{\substack{t=s+1 \\ (t-s) \text{ is odd}}}^N \lvert \mathcal O_t(\{H_k\}_{k=2}^N) \rvert + \sum_{\substack{t=s+1 \\ (t-s) \text{ is odd}}}^N \lvert \mathcal O_{t-1}(\{H_1 \cap H_k\}_{k=2}^N) \rvert \label{eq:lower-bound-induction-b-step1}\\
        &= \sum_{\substack{t=s+1 \\ (t-s) \text{ is odd}}}^{N-1} \lvert \mathcal O_t(\{H_k\}_{k=2}^N) \rvert + \sum_{\substack{t=(s-1)+1 \\ (t-(s-1)) \text{ is odd}}}^{N-1} \lvert \mathcal O_{t}(\{H_1 \cap H_k\}_{k=2}^N) \rvert \nonumber\\
        &= \sum_{\substack{t=s_{\setminus 1}+1 \\ (t-s_{\setminus 1}) \text{ is odd}}}^{N-1} \lvert \mathcal O_t(\{H_k\}_{k=2}^N) \rvert + \sum_{\substack{t=s_{/1}+1 \\ (t-s_{/1}) \text{ is odd}}}^{N-1} \lvert \mathcal O_{t}(\{H_1 \cap H_k\}_{k=2}^N) \rvert \label{eq:lower-bound-induction-b-step1.1}\\
        &< \sum_{\substack{t=s_{\setminus 1} \\ (t-s_{\setminus 1}) \text{ is even}}}^{N-1} \lvert \mathcal E_t(\{H_k\}_{k=2}^N) \rvert + \sum_{\substack{t=s_{/1} \\ (t-s_{/1}) \text{ is even}}}^{N-1} \lvert \mathcal E_t(\{H_1 \cap H_k\}_{k=2}^N) \rvert \label{eq:lower-bound-induction-b-step2}\\
        &= \sum_{\substack{t=s_{\setminus 1} \\ (t-s_{\setminus 1}) \text{ is even}}}^{N-1} \lvert \mathcal E_t(\{H_k\}_{k=2}^N) \rvert + \sum_{\substack{t=s \\ (t-s) \text{ is even}}}^{N} \lvert \mathcal E_{t-1}(\{H_1 \cap H_k\}_{k=2}^N) \rvert \nonumber\\
        &\leq \sum_{\substack{t=s \\ (t-s) \text{ is even}}}^{N} \lvert \mathcal E_t(\{H_k\}_{k=2}^N) \rvert + \sum_{\substack{t=s \\ (t-s) \text{ is even}}}^{N} \lvert \mathcal E_{t-1}(\{H_1 \cap H_k\}_{k=2}^N) \rvert \label{eq:lower-bound-induction-b-step2.1}\\
        &= \sum_{\substack{t=s \\ (t-s) \text{ is even}}}^N \lvert \mathcal E_t (\{H_k\}_{k=1}^N) \rvert, \label{eq:lower-bound-induction-b-step3}
    \end{align}
    \sloppy where, as in the preceding derivation, in \cref{eq:lower-bound-induction-b-step1} we used \cref{eq:del-contr-odd}, and \cref{eq:lower-bound-induction-b-step1.1} holds as $\lvert \mathcal O_t(\{H_k\}_{k=2}^N) \rvert = 0$ for $t\leq s_{\setminus1}$, upon noting that $s \equiv s_{\setminus 1} \Mod{2}$. Whenever $s_{\setminus1}<r_{\setminus 1}$, the first term in \cref{eq:lower-bound-induction-b-step2} follows from the induction hypothesis \cref{eq:ind-hyp-lower-bound-b} with $\{\tilde H_k\}_{k=1}^{N-1} = \{H_k\}_{k=2}^{N-1}$, $\tilde s = s_{\setminus 1}$, and $\tilde r =r_{\setminus1}$. If $s_{\setminus1}=r_{\setminus 1}$, we have, by \cref{eq:rs-equality} and the fact that $C(N-1,s_{\setminus 1}-1) < C(N-1,s_{\setminus 1})$ whenever $N-1 \geq s_{\setminus1}$,
    \begin{align}
        \left(\sum_{\substack{t=s_{\setminus 1}+1 \\ \phantom{ph_1}}}^{N-1} \lvert \mathcal O_t(\{H_k\}_{k=2}^N) \rvert\right. &< \left.\sum_{\substack{t=s_{\setminus 1} \\ (t-(s_{\setminus 1}-1)) \text{ is odd}}}^{N-1} \binom{N-1}{t} \right) \nonumber\\
        \iff \left(\sum_{\substack{t=s_{\setminus 1}+1 \\ (t-s_{\setminus 1}) \text{ is odd}}}^{N-1} \lvert \mathcal O_t(\{H_k\}_{k=2}^N) \rvert \right.&<\left. \sum_{\substack{t=s_{\setminus 1} \\ (t-s_{\setminus 1}) \text{ is even}}}^{N-1} \lvert \mathcal E_t(\{H_k\}_{k=2}^N) \rvert \right). \label{eq:rs-equality-b}
    \end{align}
    Note that $N-1 \geq s_{\setminus1}$ is satisfied whenever $s_{\setminus1}=r_{\setminus 1} = r$ as $s<r$.  
    The second term in \cref{eq:lower-bound-induction-b-step2} is obtained by applying the induction hypothesis \cref{eq:ind-hyp-lower-bound-b} with $\{\tilde H_k\}_{k=1}^{N-1} = \{H_1 \cap H_k\}_{k=2}^{N-1}$, $\tilde s = s_{/ 1}$, and $\tilde r = r_{/1}$ whenever $s_{/1} \geq 2$. If $s_{/1} = 1$, we employ the analogous cancellation argument as in \cref{eq:removing-duplicates} with \cref{eq:rs-equality-b} in case $\bar s_{/1} = \bar{r}_{/1}$. Finally, \cref{eq:lower-bound-induction-b-step2.1} follows as $s_{\setminus 1} \geq s$, and \cref{eq:lower-bound-induction-b-step3} is a consequence of \cref{eq:del-contr-even}.
    
    \underline{Case 2: $s \not\equiv s_{\setminus1} \Mod{2}$.} It holds that
    \begin{align}
        \sum_{\substack{t=s+2 \\ (t-s) \text{ is even}}}^N \lvert \mathcal O_t(\{H_k\}_{k=1}^N) \rvert 
        &= \sum_{\substack{t=s+2 \\ (t-s) \text{ is even}}}^N \lvert \mathcal O_t(\{H_k\}_{k=2}^N) \rvert + \sum_{\substack{t=s+2 \\ (t-s) \text{ is even}}}^N \lvert \mathcal O_{t-1}(\{H_1 \cap H_k\}_{k=2}^N) \rvert \nonumber\\
        &= \sum_{\substack{t=s+2 \\ (t-s) \text{ is even}}}^{N-1} \lvert \mathcal O_t(\{H_k\}_{k=2}^N) \rvert + \sum_{\substack{t=(s-1)+2 \\ (t-(s-1)) \text{ is even}}}^{N-1} \lvert \mathcal O_{t}(\{H_1 \cap H_k\}_{k=2}^N) \rvert \nonumber\\
        &= \sum_{\substack{t=s_{\setminus 1}+1 \\ (t-s_{\setminus 1}) \text{ is odd}}}^{N-1} \lvert \mathcal O_t(\{H_k\}_{k=2}^N) \rvert + \sum_{\substack{t=s_{/1}+2 \\ (t-s_{/1}) \text{ is even}}}^{N-1} \lvert \mathcal O_{t}(\{H_1 \cap H_k\}_{k=2}^N) \rvert \label{eq:lower-bound-induction-step1.1-diff-parity}\\
        &< \sum_{\substack{t=s_{\setminus 1} \\ (t-s_{\setminus 1}) \text{ is even}}}^{N-1} \lvert \mathcal E_t(\{H_k\}_{k=2}^N) \rvert + \sum_{\substack{t=s_{/1}+1 \\ (t-s_{/1}) \text{ is odd}}}^{N-1} \lvert \mathcal E_t(\{H_1 \cap H_k\}_{k=2}^N) \rvert \label{eq:lower-bound-induction-step2-diff-parity}\\
        &= \sum_{\substack{t=s_{\setminus 1} \\ (t-s_{\setminus 1}) \text{ is even}}}^{N-1} \lvert \mathcal E_t(\{H_k\}_{k=2}^N) \rvert + \sum_{\substack{t=s+1 \\ (t-s) \text{ is odd}}}^{N-1} \lvert \mathcal E_{t-1}(\{H_1 \cap H_k\}_{k=2}^N) \rvert \nonumber\\
        &\leq \sum_{\substack{t=s+1 \\ (t-s) \text{ is odd}}}^{N} \lvert \mathcal E_t(\{H_k\}_{k=2}^N) \rvert + \sum_{\substack{t=s+1 \\ (t-s) \text{ is odd}}}^{N} \lvert \mathcal E_{t-1}(\{H_1 \cap H_k\}_{k=2}^N) \rvert \label{eq:lower-bound-induction-step2.1-diff-parity}\\
        &= \sum_{\substack{t=s+1 \\ (t-s) \text{ is odd}}}^N \lvert \mathcal E_t(\{H_k\}_{k=1}^N) \rvert, \nonumber
    \end{align}
    where \cref{eq:lower-bound-induction-step1.1-diff-parity} follows as $\lvert \mathcal O_t(\{H_k\}_{k=2}^N) \rvert = 0$ for $t\leq s_{\setminus1}$ and $s \not\equiv s_{\setminus 1} \Mod{2}$. For the first term in \cref{eq:lower-bound-induction-step2-diff-parity}, we apply the induction hypothesis \cref{eq:ind-hyp-lower-bound-b} with $\{\tilde H_k\}_{k=1}^{N-1} = \{H_k\}_{k=2}^{N-1}$, $\tilde s = s_{\setminus 1}$, and $\tilde r = r_{\setminus 1}$ whenever $s_{\setminus 1}<r_{\setminus 1}$, and use \cref{eq:rs-equality-b} otherwise. The second term in \cref{eq:lower-bound-induction-step2-diff-parity} is by the induction hypothesis \cref{eq:ind-hyp-lower-bound-a} with $\{\tilde H_k\}_{k=1}^{N-1} = \{H_1 \cap H_k\}_{k=2}^{N-1}$, $\tilde s = s_{/ 1}$, and $\tilde r = r_{/1}$ if $s_{/1}\geq 2$. For $s_{/1}=1$, the same cancellation argument as in \cref{eq:removing-duplicates} is employed. In \cref{eq:lower-bound-induction-step2.1-diff-parity} we used that $s_{\setminus 1} \geq s$ and $s \not\equiv s_{\setminus 1} \Mod{2}$.

    Finally, we compute
    \begin{align}
        \sum_{\substack{t=s+1 \\ (t-s) \text{ is odd}}}^N \lvert \mathcal O_t(\{H_k\}_{k=1}^N) \rvert
        &= \sum_{\substack{t=s+1 \\ (t-s) \text{ is odd}}}^N \lvert \mathcal O_t(\{H_k\}_{k=2}^N) \rvert + \sum_{\substack{t=s+1 \\ (t-s) \text{ is odd}}}^N \lvert \mathcal O_{t-1}(\{H_1 \cap H_k\}_{k=2}^N) \rvert \nonumber\\
        &= \sum_{\substack{t=s+1 \\ (t-s) \text{ is odd}}}^{N-1} \lvert \mathcal O_t(\{H_k\}_{k=2}^N) \rvert + \sum_{\substack{t=(s-1)+1 \\ (t-(s-1)) \text{ is odd}}}^{N-1} \lvert \mathcal O_{t}(\{H_1 \cap H_k\}_{k=2}^N) \rvert \nonumber\\
        &= \sum_{\substack{t=s_{\setminus 1}+2 \\ (t-s_{\setminus 1}) \text{ is even}}}^{N-1} \lvert \mathcal O_t(\{H_k\}_{k=2}^N) \rvert + \sum_{\substack{t=s_{/1}+1 \\ (t-s_{/1}) \text{ is odd}}}^{N-1} \lvert \mathcal O_{t}(\{H_1 \cap H_k\}_{k=2}^N) \rvert \label{eq:lower-bound-induction-b-step1.1-diff-parity}\\
        &\leq \sum_{\substack{t=s_{\setminus 1}+1 \\ (t-s_{\setminus 1}) \text{ is odd}}}^{N-1} \lvert \mathcal E_t(\{H_k\}_{k=2}^N) \rvert + \sum_{\substack{t=s_{/1} \\ (t-s_{/1}) \text{ is even}}}^{N-1} \lvert \mathcal E_t(\{H_1 \cap H_k\}_{k=2}^N) \rvert \label{eq:lower-bound-induction-b-step2-diff-parity}\\
        &= \sum_{\substack{t=s_{\setminus 1}+1 \\ (t-s_{\setminus 1}) \text{ is odd}}}^{N-1} \lvert \mathcal E_t(\{H_k\}_{k=2}^N) \rvert + \sum_{\substack{t=s \\ (t-s) \text{ is even}}}^{N} \lvert \mathcal E_{t-1}(\{H_1 \cap H_k\}_{k=2}^N) \rvert \nonumber\\
        &< \sum_{\substack{t=s \\ (t-s) \text{ is even}}}^{N} \lvert \mathcal E_t(\{H_k\}_{k=2}^N) \rvert + \sum_{\substack{t=s \\ (t-s) \text{ is even}}}^{N} \lvert \mathcal E_{t-1}(\{H_1 \cap H_k\}_{k=2}^N) \rvert \label{eq:lower-bound-induction-b-step2.1-diff-parity}\\
        &= \sum_{\substack{t=s \\ (t-s) \text{ is even}}}^N \lvert \mathcal E_t(\{H_k\}_{k=1}^N) \rvert, \nonumber
    \end{align}
    where in \cref{eq:lower-bound-induction-b-step1.1-diff-parity}, we used that  $\lvert \mathcal O_t(\{H_k\}_{k=2}^N) \rvert = 0$ for $t\leq s_{\setminus1}$ and $s \not\equiv s_{\setminus 1} \Mod{2}$. The first term in \cref{eq:lower-bound-induction-b-step2-diff-parity} is by the induction hypothesis \cref{eq:ind-hyp-lower-bound-a} with $\{\tilde H_k\}_{k=1}^{N-1} = \{H_k\}_{k=2}^{N-1}$, $\tilde s = s_{\setminus 1}$, $\tilde r = r_{\setminus 1}$ if $s_{\setminus 1} <r_{\setminus 1}$ and by \cref{eq:rs-equality-a} if $s_{\setminus 1}  = r_{\setminus 1}$. The second term in \cref{eq:lower-bound-induction-b-step2-diff-parity} follows from the induction hypothesis \cref{eq:ind-hyp-lower-bound-b} with $\{\tilde H_k\}_{k=1}^{N-1} = \{H_1 \cap H_k\}_{k=2}^{N-1}$, $\tilde s = s_{/ 1}$, and $\tilde r = r_{/1}$ if $s_{/1} \geq 2$. Whenever $s_{/1} = 1$, we use the analogous cancellation argument as in \cref{eq:removing-duplicates} with \cref{eq:rs-equality-b} in case $\bar s_{/1} = \bar{r}_{/1}$. Finally, \cref{eq:lower-bound-induction-b-step2.1-diff-parity} holds as $s_{\setminus 1} \geq s$ and $s \not\equiv s_{\setminus 1} \Mod{2}$.

    This completes the induction step in the proof of \cref{eq:lower-bound-C_F-a,eq:lower-bound-C_F-b}.
\end{proof}
\subsection{Proof of \texorpdfstring{\cref{L:C_F-ub}}{Lemma}}\label{subsec:C_F-ub}
\begin{proof}
    The proof of \cref{eq:ub-2} is again by induction on the number of hyperplanes $N$, employing the deletion--contraction argument in the induction step. In the regime $2\leq s<r$, we consider for the base case $N=4$, as above, the hyperplanes $\{H_k\}_{k=1}^4$ with $H_k = \{\varphi_k\}^\perp$, where $\{H_k\}_{k=1}^3$ are in general position and $H_4$ is such that, e.g., $\varphi_4 =\varphi_1+\varphi_2$. Then, $s=2$, $r=3$, and
    \begin{gather*}
    \begin{array}{cccccc}
        \lvert \mathcal{E}_0 \rvert = 1, & 
        \lvert \mathcal{E}_1 \rvert = 4, &
        \lvert \mathcal{E}_2 \rvert = 6, &
        \lvert \mathcal{E}_3 \rvert = 3, &
        \lvert \mathcal{E}_4 \rvert = 0, \\[4pt]
        \lvert \mathcal{O}_0 \rvert = 0, &
        \lvert \mathcal{O}_1 \rvert = 0, &
        \lvert \mathcal{O}_2 \rvert = 0, &
        \lvert \mathcal{O}_3 \rvert = 1, &
        \lvert \mathcal{O}_4 \rvert = 1.
    \end{array}
    \end{gather*}
    Thus, \cref{eq:ub-2} holds in the form $\lvert \mathcal{E}_4 \rvert <\lvert \mathcal{O}_3 \rvert$, verifying the base case.
    Next, suppose that \cref{eq:ub-2} is true for an arbitrary set of $N-1$ hyperplanes, $\{\tilde H_k\}_{k=1}^{N-1}$, i.e.,
    \begin{align}\label{eq:ub-ind-hyp}
        \sum_{\substack{t=\tilde r +1 \\ (t-\tilde r) \text{ is odd}}}^{N-1} \lvert \mathcal E_t (\{\tilde H_k\}_{k=1}^{N-1}) \rvert < \sum_{t=\tilde s+1}^{\tilde r} \lvert \mathcal O_t (\{\tilde H_k\}_{k=1}^{N-1}) \rvert + \sum_{\substack{t=\tilde r+2 \\ (t-\tilde r) \text{ is even}}}^{N-1} \lvert \mathcal O_t (\{\tilde H_k\}_{k=1}^{N-1})\rvert,
    \end{align}
    where $\tilde s \coloneqq \kr{\{\tilde \varphi_k\}_{k=1}^{N-1}}$ and $\tilde r \coloneqq \rank{\{\tilde \varphi_k\}_{k=1}^{N-1}}$ with $2\leq \tilde s <\tilde r$. Furthermore, $\{\tilde \varphi_k\}_{k=1}^{N-1}$ are such that $\tilde H_k = \{\tilde\varphi_k\}^\perp$, for all $k \in \{1,\ldots,N-1\}$.
    
    Recall that there exists a $\mathcal{K}\subseteq \{1,\ldots,N\}$ with $\lvert \mathcal{K} \rvert = s+1$ such that $\{\varphi_k\}_{k \in \mathcal{K}}$ is linearly dependent. We again assume without loss of generality that $1 \in \mathcal{K}$, so that $s_{\setminus1}\geq s$, $r_{\setminus 1} = r$, $s_{/1} =s-1$, and $r_{/1}=r-1$. 
    
    We compute
    \allowdisplaybreaks
    \begin{align}
        \sum_{\substack{t=r+1 \\ (t-r) \text{ is odd}}}^N \lvert \mathcal E_t \rvert &= \sum_{\substack{t=r+1 \\ (t-r) \text{ is odd}}}^N \lvert \mathcal E_t(\{H_k\}_{k=2}^N) \rvert + \sum_{\substack{t=r+1 \\ (t-r) \text{ is odd}}}^N \lvert \mathcal E_{t-1}(\{H_1\cap H_k\}_{k=2}^N) \rvert  \label{eq:ub-is-1}\\
        &=\sum_{\substack{t=r_{\setminus1}+1 \\ (t-r_{\setminus1}) \text{ is odd}}}^{N-1} \lvert \mathcal E_t(\{H_k\}_{k=2}^N) \rvert + \sum_{\substack{t=r_{/1}+1 \\ (t-r_{/1}) \text{ is odd}}}^{N-1} \lvert \mathcal E_{t}(\{H_1\cap H_k\}_{k=2}^N) \rvert \nonumber\\
        &< \sum_{t=s_{\setminus 1}+1}^{r_{\setminus1}} \lvert \mathcal O_t (\{ H_k\}_{k=2}^{N}) \rvert + \sum_{\substack{t=r_{\setminus1}+2 \\ (t-r_{\setminus1}) \text{ is even}}}^{N-1} \lvert \mathcal O_t (\{H_k\}_{k=2}^{N})\rvert \label{eq:ub-is-2}\\
        &\quad + \sum_{t=s_{/1}+1}^{r_{/1}} \lvert \mathcal O_t (\{ H_1\cap H_k\}_{k=2}^{N}) \rvert + \sum_{\substack{t=r_{/1}+2 \\ (t-r_{/1}) \text{ is even}}}^{N-1} \lvert \mathcal O_t (\{H_1 \cap H_k\}_{k=2}^{N})\rvert \nonumber\\
        &\leq \sum_{t=s+1}^{r} \lvert \mathcal O_t (\{ H_k\}_{k=2}^{N}) \rvert + \sum_{\substack{t=r+2 \\ (t-r) \text{ is even}}}^{N-1} \lvert \mathcal O_t (\{H_k\}_{k=2}^{N})\rvert \label{eq:ub-is-3}\\
        &\quad + \sum_{t=s+1}^{r} \lvert \mathcal O_{t-1} (\{ H_1\cap H_k\}_{k=2}^{N}) \rvert + \sum_{\substack{t=r+2 \\ (t-r) \text{ is even}}}^{N} \lvert \mathcal O_{t-1} (\{H_1 \cap H_k\}_{k=2}^{N})\rvert \nonumber\\
        &= \sum_{t=s+1}^{r} \lvert \mathcal O_t (\{ H_k\}_{k=1}^{N}) \rvert + \sum_{\substack{t=r+2 \\ (t-r) \text{ is even}}}^{N} \lvert \mathcal O_t (\{H_k\}_{k=1}^{N})\rvert, \label{eq:ub-is-4}
    \end{align}
    where \cref{eq:ub-is-1} is by \cref{eq:del-contr-even}. If $s_{\setminus1}<r_{\setminus 1}$, the first two terms in \cref{eq:ub-is-2} follow from the induction hypothesis \cref{eq:ub-ind-hyp} applied with $\{\tilde H_k\}_{k=1}^{N-1} = \{H_k\}_{k=2}^{N-1}$, $\tilde s = s_{\setminus 1}$, $\tilde r = r_{\setminus 1}$; if $s_{\setminus1}=r_{\setminus 1}$, they follow from \cref{eq:rs-equality-a}. The third and fourth terms in \cref{eq:ub-is-2} are by the induction hypothesis \cref{eq:ub-ind-hyp} with $\{\tilde H_k\}_{k=1}^{N-1} = \{H_1 \cap H_k\}_{k=2}^{N-1}$, $\tilde s = s_{/ 1}$, and $\tilde r = r_{/1}$ whenever $s_{/1} \geq 2$. For $s_{/1}=1$, assume without loss of generality that $\{H_1 \cap H_k\}_{k=2}^n$ are distinct and each of $\{H_1 \cap H_k\}_{k=n+1}^{N}$ is a duplicate of one of the first $(n-1)$ with $n \in \{2,\ldots,N-1\}$. In particular, for $H_1 \cap H_{n+1}$, there exists an $\ell \in \{2,\ldots,n\}$ such that $H_1 \cap H_{n+1} = H_1 \cap H_\ell$. Consider now subsets of $\{H_1 \cap H_k\}_{k=2}^N$ that contain $H_1\cap H_{n+1}$ but exclude any of $\{H_1 \cap H_k\}_{k=n+2}^N$. Then, these subsets naturally form pairs; namely, for every such subset which includes $H_1 \cap H_\ell$, there is a corresponding subset that is identical but does not contain $H_1 \cap H_\ell$. Therefore, if one subset of such a pair is in $\bigcup_{t = r_{/1}+1, \text{ $(t-r_{/1})$ odd}}^{N-1} \mathcal{E}_t(\{H_1\cap H_k\}_{k=2}^N)$, the other subset of this pair belongs to $\bigcup_{t = r_{/1}, \text{ $(t-r_{/1})$ even}}^{N-1} \mathcal{O}_t(\{H_1\cap H_k\}_{k=2}^N)$. Applying this argument to each of $\{H_1\cap H_k\}_{k=n+2}^N$,   
    we obtain the implication
    \allowdisplaybreaks
    \begin{align}
    \begin{split}
    &\left(\sum_{\substack{t=\bar r_{/1}+1 \\ (t-\bar r_{/1}) \text{ is odd}}}^{n-1} \lvert \mathcal E_{t}(\{H_1\cap H_k\}_{k=2}^n) \rvert \leq  \sum_{t=\bar s_{/1}+1}^{\bar r_{/1}} \lvert \mathcal O_t (\{ H_1\cap H_k\}_{k=2}^{n}) \rvert \right. \\ 
    &\left.\longphantom\quad + \sum_{\substack{t=\bar r_{/1}+2 \\ (t-\bar r_{/1}) \text{ is even}}}^{n-1} \lvert \mathcal O_t (\{H_1 \cap H_k\}_{k=2}^{n})\rvert\right) \label{eq:removing-duplicate-implication}\end{split}\\
    \begin{split}\label{eq:removing-duplicate-implication-strict}
    \implies
    &\left(\sum_{\substack{t=r_{/1}+1 \\ (t-r_{/1}) \text{ is odd}}}^{N-1} \lvert \mathcal E_{t}(\{H_1\cap H_k\}_{k=2}^N) \rvert <  \sum_{t=s_{/1}+1}^{r_{/1}} \lvert \mathcal O_t (\{ H_1\cap H_k\}_{k=2}^{N}) \rvert \right. \\ 
    &\left. \longphantom\quad + \sum_{\substack{t=r_{/1}+2 \\ (t-r_{/1}) \text{ is even}}}^{N-1} \lvert \mathcal O_t (\{H_1 \cap H_k\}_{k=2}^{N})\rvert\right).
    \end{split}
    \end{align}
    Here, $\bar s_{/1} \coloneqq \kr{\{P_{H_1}\varphi_k\}_{k=2}^n} \geq 2$ and $\bar r_{/1} \coloneqq \rank{\{P_{H_1}\varphi_k\}_{k=2}^n} = r_{/1}$. The inequality in \cref{eq:removing-duplicate-implication-strict} is strict because $\lvert \mathcal O_{s_{/1}+1} (\{ H_1\cap H_k\}_{k=2}^{N}) \rvert > 0$.  
    Note that \cref{eq:removing-duplicate-implication} holds by the induction hypothesis \cref{eq:ub-ind-hyp} if $\bar s_{/1} < \bar r_{/1}$ and by \cref{eq:rs-equality-a} if $\bar s_{/1} = \bar r_{/1}$. This shows \cref{eq:ub-is-2}. In \cref{eq:ub-is-3}, we used that $s \leq s_{\setminus 1}$, and finally, \cref{eq:ub-is-4} is by \cref{eq:del-contr-odd}.
\end{proof}

\section{Separation on low-dimensional datasets}\label{sec:sep-struct-data}
The framework introduced in the previous section applies to arbitrary subsets $E$, as the presented results are of pure combinatorial nature. In this section, we particularize $E$ to subsets of $\mathbb R^M$ that exhibit low-dimensional structure in a measure-theoretic sense, i.e., $\mathcal{L}^M(E)=0$. More precisely, we consider sets $E \subseteq \mathbb R^M$ which are $\mathcal{H}^s$-measurable and of positive and $\sigma$-finite $\mathcal{H}^s$-measure for some $s \geq 0$. Under these conditions, this section addresses the problem of determining the number of $\Phi$-separable dichotomies of an $N$-point set $F \subseteq E$. The particularization of $E$ to $\mathcal{H}^s$-measurable sets of positive $\mathcal{H}^s$-measure is motivated by the following considerations:
\begin{enumerate}[label=(\roman*)]
    \item This measure-theoretic notion of intrinsic low-dimensionality includes several examples which are often assumed to be reasonable models for real-world high-dimensional data; namely, sets of sparse vectors (i.e., union of linear subspaces), submanifolds, as well as union of submanifolds, see, e.g., \cite{narayanan2010sample,bengio2013representation,fefferman2016testing,carlsson2009topology,baraniuk2009random}. 
    \item Addressing the problem of counting the number of $\Phi$-separable dichotomies of $N$-point sets from a measure-theoretical viewpoint allows us to exclude $N$-point sets which yield degenerate configurations. This, in turn, enables us to understand the factors that affect the number of $\Phi$-separable dichotomies of most $N$-point sets, in terms of the properties of $E$. 
    \item It lays the foundation for studying how the low-dimensional structure of the dataset affects the separation capacity, a measure-theoretic quantity. 
\end{enumerate}

Our analysis proceeds through a chain of increasingly rich geometries for the set $E$. We begin by deriving function-counting results for homogeneous linear separation on sets of $s$-sparse vectors in $\mathbb R^M$ (i.e., sets of finite unions of $s$-dimensional linear subspaces). This analysis makes explicit how the sparsity parameter $s$ affects the number of homogeneously linearly separable dichotomies. Motivated by this result, we then extend our analysis to homogeneous linear separation on so-called countably $\mathcal{H}^s$-rectifiable sets. The latter are sets which can be decomposed as countable unions of $s$-dimensional $C^1$-submanifolds up to a set of $\mathcal{H}^s$-measure zero and thus constitute a natural (and geometrically richer) generalization of the sets of $s$-sparse vectors. Finally, the function-counting results for countably $\mathcal{H}^s$-rectifiable sets let us treat $\Phi$-separability on general $\mathcal{H}^s$-measurable sets with positive and $\sigma$-finite $\mathcal{H}^s$-measure. 

\subsection{Sparse vectors}\label{subsec:sep-sparse}
    Sparse models for datasets naturally arise in the context of representations based on bases or frames (i.e., redundant spanning sets \cite{christensen2003introduction,feichtinger2012gabor}) in which a given vector (signal) can be written as a linear combination of only a few basis or frame elements.  For example, natural images often exhibit sparsity when represented in wavelet bases \cite{feichtinger2012gabor}. In this subsection, we aim to determine the number of homogeneously separable dichotomies of an $N$-point set consisting of vectors that are sparse in a given basis or frame.   
    We start with sparsity in an arbitrary basis for $\mathbb R^M$, and then generalize to sparsity in frames. The former allows for a simple and clean statement of the result and serves as a natural stage for developing the general ideas. 
    
    \subsubsection{Bases} 
    Let $\Xi = \{\xi_k\}_{k=1}^M$ be an arbitrary basis for $\mathbb R^M$, $M\in \mathbb N$, and fix $s \in \{1,\ldots,M\}$. The set of $s$-sparse vectors $E$ is given by the union of $J\coloneqq \binom{M}{s}$ distinct linear subspaces $E_j \coloneqq \spanRinline{\{\xi_k\}_{k\in S_j}}$, for $j \in \{1,\ldots,J\}$, each of dimension $s$, i.e., $E = \bigcup_{j=1}^JE_j$. Here, $\{S_j\}_{j=1}^J$ denotes the set of pairwise distinct subsets of $\{1,\ldots,M\}$ with $\lvert S_j \rvert = s$, $j \in \{1,\ldots,J\}$. In words, $f \in \mathbb R^M$ is $s$-sparse (i.e., $f \in E$) if it can be written as a linear combination of \emph{at most} $s$ elements of $\Xi$. We further denote by $\pi_j \colon E_j \to \mathbb R^s$ the map from the linear subspace $E_j$, $j \in \{1,\ldots,J\}$, to the space of expansion coefficients, i.e., 
        \begin{align}\label{eq:coordinates}
            \pi_j \colon \sum_{i=1}^s c_i\xi_{k_{j,i}} \mapsto c
    \end{align}
    with the labeling $S_j=\{k_{j,i}\}_{i=1}^s$.  
    Consider now an $N$-point set $F\coloneqq \{f_1, \ldots, f_N\} \subset E$. We write $F_j \coloneqq F \cap E_j$ for the points of $F$ in the subspace $E_j$, $j\in \{1,\ldots,J\}$, and make the following assumption. The shift from general position to this assumption allows the intrinsic structure to surface in the dichotomy count. 
    \begin{assumption}\label{A:sep-sparse-basis}
    For every $j\in \{1,\ldots,J\}$, assume that, whenever $F_j \neq \emptyset$,
    \begin{enumerate}[label=(b-\roman*)]
        \item\label{it:sep-sp-b-i} $F_j$ is in $\pi_j$-general position,
        \item\label{it:sep-sp-b-ii} $F_j \cap E_i = \emptyset$, for every $i \in \{1,\ldots,J\}$ with $i \neq j$.
    \end{enumerate}
    \end{assumption}
    \cref{it:sep-sp-b-i} ensures that the coefficient vectors of the vectors in $F_j$ are in general position, while \cref{it:sep-sp-b-ii} excludes point sets consisting of $(s-1)$-sparse vectors.
    As we shall see later, in \cref{rem:assumptions-almost-surely-satisfied}, \cref{A:sep-sparse-basis} is very mild in the sense that for $(\mathcal{H}^s)^N$-a.e. $N$-tuple $(f_1,\ldots,f_N) \in E^N$, the corresponding $N$-point set $\{f_1,\ldots,f_N\}$ satisfies \cref{A:sep-sparse-basis}. Our objective is to determine the number of homogeneously linearly separable dichotomies of $F$ under \cref{A:sep-sparse-basis}. The particular case where $M=2$ and $s=1$ is illustrated in \cref{fig:sep-dataset}, and one may observe that here the number of homogeneously separable dichotomies of $F$ is given by $4$ irrespective of $N$. For general $M$ and $s$, however, counting the homogeneously linearly separable dichotomies becomes more challenging. To this end, let $N_j \coloneqq \lvert F_j \rvert$, $j \in \{1,\ldots,J\}$, so that $N= \sum_{j=1}^J N_j$, and set $\underline{N} \coloneqq (N_j)_{j=1}^J$. We shall refer to $\underline{N}$ as the \emph{configuration associated with} $F$. With \cref{thm:winder} as a cornerstone, we obtain the following dichotomy count.
    
    \begin{proposition}[Sparsity in a basis]\label{prop:sep-sparse-basis}
        Under \cref{A:sep-sparse-basis} the number of homogeneously linearly separable dichotomies of $F$ is given by 
        \begin{align*}
            C_{\mathrm{sp,b}}(\underline{N},M,s) \coloneqq 2^N - 2 \sum_{t=s+1}^N \sum_{\underline{\nu} \in \mathcal{I}^{\mathrm{sp,b}}_t} \binom{\underline N}{\underline \nu},
        \end{align*}
        where $\mathcal{I}^{\mathrm{sp,b}}_t \coloneqq \{\underline{\nu} \in \mathbb N_0^J \colon \lenin{\underline \nu} = t, \Upsilon^{\mathrm{sp,b}}(\underline{\nu}) \not\equiv t \Mod{2}\}$
        with\footnote{Here, $\mathfrak{s}^c$ denotes the complement of $\mathfrak{s}$ in $\supp{\underline{\nu}}$.} 
        \begin{align*}
            \Upsilon^{\mathrm{sp,b}}(\underline{\nu}) \coloneqq \min_{\mathfrak{s}\subseteq \suppinline{\underline{\nu}}} \left\{\sum_{j \in \mathfrak{s}} \nu_j + \left\lvert \bigcup_{j \in \mathfrak{s}^c} S_j\right\rvert\right\}, \quad \underline{\nu}=(\nu_j)_{j=1}^J \in \mathbb N_0^J. 
        \end{align*}
    \end{proposition}
    \begin{remark}[Graph-theoretic interpretation] 
        The quantity $\Upsilon^{\mathrm{sp,b}}(\underline \nu)$ admits a natural interpretation in terms of a minimum weighted vertex cover in a bipartite graph. Namely, consider the bipartite graph $G = (X, Y; L)$, illustrated in \cref{fig:bipartite-graph}, where 
        \begin{itemize}
            \item $X = \{x_1, \ldots, x_J\}$ is a vertex set with each $x_j$ of weight $\nu_j$, $j \in \{1,\ldots,J\}$, 
            \item $Y = \{y_1, \ldots, y_M\}$ is a vertex set disjoint from $X$ with each vertex $y_m$ carrying weight $1$, $m \in \{1,\ldots,M\}$, and  
            \item the edge set $L$ is defined such that $(x_j, y_m) \in L$ whenever $m \in S_j$. In particular, each $x_j \in X$ is connected to exactly $s$ distinct vertices in $Y$.  
    \end{itemize}
    A \emph{minimum weighted vertex cover} of $G$ is a set of vertices $C \subseteq X \cup Y$ such that every edge in $L$ has at least one endpoint in $C$, and the sum of the weights of all vertices in $C$ is as small as possible. The set $\mathcal{I}^{\mathrm{sp,b}}_t$ is then the set of all $\underline{\nu} \in \mathbb N_0^J$ with $\lenin{\underline{\nu}} = t$ for which the resulting minimum vertex cover $\Upsilon^{\mathrm{sp,b}}(\underline \nu)$ differs in parity from $t$.    
    \begin{figure}[t]
        \centering
        \begin{subfigure}{0.45\textwidth}
            \centering
            \resizebox{0.7\columnwidth}{!}{
            \begin{tikzpicture}[scale=1, every node/.style={scale=0.8}]
    \foreach \i in {1,2,3,4,5,6}
        \node[circle, draw, fill=blue!15, minimum size=10mm,align=center] (x\i) at (0,-\i*1) {$x_{\i}$\\$\nu_{\i}\vphantom{043}$};
        
    \foreach \j in {1,2,3,4}
        \node[circle, draw, fill=red!15, minimum size=10mm,align=center] (y\j) at (4,-1-\j*1.0) {$y_\j$\\$1$};

    \draw (x1) -- (y1);
    \draw (x1) -- (y2);

    \draw (x2) -- (y1);
    \draw (x2) -- (y3);

    \draw (x3) -- (y1);
    \draw (x3) -- (y4);

    \draw (x4) -- (y2);
    \draw (x4) -- (y3);

    \draw (x5) -- (y2);
    \draw (x5) -- (y4);

    \draw (x6) -- (y3);
    \draw (x6) -- (y4);

    \node at (0,0) {$X$};
    \node at (4,0) {$Y$};
\end{tikzpicture}}
            \caption{}
            \label{fig:bipartite-graph-a}
        \end{subfigure}
        \hfill
        \begin{subfigure}{0.45\textwidth}
            \centering
            \resizebox{0.7\columnwidth}{!}{
            \begin{tikzpicture}[scale=1, every node/.style={scale=0.8}]
    \foreach \i in {1,3,5} {
        \ifnum\i=1 
            \node[circle, draw, fill=blue!15, minimum size=10mm,align=center] (x\i) at (0,-\i*1) {$x_{\i}$\\$4\vphantom{\nu_{\i}}$};
        \else
            \ifnum\i=3
                \node[circle, draw, fill=blue!15, minimum size=10mm,align=center] (x\i) at (0,-\i*1) {$x_{\i}$\\$3\vphantom{\nu_{\i}}$};
            \else
                \node[circle, draw, fill=blue!15, minimum size=10mm,align=center] (x\i) at (0,-\i*1) {$x_{\i}$\\$0\vphantom{\nu_{\i}}$};
            \fi
        \fi
    }
    \foreach \i in {2,4,6}
        \node[circle, draw, fill=green!15, minimum size=10mm,align=center] (x\i) at (0,-\i*1) {$x_{\i}$\\$0\vphantom{\nu_{\i}}$};
    \foreach \j in {1,2,4}
        \node[circle, draw, fill=green!15, minimum size=10mm,align=center] (y\j) at (4,-1-\j*1.0) {$y_\j$\\$1$};
    \foreach \j in {3}
        \node[circle, draw, fill=red!15, minimum size=10mm,align=center] (y\j) at (4,-1-\j*1.0) {$y_\j$\\$1$};

    \draw (x1) -- (y1);
    \draw (x1) -- (y2);

    \draw (x2) -- (y1);
    \draw (x2) -- (y3);

    \draw (x3) -- (y1);
    \draw (x3) -- (y4);

    \draw (x4) -- (y2);
    \draw (x4) -- (y3);

    \draw (x5) -- (y2);
    \draw (x5) -- (y4);

    \draw (x6) -- (y3);
    \draw (x6) -- (y4);

    \node at (0,0) {$X$};
    \node at (4,0) {$Y$};
\end{tikzpicture}}
            \caption{}
            \label{fig:bipartite-graph-b}
        \end{subfigure}
        \caption{Illustration of the bipartite graph $G=(V_1,V_2;L)$, whose minimum weighted vertex cover equals $\Upsilon^{\mathrm{sp,b}}(\underline{\nu})$. \subref{fig:bipartite-graph-a} Bipartite graph with $M=4$, $s=2$, and $J=\binom{4}{2}=6$. \subref{fig:bipartite-graph-b} Minimum weighted vertex cover highlighted in green for $\underline \nu = (4,0,3,0,0,0)^\mathsf{T}$.}
        \label{fig:bipartite-graph}
    \end{figure}
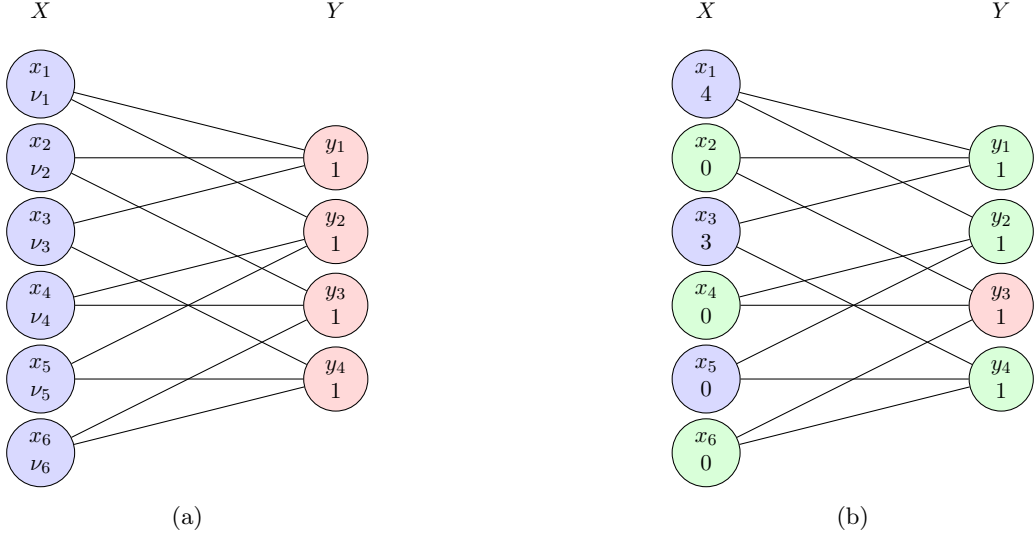
    \end{remark}
    \begin{proof}[Proof of \cref{prop:sep-sparse-basis}] 
        The proof follows from particularizing \cref{prop:sep-sparse-frame}. Namely, 
        as $E_j = \spanRinline{\{\xi_k\}_{k\in S_j}}$ and $\Xi$ is a basis, we can write
        \begin{align}\label{eq:sep-sparse-basis-3}
            \dimR{\sum_{j \in \mathfrak{s}^c} E_j} = \left\lvert \bigcup_{j \in \mathfrak{s}^c} S_j\right\rvert. 
        \end{align}
    \end{proof}
    Let us discuss the ramifications of \cref{prop:sep-sparse-basis}. The number of homogeneously linearly separable dichotomies of $F$, $C_{\mathrm{sp,b}}(\underline{N},M,s)$, depends only on $\underline{N}$ and not on the specific location of the points in $F$ within each subspace $E_j$, $j\in \{1,\ldots,J\}$. As the points in $F$ are, by \cref{A:sep-sparse-basis}, not $(s-1)$-sparse with respect to $\Xi$ the assignment of the points to the subspaces $\{E_j\}_{j=1}^J$ is unique. It is important to highlight that $C_{\mathrm{sp,b}}(\underline{N},M,s)$ is independent of the choice of the basis $\Xi$. Furthermore, \cref{prop:sep-sparse-basis} can be leveraged to easily compute the number of homogeneously linearly separable dichotomies for the special case where all points lie in one subspace, as will be carried out in the next remark.
    \begin{remark}\label{rem:all-in-one}
        If $N_{j_0} = N$ for some $j_0 \in \{1,\ldots,J\}$ and $N_j=0$ for $j \neq j_0$, the number of homogeneously linearly separable dichotomies is given by $C(N,s)$. Indeed, since $N_j=0$ for $j \neq j_0$, it holds that $\binom{N_j}{\nu_j} =0$ if $\nu_j >0$ for $j \neq j_0$. Thus, for $t \geq s+1$, it suffices to consider the $J$-tuple $(\nu_1,\ldots,\nu_J)$ where $\nu_j = 0$ for $j \neq j_0$ and $\nu_{j_0}=t$. Upon noting that this $J$-tuple belongs to $\mathcal{I}^{\mathrm{sp,b}}_t$ if and only if $s \not \equiv t \Mod{2}$, we obtain
        \allowdisplaybreaks
        \begin{align}
            C_{\mathrm{sp,b}}(\underline{N},M,s) &= 2^N - 2 \sum_{\substack{t=s+1 \\ (t-s) \text{ is odd}}}^N \binom{N}{t} \nonumber \\
            &= 2^N - 2 \sum_{\substack{t=s+1 \\ (t-s) \text{ is odd}}}^N \left(\binom{N-1}{t}+\binom{N-1}{t-1}\right) \label{eq:s-dim-subsp-pascal}\\ 
            &= 2^N - 2 \sum_{t=s}^N \binom{N-1}{t} \nonumber\\ 
            &= 2 \sum_{t=0}^N \binom{N-1}{t} - 2 \sum_{t=s}^N \binom{N-1}{t} \label{eq:s-dim-subsp-sum} \\
            &= 2 \sum_{t=0}^{s-1} \binom{N-1}{t} \nonumber\\
            &= C(N,s), \label{eq:all-in-one}
        \end{align}
        where \cref{eq:s-dim-subsp-pascal,eq:s-dim-subsp-sum} follow from the well-known identities $\binom{N}{t} = \binom{N-1}{t}+\binom{N-1}{t-1}$ and $2^{N-1} = \sum_{t=0}^{N-1} \binom{N-1}{t}$, respectively. 
    \end{remark}
    Finally, we will discuss the influence of the ambient dimension $M$ and the sparsity parameter $s$ on $C_{\mathrm{sp,b}}(\underline{N},M,s)$ in the next two paragraphs. Before doing so, note that from \cref{eq:all-in-one} one can already deduce that, for the special case in \cref{rem:all-in-one}, $C_{\mathrm{sp,b}}(\underline{N},M,s)$ is independent of $M$ and nondecreasing in $s$. 
    
    \paragraph{Ambient dimension}
        Let $M'\in \mathbb N$ with $M\leq M'$, and set $J'\coloneqq \binom{M'}{s}$. Note that $J \leq J'$. Denoting by $\{S_j'\}_{j=1}^{J'}$ all pairwise distinct subsets of $\{1, \ldots, M'\}$ of cardinality $s$, we order these subsets such that for every $\mathfrak{s} \subseteq \{1, \ldots, J\}$, $\lvert \bigcup_{j\in\mathfrak{s}}S_j'\rvert=\lvert \bigcup_{j \in \mathfrak{s}}S_j\rvert$. Let $\underline{N} \in \mathbb N_0^{J}$. To investigate the effect of the ambient dimension $M$ on $C_{\mathrm{sp,b}}(\underline{N},M,s)$, consider the following embedding of $\underline{N}$ into $\mathbb N_0^{J'}$, given by 
        \begin{align}\label{eq:embedding}
            \iota\colon \mathbb N_0^{J} \to \mathbb N_0^{J'},\;\; \underline{N}\mapsto(\underline{N}^\mathsf{T},\underbrace{0,\ldots,0}_{\text{$(J'-J)$ times}})^\mathsf{T}.
        \end{align}
        \begin{lemma}[Independence of the ambient dimension]
            For $M,M'\in \mathbb N$ with $M\leq M'$, $s\in\{1,\ldots,M\}$, and $\underline{N} \in \mathbb N_0^{J}$, it holds that
            \begin{align}\label{eq:amb-dim}
                C_{\mathrm{sp,b}}(\underline{N},M,s) = C_{\mathrm{sp,b}}(\iota(\underline{N}),M',s).
            \end{align}
        \end{lemma}
        \begin{proof}
            To see this, simply note that for every $\underline{\nu}\in\mathbb N_0^{J'}$, we have $\binom{\iota(\underline{N})}{\underline \nu} = 0$ whenever $\supp{\underline{\nu}} \not\subseteq \supp{\iota(\underline{N})}$. Moreover, the ordering of the sets $\{S_j'\}_{j=1}^{J'}$ ensures that $\lvert \bigcup_{j\colon \nu_j>0}S_j'\rvert = \lvert \bigcup_{j\colon \nu_j>0}S_j\rvert$, for every $\underline{\nu}\in\mathbb N_0^{J'}$ with $\supp{\underline{\nu}} \subseteq \supp{\iota(\underline{N})}\subseteq\{1,\ldots,J\}$. This, together with the definition of $C_{\mathrm{sp,b}}(\iota(\underline{N}),M',s)$,  establishes \cref{eq:amb-dim}.
        \end{proof}
        \sloppy We can thus conclude that the number of homogeneously linearly separable dichotomies of an $N$-point subset of the set of $s$-sparse vectors is independent of the ambient dimension, when embedded in the sense of \cref{eq:embedding}.    

    \paragraph{Sparsity parameter}
    Let $s'\in\{1,\ldots,M\}$ with $s\leq s'$. We now set $J' \coloneqq \binom{M}{s'}$, and denote by $\{S_j'\}_{j=1}^{J'}$ all pairwise distinct subsets of $\{1, \ldots, M\}$ of cardinality $s'$.  Note that $J \leq J'$ does not necessarily hold. Given a configuration $\underline{N}=(N_j)_{j=1}^J \in \mathbb N_0^J$, our goal is to construct a representation of $\underline{N}$ that is compatible with the set of $s'$-sparse vectors. This will allow us to analyze the effect of the sparsity parameter $s$ on $C_{\mathrm{sp,b}}(\underline{N},M,s)$. To this end, consider the transformation 
    \begin{align}\label{eq:spars-param-transf}
        \varpi\colon \mathbb N_0^J\to\mathbb N_0^{J'}, \;\; \underline{N} \mapsto \begin{pmatrix}
            \sum_{j \in \mathfrak{s}_1} N_j, \ldots, \sum_{j \in \mathfrak{s}_{J'}} N_j
        \end{pmatrix}^{\mathsf{T}}.
    \end{align}
    Here, $\{\mathfrak{s}_k\}_{k=1}^{J'} \subseteq \{1,\ldots,J\}$ constitute a disjoint decomposition\footnote{Specifically, we allow $\mathfrak{s}_k = \emptyset$, $k\in\{1,\ldots,J'\}$.} of $\{1,\ldots,J\}$ and are defined according to the following rule: For each $j \in \{1,\ldots,J\}$ and $k \in \{1,\ldots,J'\}$, 
    \begin{align}\label{eq:decomp-rule}
        \left(j \in \mathfrak{s}_k\right) \Longleftrightarrow \left(S_j \subseteq S_k' \text{ and } j \notin \bigcup_{\ell =1}^{k-1} \mathfrak{s}_\ell \right)
    \end{align}
    with the convention $\bigcup_{\ell =1}^{0} \mathfrak{s}_\ell = \emptyset$.  
    \begin{lemma}[Monotonicity in the sparsity parameter]\label{lem:spars-param-nondecreasing}
        For $M\in\mathbb N$, $s,s'\in\{1,\ldots,M\}$ with $s\leq s'$, and $\underline{N} \in \mathbb N_0^J$, we have
        \begin{align}\label{eq:spars-param-nondecreasing}
            C_{\mathrm{sp,b}}(\underline{N},M,s) \leq C_{\mathrm{sp,b}}(\varpi(\underline{N}),M,s').
        \end{align}
    \end{lemma}
    To prove this lemma, we require the following result.
    \begin{lemma}\label{lem:monotonicity}
        Let $\{\bar E_j\}_{j=1}^J$ and $\{\bar E'_k\}_{k=1}^{J'}$ be linear subspaces of $\mathbb R^M$, and suppose that there exists a map $\phi \colon \{1,\ldots,J\} \to \{1,\ldots,J'\}$ is such that $\dimRinline{\sum_{j \in \mathfrak{s}}\bar E_j} \leq \dimRinline{\sum_{k \in \phi(\mathfrak{s})}\bar E'_k}$, for all $\mathfrak{s}\subseteq\{1,\ldots,J\}$. Then, for every $N\in \mathbb N$ and every $\underline N \in \mathbb N_0^J$ with $\lenin{\underline N} = N$, 
        \begin{align}\label{eq:s-nondecreasing-claim}
            \sum_{t=1}^{N} \sum_{\underline{\nu} \in \mathcal{I}_t} \binom{\underline{N} }{\underline \nu} \geq \sum_{t=1}^{N} \sum_{\underline{\nu}' \in \mathcal{I}'_t} \binom{\bar\varpi(\underline{N})}{\underline \nu'}, 
        \end{align}
        where $\mathcal{I}_t \coloneqq \{\underline{\nu} \in \mathbb N_0^J \colon \lenin{\underline \nu} =t, \Upsilon(\underline \nu) \not \equiv t\Mod{2}\}$ with $\Upsilon(\underline{\nu}) = \min_{\mathfrak{s}\subseteq\suppinline{\underline \nu}}\{ \sum_{j \in \mathfrak{s}} \nu_j + \dimRinline{\sum_{j \in \mathfrak{s}^c} \bar E_j}\}$, for $\underline \nu \in \mathbb N_0^J$, and $\mathcal{I}'_t \coloneqq \{\underline{\nu}' \in \mathbb N_0^{J'} \colon \lenin{\underline \nu'} =t, \Upsilon'(\underline \nu') \not \equiv t\Mod{2}\}$ with $\Upsilon'(\underline{\nu}') = \min_{\mathfrak{s}\subseteq\suppinline{\underline \nu}'} \{\sum_{k \in \mathfrak{s}} \nu'_k + \dimRinline{\sum_{k \in \mathfrak{s}^c} \bar E'_k}\}$, for $\underline{\nu}' \in \mathbb N_0^{J'}$. Here, $\bar{\varpi} \colon \mathbb N_0^J \to \mathbb N_0^{J'}$ is the map of the form \cref{eq:spars-param-transf} with $\mathfrak{s}_k \coloneqq\phi^{-1}(\{k\})$.
    \end{lemma}
    \begin{proof}
        The proof of \cref{eq:s-nondecreasing-claim} is by induction on $N$. For the base case $N=1$, let $j_0$ be the unique element in $\suppinline{\underline{N}}$. By assumption, $\dimRinline{\bar E_{j_0}} \leq \dimRinline{\bar E'_{k_0}}$, where $k_0 = \phi(j_0)$. Denoting by $\underline e_{j_0} \in \mathbb N_0^J$ and $\underline e'_{k_0} \in \mathbb N_0^{J'}$ the multi-indices whose $j_0$th and $k_0$th entries equal $1$, respectively, and all others are $0$, we have 
        \begin{align*}
            \Upsilon(\underline{e}_{j_0}) = \min\left\{1,\dimR{\bar E_{j_0}}\right\} \leq \min\left\{1,\dimR{\bar E'_{k_0}}\right\} = \Upsilon'(\underline{e}'_{k_0}).
        \end{align*}
        Thus, if $\Upsilon'(\underline{e}'_{k_0}) \not\equiv 1 \Mod{2}$, i.e., $\Upsilon'(\underline{e}'_{k_0}) =0$, then $\Upsilon(\underline{e}_{j_0}) \not \equiv 1 \Mod{2}$. Upon noting that $\bar \varpi(\underline{e}_{j_0}) = \underline{e}'_{k_0}$, it follows that $\underline{e}_{j_0} \in \mathcal{I}_1$ if $\underline{e}_{j_0} \in \bar\varpi^{-1}(\mathcal{I}'_1)$, and hence 
        \begin{align*}
            \sum_{\underline{\nu} \in \mathcal{I}_1} \binom{\underline{N}}{\underline \nu} \geq \sum_{\underline{\nu} \in \bar{\varpi}^{-1}(\mathcal{I}'_1)} \binom{\underline{N}}{\underline \nu} = \sum_{\underline{\nu}' \in \mathcal{I}'_1} \binom{\bar\varpi(\underline{N})}{\underline \nu'}, 
        \end{align*}
        verifying the base case.
        
        Now suppose \cref{eq:s-nondecreasing-claim} is true for $N-1$. To establish the induction step, let $j_0 \in \suppinline{\underline N}$ and $k_0 = \phi(j_0)$. We further introduce $\mathcal{I}^{/j_0}_t \coloneqq \{\underline{\nu} \in \mathbb N_0^J \colon \lenin{\underline \nu} =t, \Upsilon^{/j_0}(\underline \nu) \not \equiv t\Mod{2}\}$ and $\mathcal{I}'^{/k_0}_t \coloneqq \{\underline{\nu}' \in \mathbb N_0^{J'} \colon \lenin{\underline \nu'} =t, \Upsilon'^{/k_0}(\underline \nu') \not \equiv t\Mod{2}\}$, where $\Upsilon^{/j_0}(\underline{\nu}) = \min_{\mathfrak{s}\subseteq\suppinline{\underline \nu}}\{ \sum_{j \in \mathfrak{s}} \nu_j + \dimRinline{\sum_{j \in \mathfrak{s}^c} P_{j_0}\bar E_j}\}$, for $\underline \nu \in \mathbb N_0^J$, and $\Upsilon'^{/k_0}(\underline{\nu}') = \min_{\mathfrak{s}\subseteq\suppinline{\underline \nu'}}\{ \sum_{k \in \mathfrak{s}} \nu'_k + \dimRinline{\sum_{k \in \mathfrak{s}^c} P'_{k_0}\bar E'_k}\}$, for $\underline \nu' \in \mathbb N_0^{J'}$. Here, $P_{j_0}$ and $P'_{k_0}$ denote the orthogonal projections onto $\{f^{j_0}\}^\perp$ and $\{f'^{k_0}\}^\perp$ for some $f^{j_0} \in \bar E_{j_0} \setminus U^{j_0}$ and $f'^{k_0} \in \bar E'_{k_0} \setminus V^{k_0}$. Here,
        \begin{align*}
            U^{j_0} \coloneqq \bigcup_{\mathfrak{s}\colon \bar E_{j_0} \not\subseteq \sum_{j \in \mathfrak{s}} \bar E_j} \left(\bar E_{j_0} \cap \sum_{j \in \mathfrak{s}} \bar E_j\right),
        \text{ and }
            V^{k_0} \coloneqq \bigcup_{\mathfrak{s}\colon \bar E'_{k_0} \not\subseteq \sum_{k \in \mathfrak{s}} \bar E'_k} \left(\bar E'_{k_0} \cap \sum_{k \in \mathfrak{s}} \bar E'_k\right).
        \end{align*}
        As a linear space cannot be covered by finitely many proper linear subspaces, the sets $\bar E_{j_0} \setminus U^{j_0}$ and $\bar E'_{k_0} \setminus V^{k_0}$ are non-empty. This choice of $f^{j_0}$ and $f'^{k_0}$ ensures that    
        \begin{enumerate}[label=(\roman*)]
            \item\label{it:s-nondecr-contr-1} for every $\mathfrak{s} \subseteq \{1,\ldots,J\}$,
            \begin{align*}
                \left(f^{j_0} \in \sum_{j \in \mathfrak{s}} \bar E_j\right) \iff \left(\bar E_{j_0} \subseteq \sum_{j \in \mathfrak{s}} \bar E_j\right),
            \end{align*} 
            \item\label{it:s-nondecr-contr-2} and for every $\mathfrak{s} \subseteq \{1,\ldots,J'\}$, 
            \begin{align*}
                \left(f'^{k_0} \in \sum_{k \in \mathfrak{s}} \bar E'_k\right) \iff \left(\bar E'_{k_0} \subseteq \sum_{k \in \mathfrak{s}} \bar E'_k\right).
            \end{align*}
        \end{enumerate}
        \Cref{it:s-nondecr-contr-1,it:s-nondecr-contr-2} can be leveraged to show for every $\mathfrak{s} \subseteq \{1,\ldots,J\}$, 
        \begin{align}\label{eq:subspace-dominance-proj}
            \dimR{\sum_{j \in \mathfrak{s}} P_{j_0}\bar E_j} \leq \dimR{\sum_{k \in \phi(\mathfrak{s})}\bar P'_{k_0}E'_k}.
        \end{align}
        Indeed, we have, by the rank--nullity theorem and \cref{it:s-nondecr-contr-1}, 
        \begin{align*}
            \dimR{\sum_{j \in \mathfrak{s}} P_{j_0}\bar E_j} &= \dimR{\sum_{j \in \mathfrak{s}} \bar E_j} - \dimR{\spanR{\{f^{j_0}\}}\cap\sum_{j \in \mathfrak{s}} \bar E_j} \\
            &= \dimR{\sum_{j \in \mathfrak{s}} \bar E_j} - \mathbbm{1}_{\left\{\bar E_{j_0} \subseteq \sum_{j \in \mathfrak{s}} \bar E_j\right\}}, 
        \end{align*}
        and likewise,
        \begin{align*}
            \dimR{\sum_{k \in \phi(\mathfrak{s})} P'_{k_0}\bar E'_k} 
            &= \dimR{\sum_{k \in \phi(\mathfrak{s})} \bar E'_k} - \mathbbm{1}_{\left\{\bar E'_{k_0} \subseteq \sum_{k \in \phi(\mathfrak{s})} \bar E'_k\right\}},
        \end{align*}
        so that \cref{eq:subspace-dominance-proj} holds whenever 
        \begin{align}\label{eq:subsp-dom-impl}
        \begin{split}
            \left(\bar E'_{k_0} \subseteq \sum_{k \in \phi(\mathfrak{s})} \bar E'_k \text{ and } \dimR{\sum_{j \in \mathfrak{s}} \bar E_j} = \dimR{\sum_{k \in \phi(\mathfrak{s})} \bar E'_k}\right) \implies \left(\bar E_{j_0} \subseteq \sum_{j \in \mathfrak{s}} \bar E_j\right).            
        \end{split}
        \end{align}
        Now note that whenever the left-hand side (LHS) of the implication in \cref{{eq:subsp-dom-impl}} is true, we have 
        \begin{align*}
            \dimR{\sum_{j \in \mathfrak{s} \cup\{j_0\}} \bar E_j} &\leq \dimR{\sum_{k \in \phi(\mathfrak{s} \cup\{j_0\})} \bar E'_k} \\
            &= \dimR{\sum_{k \in \phi(\mathfrak{s}) \cup\{k_0\}} \bar E'_k} \\
            &= \dimR{\sum_{k \in \phi(\mathfrak{s})} \bar E'_k} \\ 
            &= \dimR{\sum_{j \in \mathfrak{s}} \bar E_j},  
        \end{align*}
        and the right-hand side (RHS) of \cref{{eq:subsp-dom-impl}} follows. This establishes \cref{eq:subspace-dominance-proj}. 
        
        We next prove that for $t\in \mathbb N$,
        \begin{align}
            \left(\underline{\nu} \in \mathcal{I}^{/j_0}_{t-1}\right) &\iff \left((\underline{ \nu}+\underline{e}_{j_0}) \in \mathcal{I}_t\right), \label{eq:s-nondecr-is-contraction-a}\\
        \intertext{and}
            \left(\underline{\nu}' \in  \mathcal{I}'^{/k_0}_{t-1}\right) &\iff \left((\underline{\nu}'+\underline{e}'_{k_0}) \in \mathcal{I}'_t\right), \label{eq:s-nondecr-is-contraction-b}
        \end{align}
        where $\underline e_{j_0} \in \mathbb N_0^J$ denotes the multi-index whose $j_0$th entry equals $1$ and all others are $0$. To see that \cref{eq:s-nondecr-is-contraction-a} holds, note that\footnote{By $\mathfrak{s}^c$ we denote the complement with respect to the indexing set appearing in the minimization (here, either $\supp{\cdot}$ or $\{1,\ldots,J\}$).}
        \begin{align}
            \Upsilon(\underline{ \nu}+\underline{e}_{j_0}) 
            &= \min_{\mathfrak{s} \subseteq\suppinline{\underline \nu +\underline{e}_{j_0}}}\left\{ \sum_{j \in \mathfrak{s}} \nu_j + \dimR{\sum_{j \in \mathfrak{s}^c} \bar E_j} + \mathbbm{1}_{\{j_0 \in \mathfrak{s}\}}\right\} \nonumber\\
            &= \min_{\mathfrak{s} \subseteq\{1,\ldots,J\}}\left\{ \sum_{j \in \mathfrak{s}} \nu_j + \dimR{\sum_{j \in \mathfrak{s}^c} \bar E_j} + \mathbbm{1}_{\{j_0 \in \mathfrak{s}\}}\right\} \nonumber\\
            &= \min_{\mathfrak{s} \subseteq\{1,\ldots,J\}}\left\{ \sum_{j \in \mathfrak{s}} \nu_j + \dimR{\sum_{j \in \mathfrak{s}^c} \bar E_j} - \mathbbm{1}_{\{j_0 \in \mathfrak{s}^c\}}\right\} +1\nonumber\\ 
            &= \min_{\mathfrak{s} \subseteq\{1,\ldots,J\}}\left\{ \sum_{j \in \mathfrak{s}} \nu_j + \dimR{\sum_{j \in \mathfrak{s}^c} \bar E_j} - \mathbbm{1}_{\left\{\bar E_{j_0} \subseteq \sum_{j \in \mathfrak{s}^c} \bar E_j\right\}}\right\} +1 \label{eq:s-nondecr-contraction-a1}\\
            &= \min_{\mathfrak{s} \subseteq\{1,\ldots,J\}}\left\{ \sum_{j \in \mathfrak{s}} \nu_j + \dimR{\sum_{j \in \mathfrak{s}^c} \bar E_j} - \dimR{\spanR{\{f^{j_0}\}}\cap\sum_{j \in \mathfrak{s}^c} \bar E_j}\right\}+1 \label{eq:s-nondecr-contraction-a2}\\
            &= \min_{\mathfrak{s} \subseteq\{1,\ldots,J\}}\left\{ \sum_{j \in \mathfrak{s}} \nu_j + \dimR{\sum_{j \in \mathfrak{s}^c} P_{j_0}\bar E_j}\right\} +1 \label{eq:s-nondecr-contraction-a3}\\
            &= \min_{\mathfrak{s} \subseteq\suppinline{\underline{\nu}}}\left\{ \sum_{j \in \mathfrak{s}} \nu_j + \dimR{\sum_{j \in \mathfrak{s}^c} P_{j_0}\bar E_j}\right\} +1 \nonumber\\
            &= \Upsilon^{/j_0}(\underline{\nu})+1. \nonumber
        \end{align}
        Indeed, \cref{eq:s-nondecr-contraction-a1} holds since $j_0 \in \mathfrak{s}^c$ implies $\bar E_{j_0} \subseteq \sum_{j \in \mathfrak{s}^c} \bar E_j$, and since, conversely, whenever $\bar E_{j_0} \subseteq \sum_{j \in \mathfrak{s}^c} \bar E_j$, the minimization allows us to include $j_0 \in \mathfrak{s}^c$. In \cref{eq:s-nondecr-contraction-a2}, we used \cref{it:s-nondecr-contr-1}, and \cref{eq:s-nondecr-contraction-a3} follows from the rank--nullity theorem. This establishes \cref{eq:s-nondecr-is-contraction-a}.

        Likewise, by employing the same arguments, we obtain
        \begin{align*}
            \Upsilon'(\underline{ \nu}'+\underline{e}'_{k_0})
            &=\min_{\mathfrak{s}\subseteq\suppinline{\underline \nu'+\underline{e}'_{k_0}}} \left\{\sum_{k \in \mathfrak{s}} \nu'_k + \dimR{\sum_{k \in \mathfrak{s}^c} \bar E'_k}+\mathbbm{1}_{\{k_0 \in \mathfrak{s}\}}\right\} \\
            &= \min_{\mathfrak{s} \subseteq\suppinline{\underline{\nu}'}}\left\{ \sum_{k \in \mathfrak{s}} \nu'_k + \dimR{\sum_{k \in \mathfrak{s}^c} P'_{k_0}\bar E'_k}\right\} +1 \\
            &= \Upsilon'^{/k_0}(\underline{\nu}')+1
        \end{align*}
        from which \cref{eq:s-nondecr-is-contraction-b} follows.

        We are now ready to establish the induction step and, to this end, compute
        \allowdisplaybreaks
        \begin{align}
             \sum_{t=1}^N \sum_{\underline{\nu} \in \mathcal{I}_t} \binom{\underline N}{\underline \nu} 
             &= \sum_{t=1}^{N-1} \sum_{\underline{\nu} \in \mathcal{I}_t} \binom{\underline N - \underline e_{j_0}}{\underline \nu} + \sum_{t=1}^N \sum_{\underline{\nu} \in \mathcal{I}_t} \binom{\underline N -\underline e_{j_0}}{\underline \nu - \underline e_{j_0}} \label{eq:s-nondecr-is-1}\\
             &= \sum_{t=1}^{N-1} \sum_{\underline{\nu} \in \mathcal{I}_t} \binom{\underline N - \underline e_{j_0}}{\underline \nu} + \sum_{t=1}^N \sum_{\underline{\nu} \in \mathcal{I}^{/j_0}_{t-1}} \binom{\underline N -\underline e_{j_0}}{\underline \nu} \label{eq:s-nondecr-is-2}\\
             &= \sum_{t=1}^{N-1} \sum_{\underline{\nu} \in \mathcal{I}_t} \binom{\underline N - \underline e_{j_0}}{\underline \nu} + \sum_{t=1}^{N-1} \sum_{\underline{\nu} \in \mathcal{I}^{/j_0}_{t}} \binom{\underline N -\underline e_{j_0}}{\underline \nu} \label{eq:s-nondecr-is-2.1}\\
             &\geq \sum_{t=1}^{N-1} \sum_{\underline{\nu}' \in \mathcal{I}'_t} \binom{\bar\varpi(\underline{ N} -\underline{e}_{j_0})}{\underline \nu'} + \sum_{t=1}^{N-1} \sum_{\underline{\nu}' \in \mathcal{I}'^{/k_0}_t} \binom{\bar\varpi(\underline{N} -\underline{e}_{j_0})}{\underline \nu'} \label{eq:s-nondecr-is-3}\\
             &= \sum_{t=1}^{N-1} \sum_{\underline{\nu}' \in \mathcal{I}'_t} \binom{\bar\varpi(\underline{ N}) -\underline{e}'_{k_0}}{\underline \nu'} + \sum_{t=1}^{N-1} \sum_{\underline{\nu}' \in \mathcal{I}'^{/k_0}_t} \binom{\bar\varpi(\underline{N}) -\underline{e}'_{k_0}}{\underline \nu'} \nonumber\\
             &= \sum_{t=1}^{N-1} \sum_{\underline{\nu}' \in \mathcal{I}'_t} \binom{\bar\varpi(\underline{ N}) -\underline{e}'_{k_0}}{\underline \nu'} + \sum_{t=1}^{N} \sum_{\underline{\nu}' \in \mathcal{I}'^{/k_0}_{t-1}} \binom{\bar\varpi(\underline{N}) -\underline{e}'_{k_0}}{\underline \nu'} \label{eq:s-nondecr-is-4}\\
             &= \sum_{t=1}^{N-1} \sum_{\underline{\nu}' \in \mathcal{I}'_t} \binom{\bar\varpi(\underline{ N}) -\underline{e}'_{k_0}}{\underline \nu'} + \sum_{t=1}^{N} \sum_{\underline{\nu}' \in \mathcal{I}'_{t}} \binom{\bar\varpi(\underline{N}) -\underline{e}'_{k_0}}{\underline \nu' - \underline{e}'_{k_0}} \label{eq:s-nondecr-is-5}\\
             &=\sum_{t=1}^{N} \sum_{\underline{\nu}' \in \mathcal{I}'_t} \binom{\bar\varpi(\underline{N})}{\underline \nu'}, \label{eq:s-nondecr-is-6} 
        \end{align}
        where \cref{eq:s-nondecr-is-1} is by Pascal's rule, \cref{eq:s-nondecr-is-2} follows from \cref{eq:s-nondecr-is-contraction-a}, and \cref{eq:s-nondecr-is-2.1} holds as $\mathcal{I}_0^{/j_0} = \emptyset$. In \cref{eq:s-nondecr-is-3}, we employed the induction hypothesis with $\{\bar E_j\}_{j=1}^J$ and $\{\bar E'_k\}_{k=1}^{J'}$. The second term in \cref{eq:s-nondecr-is-3} is by the induction hypothesis particularized to $\{P_{j_0}\bar E_j\}_{j=1}^J$ and $\{P'_{k_0}\bar E'_k\}_{k=1}^{J'}$, together with \cref{eq:subspace-dominance-proj}. Finally, in \cref{eq:s-nondecr-is-4}, we used that $\mathcal{I}_0'^{/k_0} = \emptyset$, \cref{eq:s-nondecr-is-5} holds by \cref{eq:s-nondecr-is-contraction-b}, and \cref{eq:s-nondecr-is-6} is again a consequence of Pascal's rule.
    \end{proof}
    \begin{proof}[Proof of \cref{lem:spars-param-nondecreasing}]
        The proof is immediate by \cref{lem:monotonicity} particularized to $\{E_j\}_{j=1}^J$ and $\{\spanRinline{\{\xi_k\}_{k\in S'_{j}}}\}_{j=1}^{J'}$. The map $\phi$ is determined by \cref{eq:decomp-rule}. Note that $\mathcal{I}^{\mathrm{sp,b}}_t = 0$ for $t \leq s$ and $\mathcal{I}'^{\mathrm{sp,b}}_t = 0$ for $t\leq s'$, where  
        $\mathcal{I}'^{\mathrm{sp,b}}_t \coloneqq \{\underline{\nu}' \in \mathbb N_0^{J'} \colon \lenin{\underline \nu'} = t, \Upsilon'^{\mathrm{sp,b}}(\underline{\nu}') \not\equiv t \Mod{2}\}$
        with $\Upsilon'^{\mathrm{sp,b}}(\underline{\nu}') \coloneqq \min_{\mathfrak{s}\subseteq \suppinline{\underline{\nu}'}} \{\sum_{j \in \mathfrak{s}} \nu'_j + \lvert \bigcup_{j \in \mathfrak{s}^c} S'_j\rvert\}$ for $\underline{\nu}'=(\nu'_j)_{j=1}^J \in \mathbb N_0^{J'}$. 
    \end{proof}
    
    \cref{lem:spars-param-nondecreasing} establishes that, under transformation \cref{eq:spars-param-transf}, the number of homogeneously linearly separable dichotomies of an $N$-point set of the set of $s$-sparse vectors is nondecreasing in the sparsity parameter $s$. The sparsity parameter $s$, rather than the ambient dimension $M$, can therefore be interpreted as the effective complexity parameter of the dataset $E$, in the sense that it is the fundamental quantity which determines how many dichotomies can be realized. 
    
    Next, we proceed to deriving function-counting results for the case when the dataset exhibits sparsity in frames.

    \subsubsection{Frames} 
    In practical applications, sparsity typically arises with respect to frames rather than bases, as representations in redundant systems (i.e., frames) often yield sparser representations.

    Let now $\Xi \coloneqq \{\xi_k\}_{k \in \mathcal K}$ be a frame for $\mathbb R^M$, where $\mathcal K$ is a countable index set. That is, there exist constants $0<A\leq B< \infty$ such that \cite{christensen2003introduction} 
    \begin{align}\label{eq:frame-cond}
        A \lVert f \rVert^2 \leq \sum_{k \in \mathcal{K}} \lvert \innerprod{f}{\xi_k} \rvert^2 \leq B \lVert f \rVert^2, \quad \text{for all $f \in \mathbb R^M$}.
    \end{align}
    If $\mathcal{K}$ is finite, \cref{eq:frame-cond} is equivalent to simply $\spanRinline{\Xi} = \mathbb R^M$. 
    Fix $s \in \{1,\ldots,M\}$, and let $\{S_j\}_{j \in \mathcal J}$ be the set of all pairwise distinct index subsets $S_j \subseteq \mathcal{K}$ with $\lvert S_j \rvert =s$. Assume that $\{\xi_k\}_{k \in S_j}$ is linearly independent for every $j \in \mathcal{J}$. Denote by $E \coloneqq \bigcup_{j \in \mathcal{J}} E_j$ the set of $s$-sparse vectors in the frame $\Xi$, where $E_j \coloneqq \spanRinline{\{\xi_k\}_{k \in S_j}}$. Furthermore, let 
    \begin{align*}
        \pi_j \colon E_j \to \mathbb R^s, \,\, \sum_{i=1}^s c_i\xi_{k_{j,i}} \mapsto c
    \end{align*}
    be the map from $E_j$ to the space of expansion coefficients with the labeling $S_j=\{k_{j,i}\}_{i=1}^s$.  
    Consider an $N$-point set $F \coloneqq \{f_1, \ldots, f_N\} \subset E$, where $N \in \mathbb N$. The set of points in the subspace $E_j$ are written as $F_j \coloneqq F\cap E_j$, $j \in \mathcal{J}$. Similarly to the basis case, we impose the following very mild assumption; the only difference is that in the frame case, we may have $E_i=E_j$ for $i,j \in \mathcal{J}$ with $i\neq j$. As will be shown in \cref{rem:assumptions-almost-surely-satisfied}, for $(\mathcal{H}^s)^N$-a.e. $N$-tuple $(f_1,\ldots,f_N) \in E^N$ the corresponding $N$-point set $\{f_1,\ldots,f_N\}$ satisfies this assumption. 
    \begin{assumption}\label{A:sep-sparse-frame}
    Assume that for each $j \in \mathcal{J}$, whenever $F_j \neq \emptyset$,
    \begin{enumerate}[label=(f-\roman*)]
        \item\label{it:sep-sp-f-i} $F_j$ is in $\pi_{j}$-general position,
        \item\label{it:sep-sp-f-ii} $F_j \cap E_i = \emptyset$, for every $i \in \mathcal J$ with $i \neq j$ whenever $E_i \neq E_j$.
    \end{enumerate}
    \end{assumption}
    Informally, \cref{it:sep-sp-f-i} guarantees that in each subspace $E_j$ the points are in general position. By \cref{it:sep-sp-f-ii} we ensure that no vector in $F$ is $(s-1)$-sparse, so that the assignment of the points in $F$ to the subspaces $(E_j)_{j\in \mathcal{J}}$ is unique.
    Writing $N_j$ for the number of points of $F$ which are in $E_j$, $j \in \mathcal{J}$, and setting $\underline{N}=(N_j)_{j\in\mathcal{J}}$, we have the following result.
    \begin{proposition}[Sparsity in a frame]\label{prop:sep-sparse-frame}
        Under \cref{A:sep-sparse-frame} the number of homogeneously linearly separable dichotomies of $F$ is given by
        \begin{align*}
            C_{\mathrm{sp,f}}(\underline{N},\Xi,s) \coloneqq 2^N - 2 \sum_{t=s+1}^N \sum_{\underline{\nu} \in \mathcal{I}^{\mathrm{sp,f}}_t} \prod_{j \in \mathcal{J}_0}\binom{N_j}{\nu_j},
        \end{align*}
        where $\mathcal{J}_0 \coloneqq \supp{\underline{N}} \subseteq \mathcal{J}$ is a finite subset, and where $\mathcal{I}^{\mathrm{sp,f}}_t \coloneqq \{\underline \nu \in \mathbb N_0^{\mathcal{J}_0} \colon \lenin{\underline{\nu}} = t, \Upsilon^{\mathrm{sp,f}}(\underline \nu) \not\equiv t \Mod{2}\}$ with
        \begin{align*}
            \Upsilon^{\mathrm{sp,f}}(\underline \nu) \coloneqq \min_{\mathfrak{s} \subseteq \suppinline{\underline{\nu}}} \left\{\sum_{j \in \mathfrak{s}} \nu_j+\dimR{\sum_{j \in \mathfrak{s}^c} E_j} \right\},\quad \underline \nu = (\nu_j)_{j \in \mathcal{J}_0} \in \mathbb N_0^{\mathcal{J}_0}.
        \end{align*}
    \end{proposition}
    \begin{proof}
        We apply \cref{thm:winder}, and, to this end, consider the $(M-1)$-dimensional hyperplanes $H_k \coloneqq \{f_k\}^\perp$, $k \in \{1,\ldots,N\}$, and count the number of even- and odd-degenerate sets of these hyperplanes, denoted by $\lvert \mathcal E \rvert$ and $\lvert \mathcal{O} \rvert$, respectively. As $\lvert \mathcal E \rvert + \lvert \mathcal O \rvert = 2^N$, see \cref{rem:number-even-odd-deg}, we have, by \cref{thm:winder}, that the number of regions into which $\{H_k\}_{k=1}^N$ divide $\mathbb R^M$ is given by
        \begin{align}\label{eq:winder-sp-frame}
            \lvert \mathcal E \rvert - \lvert \mathcal{O} \rvert = 2^N - 2\lvert \mathcal{O} \rvert.
        \end{align}
        It hence suffices to determine the number of odd-degenerate sets of hyperplanes. To this end, consider an arbitrary set of hyperplanes $\{H_k\}_{k \in \mathcal K}$ with $\mathcal K \subseteq \{1,\ldots,N\}$ and $\lvert \mathcal K \rvert = t$, and compute
        \begin{align}
            \dimR{\bigcap_{k \in \mathcal K} H_k} &= 
            \dimR{\left( \spanR{\{f_k\}_{k \in \mathcal{K}}}\right)^\perp} \nonumber\\ 
            &= M - \dimR{\spanR{\{f_k\}_{k \in \mathcal K}}}. \label{eq:sep-sparse-frame-1}
        \end{align}
        We first note that, by \cref{A:sep-sparse-frame}, every set of $t\leq s$ vectors $\{f_k\}_{k \in \mathcal K}$ is linearly independent, which implies $\dimR{\bigcap_{k \in \mathcal K} H_k} = M-t$, and hence all sets of $t$ hyperplanes with $t \leq s$ are even-degenerate. 
        Consider next the case $t \geq s+1$, and decompose the index set $\mathcal K$ into disjoint sets\footnote{Specifically, we allow $\mathcal{K}_j = \emptyset$, $j \in \mathcal{J}_0$.} $\{\mathcal K_j\}_{j\in\mathcal{J}_0}$ such that $\{f_k\}_{k \in \mathcal K_j} \subset E_j$, for all $j \in \mathcal{J}_0$. Then $\lvert \mathcal K_j \rvert \eqqcolon\nu_j$ is equal to the number of vectors in $\{f_k\}_{k \in \mathcal K}$ which belong to the $s$-dimensional linear subspace $E_j$, $j \in \mathcal{J}_0$. Setting $\underline{\nu} = (\nu_j)_{j\in\mathcal{J}_0} \in \mathbb N_0^{\mathcal{J}_0}$, we claim 
        \begin{align}\label{eq:sep-sparse-frame-2}
            \dimR{\spanR{\{f_k\}_{k \in \mathcal K}}} = \min_{\mathfrak{s}\subseteq \suppinline{\underline{\nu}}} \left\{\sum_{j \in \mathfrak{s}} \nu_j + \dimR{\sum_{j \in \mathfrak{s}^c} E_j}\right\} \eqqcolon \Upsilon(\underline{\nu}). 
        \end{align}
        Indeed, \cref{eq:sep-sparse-frame-2} follows by induction on $\lvert \supp{\underline{\nu}} \rvert$.
        For the base case $\lvert \supp{\underline{\nu}} \rvert = 1$, let $j_0$ be the unique element of $\supp{\underline{\nu}}$. Then, $\nu_j = 0$ for all $j \neq j_0$. By \cref{A:sep-sparse-frame}\ref{it:sep-sp-f-i}, we therefore get
        \begin{align*}
            \dimR{\spanR{\{f_k\}_{k \in \mathcal K}}} = \dimR{\spanR{\{f_k\}_{k \in \mathcal K_{j_0}}}} = \min\{\nu_{j_0},s\}. 
        \end{align*}
        Moreover, we observe that  
        \begin{align*}
            \min_{\mathfrak{s}\subseteq \{j_0\}} \left\{\sum_{j \in \mathfrak{s}} \nu_j + \dimR{\sum_{j \in \mathfrak{s}^c} E_j}\right\} = \min\{\nu_{j_0},s\},
        \end{align*}
        which confirms the base case. Proceeding to the induction step, now assume that \cref{eq:sep-sparse-frame-2} holds for all $\underline{\nu} \in \mathbb N_0^{\mathcal{J}_0}$ with $\lvert \supp{\underline{\nu}} \rvert \leq r$ for some $r \in \{1,\ldots,\lvert \mathcal{J}_0\rvert-1\}$. 
        By \cref{A:sep-sparse-frame}, for every $\mathfrak{s} \subseteq \mathcal{J}_0$ and $\{\mathfrak{k}_j\}_{j \in \mathfrak{s}}$ with $\mathfrak{k}_j \subseteq \mathcal{K}_j$ and $\lvert \mathfrak{k}_j \rvert \leq s$, 
        \begin{align}\label{eq:pi_s-gen-pos}
            \bigcup_{j \in \mathfrak{s}} \{f_k\}_{k \in \mathfrak{k}_j} \text{ is in $\pi_\mathfrak{s}$-general position,}
        \end{align}
        where $\pi_\mathfrak{s}\colon \sum_{j \in \mathfrak{s}}E_j \to \mathbb R^{d_{\mathfrak{s}}}$ is the linear embedding of $\sum_{j \in \mathfrak{s}}E_j$ into $\mathbb R^{d_{\mathfrak{s}}}$ with $d_\mathfrak{s} \coloneqq \dimR{\sum_{j \in \mathfrak{s}}E_j}$. As a result of \cref{eq:pi_s-gen-pos}, we have, for every $j_0 \in \supp{\underline{\nu}}$,
        \begin{align*}
            &\dimR{\spanR{\{f_k\}_{k \in \mathcal K}}} \\
            &= \min\!\left\{\dimR{\spanR{\{f_k\}_{k \in \mathcal K\setminus\mathcal{K}_{j_0}}}}+ \dimR{\spanR{\{f_k\}_{k \in \mathcal K_{j_0}}}}\!,\dimR{\sum_{j \in \suppinline{\underline{\nu}}}E_j}\right\}.
        \end{align*}
        Denoting by $\underline{\nu}^{\setminus j_0}$ the multi-index $\underline{\nu}$ with $\nu_{j_0}$ set to zero, we obtain, under the induction hypothesis,
        \begin{align*}
            \dimR{\spanR{\{f_k\}_{k \in \mathcal K}}} = \min\!\left\{\Upsilon(\underline{\nu}^{\setminus j_0})+ \min\{\nu_{j_0},s\},\dimR{\sum_{j \in \suppinline{\underline{\nu}}}E_j}\right\},
        \end{align*}
        for all $j_0 \in \supp{\underline{\nu}}$.
        In particular,
        \begin{align}\label{eq:dim-span-recursive}
            \begin{split}
            \dimR{\spanR{\{f_k\}_{k \in \mathcal K}}} &= \min\!\left\{\min_{j_0 \in \suppinline{\underline{\nu}}}\left\{\Upsilon(\underline{\nu}^{\setminus j_0})+ \nu_{j_0}\right\},\dimR{\sum_{j \in \suppinline{\underline{\nu}}}E_j}\right\}.
            \end{split}
        \end{align}
        We now show that the RHS of \cref{eq:dim-span-recursive} is equal to $\Upsilon(\underline{\nu})$. To this end, first note that, by choosing $\mathfrak{s}=\emptyset$ in the definition of $\Upsilon(\underline{\nu})$, one obtains 
        \begin{align}\label{eq:dim-span-recursive1}
            \sum_{j \in \emptyset} \nu_j + \dimR{\sum_{j \in \suppinline{\underline{\nu}}\setminus\emptyset} E_j} = \dimR{\sum_{j \in \suppinline{\underline{\nu}}}E_j}. 
        \end{align}
        Second, we have
        \begin{align}\label{eq:dim-span-recursive2}
        \begin{split}
            &\min_{\substack{\mathfrak{s}\subseteq \suppinline{\underline{\nu}} \\ \mathfrak{s} \neq \emptyset}} \left\{\sum_{j \in \mathfrak{s}} \nu_j + \dimR{\sum_{j \in \mathfrak{s}^c} E_j}\right\} \\&= \min_{j_0 \in \suppinline{\underline{\nu}}}\left\{ \min_{\mathfrak{s}\subseteq \suppinline{\underline{\nu}^{\setminus j_0}}} \left\{\sum_{j \in \mathfrak{s}} \nu_j + \dimR{\sum_{j \in \mathfrak{s}^c} E_j}\right\}+ \nu_{j_0} \right\}  \\
            &=\min_{j_0 \in \suppinline{\underline{\nu}}}\left\{\Upsilon(\underline{\nu}^{\setminus j_0})+ \nu_{j_0}\right\}.
        \end{split}
        \end{align}
        From \cref{eq:dim-span-recursive1,eq:dim-span-recursive2} one observes that all possible candidates for the minimum in the definition of $\Upsilon(\underline \nu)$ are captured by all candidates for the minimum in the RHS of \cref{eq:dim-span-recursive}.
        This shows that the RHS of \cref{eq:dim-span-recursive} is equal to $\Upsilon(\underline{\nu})$, completing the proof of \cref{eq:sep-sparse-frame-2}.

        Thus, by \cref{eq:sep-sparse-frame-1}, the set $\{H_k\}_{k \in \mathcal K}$ is odd-degenerate if and only if 
        \begin{align*}
            \Upsilon^{\mathrm{sp,f}}(\underline{\nu}) \not\equiv t \Mod{2}.
        \end{align*}
        Given $\underline{N}$, it follows that the number of all odd-degenerate sets $\{H_k\}_{k \in \mathcal K}$ with $\lvert \mathcal{K}_j\rvert \leq N_j$, $j \in \mathcal{J}_0$, is given by
        \begin{align*}
            \lvert \mathcal{O}\rvert = \sum_{t=s+1}^N\sum_{\underline{\nu} \in \mathcal{I}^{\mathrm{sp,f}}_t} \prod_{j \in \mathcal{J}_0}\binom{N_j}{\nu_j}.
        \end{align*}
        Finally, application of \cref{thm:winder} in the form \cref{eq:winder-sp-frame} yields the desired expression. 
    \end{proof}
    As in the case of sparsity in a basis (see \cref{prop:sep-sparse-basis}), the number of homogeneously linearly separable dichotomies is independent of the specific location in $E$ of the points in $F$; it is determined by the configuration $\underline{N}$ only.  
    In contrast, however, to \cref{prop:sep-sparse-basis}, $C_{\mathrm{sp,f}}(\underline{N},\Xi,s)$ now depends on the specific choice of the frame elements in $\Xi$. In the following the effect of the frame elements will be analyzed. To this end, we introduce the \emph{Gram operator associated with $\Xi$} 
    \begin{align}\label{eq:gram-operator}
        G \colon \ell^2(\mathcal{K}) \to \ell^2(\mathcal{K}), \,\,(c_k)_{k\in\mathcal{K}}\mapsto \left(\innerprod{\sum_{j\in\mathcal{K}}c_j\xi_j}{\xi_k}\right)_{k \in \mathcal{K}}.
    \end{align}
    Here, $(\ell^2(\mathcal{K}),\innerprod{\cdot}{\cdot}_{\ell^2(\mathcal{K})})$ denotes the space of square-summable sequences indexed by $\mathcal K$, equipped with the standard inner product $\innerprod{c}{d}_{\ell^2(\mathcal{K})} \coloneqq \sum_{k \in \mathcal{K}}c_kd_k$, $c,d\in \ell^2(\mathcal{K})$. As $\Xi$ constitutes a frame in the sense of \cref{eq:frame-cond}, the linear operator $G$ is bounded \cite[Lemma 3.5.1]{christensen2003introduction}, and moreover, $G$ is self-adjoint. The Gram operator can be leveraged to compare the number of realizable dichotomies of point sets of datasets exhibiting sparsity in different frames.
    \begin{lemma}\label{L:gram-nondecreasing}
        Let $\Xi= \{\xi_k\}_{k \in \mathcal K},\Xi'= \{\xi_k'\}_{k \in \mathcal K}$ be frames for $\mathbb R^M$ indexed by the same set $\mathcal K$, and let $G,G'$ be the Gram operators associated with $\Xi,\Xi'$, respectively. If $(G'-G)$ is positive semidefinite\footnote{That is, $\innerprod{(G'-G)c}{c}_{\ell^2(\mathcal{K})}\geq0$, for all $c \in \ell^2(\mathcal{K})$.}, then for all $
        N \in \mathbb N$ and all $\underline{N} \in \mathbb N_0^{ \mathcal{J}}$ with $\lvert\underline{N}\rvert = N$,
        \begin{align*}
            C_{\mathrm{sp,f}}(\underline{N},\Xi,s) \leq C_{\mathrm{sp,f}}(\underline{N},\Xi',s).
        \end{align*}
    \end{lemma}

    \begin{proof}
        We first show that
        \begin{align}\label{eq:subspace-dominance1}
            \dimR{\spanR{\{\xi_k\}_{k \in S}}} \leq \dimR{\spanR{\{\xi_k'\}_{k \in S}}}, \quad \text{for every finite $S\subseteq \mathcal{K}$}. 
        \end{align}
        To this end, let $S\subseteq \mathcal{K}$ be finite, and consider the matrices $G_S \coloneqq (\innerprod{\xi_j}{\xi_k})_{j,k \in S}$ and $G_S' \coloneqq (\innerprod{\xi'_j}{\xi'_k})_{j,k \in S}$. The assumption on the positive semidefiniteness of $(G'-G)$ implies that the matrix $(G_S'-G_S)$ is positive semidefinite. Indeed,
        \begin{align*}
            \left(\innerprod{(G'-G)c}{c} \geq 0, \, \forall c \in \ell^2(\mathcal{K})\right) &\implies \left(\innerprod{(G'-G)c}{c} \geq 0, \, \forall c \in \ell^2(\mathcal{K}), \supp{c}\subseteq S\right) \\&\implies \left(\innerprod{(G_S'-G_S)c}{c} \geq 0, \, \forall c=(c_k)_{k \in S} \in \mathbb R^{S}\right), 
        \end{align*}
        where $\supp{c} \coloneqq \{k \in \mathcal{K}\colon c_k \neq 0\}$.
        It follows that $\mathrm{rank}(G_S) \leq \mathrm{rank}(G_S')$, see, e.g., \cite[Corollary 7.7.4(c)]{horn2012matrix}. But as $G_S, G_S'$ are Gram matrices of $\{\xi_k\}_{k\in S}$ and $\{\xi'_k\}_{k\in S}$, respectively, we have, by \cite[Theorem 7.2.10(c)]{horn2012matrix}, $\mathrm{rank}(G_S)=\dimR{\spanRinline{\{\xi_k\}_{k \in S}}}$ and $\mathrm{rank}(G_S')=\dimR{\spanRinline{\{\xi'_k\}_{k \in S}}}$. This establishes \cref{eq:subspace-dominance1}.

        Fix now $\underline{N}\in \mathbb N_0^{\mathcal{J}}$ with $\mathcal{J}_0 = \supp{\underline{N}}$, and let $E'_j\coloneqq \spanRinline{\{\xi'_k\}_{k \in S_j}}$, $j \in \mathcal{J}$. Then, 
        \begin{align}\label{eq:subspace-dominance2}
            \dimR{\sum_{j \in \mathfrak{s}} E_j} \leq \dimR{\sum_{j \in \mathfrak{s}} E'_j}, \quad \text{for all } \mathfrak{s} \subseteq \mathcal{J}_0.
        \end{align}
        Indeed, \cref{eq:subspace-dominance2} follows from \cref{eq:subspace-dominance1} upon noting that
        \begin{align*}
            \dimR{\sum_{j \in \mathfrak{s}} E_j} &= \dimR{\spanR{\{\xi_k\}_{k\in \bigcup_{j \in \mathfrak{s}} S_j}}}\quad\text{and} \\
            \dimR{\sum_{j \in \mathfrak{s}} E'_j} &= \dimR{\spanR{\{\xi'_k\}_{k\in \bigcup_{j\in \mathfrak{s}} S_j}}}. 
        \end{align*} 
        Thanks to \cref{eq:subspace-dominance2}, application of \cref{lem:monotonicity} with $\{E_j\}_{j \in \mathcal{J}_0}$, $\{E'_j\}_{j \in \mathcal{J}_0}$, and $\phi = \mathrm{Id}\colon \mathbb N_0^{\mathcal{J}_0} \to \mathbb N_0^{\mathcal{J}_0}$ yields the desired bound.
    \end{proof}
    \cref{L:gram-nondecreasing} establishes that the number of realizable dichotomies is nondecreasing with respect to the ordering induced by the Gram operator associated with the frame. 
    
    \paragraph{Connection to compressed sensing} Compressed sensing \cite{donoho2003optimally,candes2006robust,donoho2006compressed,foucart2013invitation} is concerned with the recovery of sparse vectors in $\mathbb R^{\lvert \mathcal{K}\rvert}$ from a number of linear measurements $M$ that is small relative to $\lvert \mathcal{K}\rvert$, assuming $\lvert \mathcal{K}\rvert<\infty$. More concretely, in our notation, the goal is to reconstruct $(c_k)_{k \in \mathcal{K}} \in \mathbb R^{\mathcal{K}}$ from $M$ measurements of the form $f=\sum_{k \in \mathcal{K}}c_k \xi_k$, where $(c_k)_{k \in \mathcal{K}}$ has at most $s$ nonzero entries. Deterministic recovery guarantees typically rely on properties of the Gram operator associated with $\Xi$. In particular, a straightforward, yet computationally hard to verify, sufficient recovery condition is based on the so-called spark of the frame $\Xi$, see \cite{donoho2003optimally}. The spark of $\Xi$, denoted $\spark{\Xi}$, is defined as the cardinality of the smallest subset of $\Xi$ which is linearly dependent. Note that $\mathrm{spark}(\Xi) = \kr{\Xi}+1$. The recovery guarantee states that if $s < \spark{\Xi}/2$, then $(c_k)_{k \in \mathcal{K}}$ can be uniquely recovered through a combinatorial search according to
    \begin{align*}
        \operatorname*{argmin}_{\widetilde c \in \mathbb R^{ \mathcal{K}}}\, \lVert \widetilde c\rVert_0 \quad \text{ subject to } \quad \sum_{k \in \mathcal{K}} \widetilde c_k \xi_k = f, 
    \end{align*}
    where $\lVert \widetilde c\rVert_0$ denotes the number of nonzero entries of $\widetilde c$. Considering now the setup of \cref{L:gram-nondecreasing}, we note that, as demonstrated in the proof of \cref{L:gram-nondecreasing}, $(G'-G)$ being positive semidefinite implies \cref{eq:subspace-dominance1}, and from \cref{eq:subspace-dominance1} we can deduce that $\spark{\Xi} \leq \spark{\Xi'}$. Thus, recalling the recovery threshold $s < \spark{\Xi}/2$, the recovery performance under measurements with respect to $\Xi'$ is guaranteed to be at least as good as that under $\Xi$. Comparing to \cref{L:gram-nondecreasing}, we have the analogous result in the context of separation (rather than recovery). Namely, performing homogeneous linear separation on the set of sparse vectors with respect to $\Xi'$ yields at least as many realizable dichotomies as with respect to $\Xi$.          

    \paragraph{$\Phi$-separability} Finally, let us briefly comment on $\Phi$-separability of point sets on the set of sparse vectors. If $\Phi \colon E \to \mathbb R^{M'}$ is a linear map, then $\Phi(\sum_{k\in S_j}c_k \xi_k) = \sum_{k\in S_j}c_k \Phi(\xi_k)$, for all $(c_k)_{k\in S_j} \in \mathbb R^{S_j}$. So, in particular, we have 
    \begin{align*}
        \Phi(E) = \bigcup_{j\in \mathcal{J}} \Phi(E_j) = \bigcup_{j \in \mathcal{J}} \spanR{\{\Phi(\xi_k)\}_{k \in S_j}}.
    \end{align*}
    In words, the union of linear subspaces structure is preserved under linear maps. If, moreover, the linear map $\Phi$ has full rank, the sparsity level $s$ remains unchanged. That is, $\Phi(E)$ constitutes again a set of $s$-sparse vectors, but now with respect to a different frame, namely, $\{\Phi(\xi_k)\}_{k \in \mathcal{K}}$. To determine the number of $\Phi$-separable dichotomies of the $N$-point set $F \subset E$, one may thus apply \cref{prop:sep-sparse-frame}. 
    For nonlinear transformations $\Phi$, the structure of sparsity is generally not preserved, and depending on $\Phi$, we obtain a different data structure. This case will be discussed in the next subsections.
\subsection{Rectifiable sets}\label{subsec:sep-rect}
The set of sparse vectors forms a special case within the broader class of rectifiable sets. In this subsection, we generalize our analysis to separation on rectifiable sets. We begin by stating the definition of rectifiable sets.
\begin{definition}[Countably $\mathcal{H}^s$-rectifiable set, \cite{ambrosio2000currents,federer2014geometric,simon2014introduction}]
    An $\mathcal{H}^s$-measurable set $E \subseteq \mathbb R^M$ is said to be \emph{countably $\mathcal H^s$-rectifiable} if there is a countable family of Lipschitz maps $\{\psi_k \colon \mathbb R^s \to \mathbb R^M\}_{k \in \mathbb N}$ such that
            \begin{align*}
                \mathcal H^s \left(E \setminus \bigcup_{k \in \mathbb N} \psi_k(\mathbb R^s) \right) =0.
            \end{align*}
\end{definition}
Countably $\mathcal{H}^s$-rectifiable sets also admit the following useful parametrization.
\begin{lemma}[Bi-Lipschitz parametrization, \cite{ambrosio2000currents,federer2014geometric}]\label{L:bi-Lip}
    Let $E \subseteq \mathbb R^M$ be countably $\mathcal H^s$-rectifiable. Then there exist finitely or countably many compact sets $K_j \subset \mathbb R^s$ and bi-Lipschitz maps $\psi_j \colon K_j \to \psi_j(K_j) \subseteq E$, indexed by $\mathcal{J}$, such that $\{\psi_j(K_j)\}_{j \in \mathcal{J}}$ are pairwise disjoint and 
    \begin{align*}
        \mathcal H^s \left(E \setminus \bigcup_{j \in \mathcal{J}} \psi_j(K_j) \right) =0.
    \end{align*}
\end{lemma}
Let $E \subseteq \mathbb R^M$ be a countable $\mathcal{H}^s$-rectifiable set with bi-Lipschitz parametrization $\{\psi_j \colon K_j \to\psi_j(K_j)\}_{j \in \mathcal{J}}$ as in \cref{L:bi-Lip}, and define $\{s_j\}_{j \in \mathcal{J}}$ according to 
\begin{align*}
    s_j \coloneqq \dimR{\spanR{\psi_j(K_j)}}, \quad j \in \mathcal{J}.
\end{align*}
We further introduce the map $\pi_j \colon \psi_j(K_j) \to \mathbb R^{s_j}$ as a linear embedding of $\psi_j(K_j)$ into $\mathbb R^{s_j}$, $j \in \mathcal J$. Concretely, choosing a set of $s_j$ linearly independent vectors in $\psi_j(K_j)$, we define $\pi_j$ as the linear map that assigns to each $f \in \psi_j(K_j)$ the vector of its expansion coefficients $\pi_j(f) \in \mathbb R^{s_j}$ with respect to the chosen $s_j$ linearly independent vectors. The extension of the map $\pi_j$ to $\spanR{\psi_j(K_j)}$ shall be denoted by $\widetilde\pi_j$. 
The following remark provides a simple lower bound for the quantity $s_j$, which will play a crucial role in our analysis.
\begin{remark}\label{rem:lower-bound-by-rect-para}
    First observe that without loss of generality one may assume that $\mathcal H^s(\psi_j(K_j))>0$, as otherwise $\psi_j$ can be omitted from the bi-Lipschitz parametrization, and we still retain a bi-Lipschitz parametrization in the sense of \cref{L:bi-Lip}. 
    Application of \cref{prop:hausdorff-meas} then establishes that $s_j \geq s$. 
\end{remark}

Consider now the $N$-point set $F \coloneqq \{f_1,\ldots,f_N\} \subset E$ with $N \in \mathbb N$ and denote by $F_j \coloneqq F \cap \psi_j(K_j)$ the points of $F$ in the bi-Lipschitz image $\psi_j(K_j)$.
In line with the philosophy of \cref{A:sep-sparse-basis,A:sep-sparse-frame}, we impose the following very mild assumption, and, as we shall see in \cref{rem:assumptions-almost-surely-satisfied}, the bi-Lipschitz parametrization of $E$ can be chosen such that for $(\mathcal{H}^s)^N$-almost every $N$-tuple $(f_1,\ldots,f_N)$, the corresponding $N$-point set $\{f_1,\ldots,f_N\}$ satisfies this assumption.   
\begin{assumption}\label{A:sep-rect} Suppose that $F \subset \bigcup_{j \in \mathcal{J}} \psi_j(K_j)$ and, for every $j \in \mathcal{J}$, whenever $F_j \neq \emptyset$, we assume the following:
\begin{enumerate}[label=(r-\roman*)] 
    \item\label{it:sep-rect-i} $F_j$ is in $\pi_j$-general position. 
    \item\label{it:sep-rect-ii} If there is an index set $\mathcal J' \subseteq \mathcal J$ such that $\spanR{\psi_i(K_i)} = \spanR{\psi_j(K_j)}$, for all $i \in \mathcal{J}'$, then $\bigcup_{i \in \mathcal{J}'}F_i$ is in $\widetilde\pi_j$-general position. 
    \item\label{it:sep-rect-iii} $F_j \cap \sum_{i \in \mathcal J''}\spanR{\psi_i(K_i)} = \emptyset$ whenever $\psi_j(K_j) \not \subseteq \sum_{i \in \mathcal J''}\spanR{\psi_i(K_i)}$, for every $\mathcal{J}'' \subset \mathcal{J}$.  
\end{enumerate}
\end{assumption}
\cref{it:sep-rect-i} means that the points of $F$ in $\psi_j(K_j)$ are in general position when embedded into $\mathbb R^{s_j}$. \cref{it:sep-rect-ii} ensures that, whenever multiple images $\{\psi_i(K_i)\}_{i \in \mathcal{J}'}$ span the same subspace as $\psi_j(K_j)$, the points of $F$ in $\{\psi_i(K_i)\}_{i \in \mathcal{J}'}$ are in general position with respect to the common embedding $\widetilde \pi_j$. This prevents degeneracies within shared subspaces. Finally, \cref{it:sep-rect-iii} prevents the points in $F_j$ from lying in a lower-dimensional subspace of $\spanR{\psi_j(K_j)}$ that is also contained in the span of other components $\{\psi_i(K_i)\}_{i \in \mathcal{J}''}$, unless $\psi_j(K_j)$ itself is entirely contained in $\sum_{i \in \mathcal{J}''} \spanR{\psi_i(K_i)}$.
In contrast to \cref{A:sep-sparse-frame}, \cref{it:sep-rect-i,it:sep-rect-ii} constitute a refined version of \cref{it:sep-sp-f-i}, while \cref{it:sep-rect-iii} extends \cref{it:sep-sp-f-ii}.
With $\underline{N} \coloneqq (N_j)_{j \in \mathcal J}$ and $N_j = \vert F_j \rvert$, we have the following result.
\begin{proposition}\label{prop:sep-rect}
    Under \cref{A:sep-rect} the number of homogeneously separable dichotomies of $F$ is given by
     \begin{align*}
            C_{\mathrm{r}}(\underline{N},E) \coloneqq 2^N - 2 \sum_{t=s+1}^N \sum_{\underline{\nu} \in \mathcal{I}^{\mathrm{r}}_t} \prod_{j \in \mathcal{J}_0}\binom{N_j}{\nu_j},
        \end{align*}
        where $\mathcal{J}_0 \coloneqq \supp{\underline{N}} \subseteq \mathcal{J}$ is a finite subset, and where $\mathcal{I}^{\mathrm{r}}_t \coloneqq \{\underline \nu  \in \mathbb N_0^{ \mathcal{J}_0} \colon \lenin{\underline{\nu}} = t, \Upsilon^{\mathrm{r}}(\underline \nu) \not\equiv t \Mod{2}\}$ with
        \begin{align*}
            \Upsilon^{\mathrm{r}}(\underline \nu) \coloneqq \min_{\mathfrak{s} \subseteq \suppinline{\underline{\nu}}} \left\{\sum_{j \in \mathfrak{s}} \nu_j+\dimR{\spanR{\bigcup_{j \in \mathfrak{s}^c} \psi_j(K_j)}} \right\},\quad \underline \nu = (\nu_j)_{j \in \mathcal{J}_0} \in \mathbb N_0^{ \mathcal{J}_0 }.
        \end{align*}
\end{proposition}
\begin{proof}
    By particularizing \cref{prop:sep-general-case} to $\Phi = \mathrm{Id}$ and $E_j = \psi_j(K_j)$, $j \in \mathcal{J}_0$, the desired expression for the number of homogeneously separable dichotomies of $F$ follows.
\end{proof}
Let us now analyze the impact of the geometry of $E$ on the number of realizable dichotomies. To this end, let $K_{i,j} \coloneqq K_i \times K_j$, for $i,j \in \mathcal J$, and define the kernel functions $\kappa_{i,j}$ associated with $\{\psi_{j}\colon K_j\to\psi_j(K_j)\}_{j \in \mathcal{J}}$ according to 
\begin{align*}
    \kappa_{i,j}\colon K_{i,j} \times K_{i,j} \to \mathbb R,\quad (x,y) \mapsto \frac 12 \left(\innerprod{\psi_{i}(x_1)}{\psi_j(y_2)} + \innerprod{\psi_j(x_2)}{\psi_{i}(y_1)} \right). 
\end{align*}
Note that $\kappa_{i,j}$ is symmetric in the sense that $\kappa_{i,j}(x,y) = \kappa_{i,j}(y,x)$, for $x,y \in K_{i,j}$.  
\begin{lemma}\label{lem:rect-monotonicity}
    Let $E,E'$ be countable $\mathcal{H}^s$-rectifiable sets with bi-Lipschitz parametrizations $\{\psi_{j}\colon K_j\to\psi_j(K_j)\}_{j \in \mathcal{J}}$, $\{\psi_{j}'\colon K_j\to\psi'_j(K_j)\}_{j \in \mathcal{J}}$ and kernel functions $\{\kappa_{i,j}\}_{i,j \in \mathcal{J}},\{\kappa'_{i,j}\}_{i,j \in \mathcal{J}}$, respectively. Suppose that $(\kappa'_{i,j}-\kappa_{i,j})$ is positive semidefinite\footnote{That is, for every $n \in \mathbb N$, $\{x^{(k)}\}_{k=1}^n \subseteq K_{i,j}$, and $\{c_k\}_{k=1}^n \subseteq \mathbb R$, we have $\sum_{k=1}^n\sum_{\ell=1}^n c_k c_\ell (\kappa'_{i,j}(x^{(k)},x^{(\ell)}) - \kappa_{i,j}(x^{(k)},x^{(\ell)})) \geq 0$.}, for all $i,j \in \mathcal{J}$. Then, for all $N \in \mathbb N$ and all $\underline{N} \in \mathbb N_0^\mathcal{J}$ with $\vert \underline{N}\rvert = N$,
    \begin{align*}
        C_{\mathrm{r}}(\underline{N},E) \leq C_{\mathrm{r}}(\underline{N},E').
    \end{align*}
\end{lemma}
\begin{proof}
    We begin by showing that for every $S \subseteq \mathcal{J}$,
    \begin{align}\label{eq:dominance-rect}
        \dimR{\spanR{\bigcup_{j \in S} \psi_j(K_j)}} \leq \dimR{\spanR{\bigcup_{j \in S} \psi_j'(K_j)}}.
    \end{align}
    To see this, let $S_0 \subseteq S$ be a finite subset and, for each $j \in S_0$, choose a finite set $\{x_{j,k}\}_{k \in \mathfrak s} \subseteq K_j$ such that
    \begin{align}\label{eq:span-rect-set}
        \dimR{\spanR{\left\{\psi_{j}(x_{j,k}) \colon j \in S_0, k \in \mathfrak{s}\right\}}}&=\dimR{\spanR{\bigcup_{j \in S} \psi_j(K_j)}}
        \intertext{and}\label{eq:span-rect-set'}
        \dimR{\spanR{\left\{\psi'_{j}(x_{j,k}) \colon j \in S_0, k \in \mathfrak{s}\right\}}}&=\dimR{\spanR{\bigcup_{j \in S} \psi'_j(K_j)}}.
    \end{align}
    Consider the Gram matrices $G \coloneqq (\innerprod{\psi_i(x_{i,k})}{\psi_j(x_{j,\ell})})_{\substack{i,j \in S_0 \\ k,\ell \in \mathfrak s}} \in \mathbb R^{\mathfrak{n} \times \mathfrak n}$ and $G' \coloneqq (\innerprod{\psi'_i(x_{i,k})}{\psi'_j(x_{j,\ell})})_{\substack{i,j \in S_0 \\ k,\ell \in \mathfrak s}} \in \mathbb R^{\mathfrak{n} \times \mathfrak n}$, where $\mathfrak{n} \coloneqq S_0 \times \mathfrak{s}$. Setting $x^{(k)}_{i,j} \coloneqq (x_{i,k},x_{j,k})$, we obtain, for $c=(c_{j,k})_{j \in S_0, k \in \mathfrak{s}} \in \mathbb R^{\mathfrak{n}}$, 
    \begin{align*}
        c^{\mathsf{T}}(G'-G)c &=\sum_{\substack{i,j \in S_0 \\ k,\ell \in \mathfrak s} }c_{i,k}c_{j,\ell} \left(\innerprod{\psi'_i(x_{i,k})}{\psi'_j(x_{j,\ell})} - \innerprod{\psi_i(x_{i,k})}{\psi_j(x_{j,\ell})}\right) \\ 
        &=\sum_{\substack{i,j \in S_0 \\ k,\ell \in \mathfrak s} }c_{i,k}c_{j,\ell} \bigg(\frac 12 \innerprod{\psi'_i(x_{i,k})}{\psi'_j(x_{j,\ell})}+\frac 12 \innerprod{\psi'_j(x_{j,k})}{\psi'_i(x_{i,\ell})} \\&\quad - \frac 12 \innerprod{\psi_i(x_{i,k})}{\psi_j(x_{j,\ell})}-\frac 12 \innerprod{\psi_j(x_{j,k})}{\psi_i(x_{i,\ell})}\bigg) \\
        &= \sum_{\substack{i,j \in S_0 \\ k,\ell \in \mathfrak s}} c_{i,k}c_{j,\ell} \left(\kappa'_{i,j}\left(x^{(k)}_{i,j},x^{(\ell)}_{i,j}\right)-\kappa_{i,j}\left(x^{(k)}_{i,j},x^{(\ell)}_{i,j}\right)\right) \\
        &\geq 0,
    \end{align*}
    where the inequality is by the assumption that $(\kappa'_{i,j}-\kappa_{i,j})$ is positive semidefinite, for all $i,j \in \mathcal{J}$.
    Thus, $\mathrm{rank}(G) \leq \mathrm{rank}(G')$. As $\mathrm{rank}(G)$ and $\mathrm{rank}(G)$ are equal to the LHSs of \cref{eq:span-rect-set,eq:span-rect-set'}, respectively, \cref{eq:dominance-rect} follows. 
    
    Letting $\mathcal{J}_0 = \supp{\underline{N}}$, application of \cref{lem:monotonicity} with $\{\spanRinline{\psi_j(K_j)}\}_{j \in \mathcal{J}_0}$, $\{\spanRinline{\psi'_j(K'_j)}\}_{j \in \mathcal{J}_0}$, and $\phi = \mathrm{Id}\colon \mathbb N_0^{\mathcal{J}_0}\to \mathbb N_0^{\mathcal{J}_0}$ yields, thanks to \cref{eq:dominance-rect}, the desired bound.  
\end{proof}

\Cref{lem:rect-monotonicity} shows that the number of homogeneously separable dichotomies realizable on a countably $\mathcal{H}^s$-rectifiable set $E$ is bounded from above by the number of realizable dichotomies on a countably $\mathcal{H}^s$-rectifiable set $E'$ whenever $(\kappa'_{i,j}-\kappa_{i,j})$ for all $i,j \in \mathcal{J}$. Intuitively, the positive semi-definiteness of $(\kappa'_{i,j}-\kappa_{i,j})$ means that $E'$ is spread out in at least as many directions as $E$. 

\paragraph{$\Phi$-separability} Whenever $\Phi \colon E \to \mathbb R^{M'}$ is a Lipschitz map, $\Phi(E)$ is also countably $\mathcal{H}^s$-rectifiable, so that \cref{prop:sep-rect} can be leveraged to determine the number of $\Phi$-separable dichotomies of the $N$-point set $F \subset E$. The case where $\Phi$ is non-Lipschitz will be covered in the next subsection. 

\subsection{Measurable and \texorpdfstring{$\sigma$}{sigma}-finite sets}\label{subsec:sep-gen-case}
Let now $E \subseteq \mathbb R^M$ be $\mathcal{H}^s$-measurable with $\mathcal{H}^s(E)>0$ for some $s \geq 0$, suppose that $\mathcal{H}^s \llcorner E$ is $\sigma$-finite, and let $\Phi \colon E \to \mathbb R^{M'}$ be measurable. Inspired by the previous subsections, consider the following decomposition of $E$: let $\{E_j\}_{j \in \mathcal J}$ be a countable family of $\mathcal{H}^s$-measurable sets of positive $\mathcal{H}^s$-measure such that\footnote{The minimum in \ref{it:iii} is taken over all \emph{$\mathcal{H}^s$-measurable} sets $A\subseteq E_j$ with $\mathcal{H}^s(A)>0$. We dropped the measurability constraint for notational convenience.}
\begin{enumerate}[label=(d-\roman*)]
    \item\label{it:i} $\mathcal{H}^s(E \setminus \bigcup_{j \in \mathcal{J}}E_j)=0$, 
    \item\label{it:ii} $\mathcal{H}^s(E_i \cap E_j) = 0$, $i,j \in \mathcal{J}$ with $i \neq j$, and
    \item\label{it:iii} for every $j\in \mathcal J$, 
    \begin{align*}
        \min_{\substack{A \subseteq E_j \\ \mathcal{H}^s(A)>0}} \dimR{\spanR{\Phi(A)}} = s_j,
    \end{align*}
    where $s_j \coloneqq \dimR{\spanR{\Phi(E_j)}}$.
\end{enumerate}
Such a decomposition of $E$ into $\{E_j\}_{j \in \mathcal{J}}$ may be obtained through the following procedure.
Consider 
\begin{align}\label{eq:decomp-procedure-1}
    \min_{\substack{A \subseteq E \\ \mathcal{H}^s(A)>0}} \dimR{\spanR{\Phi(A)}} \eqqcolon s_1.
\end{align}
As $\dimR{\spanRinline{\Phi(A)}} \in \{0,\ldots,M'\}$, for every $\mathcal{H}^s$-measurable $A\subseteq E$ with $\mathcal{H}^s(A)>0$, the minimum in \cref{eq:decomp-procedure-1} is attained.
Thus there exists an $\mathcal{H}^s$-measurable set $E_1 \subseteq E$ with $\mathcal{H}^s(E_1) >0$ such that $\dimR{\spanRinline{\Phi(E_1)}} = s_1$. By \cref{eq:decomp-procedure-1}, the set $E_1$ satisfies \cref{it:iii}. Replacing $E$ by $E\setminus E_1$ and repeating the above procedure, we construct successively the sets $E_2, E_3, \ldots$. As $\mathcal{H}^s\llcorner E$ is $\sigma$-finite and $\{E_j\}_{j \in \mathcal{J}}$ are pairwise disjoint with $\mathcal{H}^s(E_j)>0$, the index set $\mathcal{J}$ is at most countable. Note that the constructed sets $\{E_j\}_{j \in \mathcal{J}}$ satisfy stronger properties than those in \ref{it:i} and \ref{it:ii}: they are pairwise disjoint and $E = \bigcup_{j \in \mathcal{J}} E_j$. We have intentionally stated \ref{it:i} and \ref{it:ii} in a weaker form so that the decomposition \ref{it:i}–\ref{it:iii} aligns with the natural decompositions arising in the two main examples of interest: for sets of sparse vectors (i.e., unions of linear subspaces), where $E_i$ and $E_j$ are typically only essentially disjoint ($\mathcal{H}^s(E_i \cap E_j) = 0$); and for rectifiable sets, where the bi-Lipschitz decomposition $\{E_j\}_{j \in \mathcal{J}}$ covers $E$ only up to an $\mathcal{H}^s$-nullset. See \cref{rem:decomposition-sparse-rect} for further discussion.        

For the $N$-point set $F \coloneqq \{f_1, \ldots, f_N\} \subseteq E$, we write $F_j \coloneqq F \cap E_j$, and let $\pi_j \colon \Phi(E_j) \to \mathbb R^{s_j}$ be a linear embedding of $\Phi(E_j)$ into $\mathbb R^{s_j}$. The extension of $\pi_j$ to $\spanRinline{\Phi(E_j)}$ is denoted by $\widetilde \pi_j$. The following assumption adapts \cref{A:sep-rect} to the present setting, with the additional requirement that $F$ contains no points in nonempty intersections $E_i \cap E_j$, $i,j \in \mathcal{J}$.  
\begin{assumption}\label{A:sep-general-case}
Suppose that $F \subset \bigcup_{j \in \mathcal{J}} E_j$, and $F \cap E_i\cap E_j = \emptyset$, $i,j \in \mathcal{J}$ with $i \neq j$. Whenever $F_j \neq \emptyset$, for $j\in \mathcal{J}$, we assume the following: 
\begin{enumerate}[label=(\roman*)]
    \item\label{it:sep-gen-i} $F_j$ is in $(\pi_j \circ \Phi)$-general position.
    \item\label{it:sep-gen-ii} If there is an index set $\mathcal J' \subseteq \mathcal J$ such that $\spanR{\Phi(E_i)} = \spanR{\Phi(E_j)}$, for all $i\in \mathcal{J}'$, $\bigcup_{i \in \mathcal{J}'}F_i$ is in $(\widetilde\pi_j \circ \Phi)$-general position.
    \item\label{it:sep-gen-iii} $\Phi(F_j) \cap \sum_{i \in \mathcal J''}\spanR{\Phi(E_i)} = \emptyset$ whenever $\Phi(E_j) \not \subseteq \sum_{i \in \mathcal J''}\spanR{\Phi(E_i)}$, for every $\mathcal J'' \subset \mathcal{J}$.  
\end{enumerate}
\end{assumption}
\cref{A:sep-general-case} prevents degenerate configurations. Specifically, by \cref{it:sep-gen-i} we ensure that $\Phi(F_j)$ is in general position when linearly embedded into $\mathbb R^{s_j}$. \cref{it:sep-gen-ii} guarantees that, if several components $\{\Phi(E_i)\}_{i \in \mathcal J'}$ span the same subspace, then $F \cap \bigcup_{i \in \mathcal{J}'}E_i$ remains in general position with respect to the common embedding $\widetilde\pi_j$. Finally, \cref{it:sep-gen-iii} ensures that the points of $\Phi(F_j)$ do not lie in any lower-dimensional subspace of $\spanR{\Phi(E_j)}$ that is also contained in the span of other components $\{\Phi(E_i)\}_{i \in \mathcal J''}$, except in the case where $\Phi(E_j) \subseteq \sum_{i \in \mathcal J''} \spanR{\Phi(E_i)}$.        
In \cref{rem:assumptions-almost-surely-satisfied}, we will leverage the properties of the decomposition \ref{it:i}--\ref{it:iii} to show that for $(\mathcal{H}^s)^N$-a.e. $N$-tuple, the corresponding $N$-point set satisfies \cref{A:sep-general-case}. 
Setting $\underline{N} = (N_j)_{j\in \mathcal{J}} \coloneqq (\lvert F_j \rvert)_{j \in \mathcal{J}}$, we obtain the following dichotomy count.

\begin{proposition}\label{prop:sep-general-case}
     Under \cref{A:sep-general-case} the number of $\Phi$-separable dichotomies of $F$ is given by
     \begin{align*}
            C(\underline{N},\Phi,s) \coloneqq 2^N - 2 \sum_{t=s^*+1}^N \sum_{\underline \nu \in \mathcal{I}_t} \prod_{j \in \mathcal{J}_0}\binom{N_j}{\nu_j},
        \end{align*}
        where $\mathcal{J}_0 \coloneqq \supp{\underline{N}} \subseteq \mathcal{J}$ is a finite subset, $s^* \coloneqq \min_{j \in \mathcal{J}_0} s_j$, and where $\mathcal{I}_t \coloneqq \{\underline \nu \in \mathbb N_0^{\mathcal{J}_0} \colon \lenin{\underline{\nu}} = t, \Upsilon(\underline \nu) \not\equiv t \Mod{2}\}$ with
        \begin{align*}
            \Upsilon(\underline{\nu}) \coloneqq \min_{\mathfrak{s} \subseteq \suppinline{\underline{\nu}}} \left\{\sum_{j \in \mathfrak{s}} \nu_j+\dimR{\spanR{\bigcup_{j \in \mathfrak{s}^c} \Phi(E_j)}} \right\},\quad \underline \nu = (\nu_j)_{j \in \mathcal{J}_0} \in \mathbb N_0^{ \mathcal{J}_0 }.
        \end{align*}
\end{proposition}
\begin{proof}
    Let $\mathcal K \subseteq \{1,\ldots,N\}$ and consider the disjoint decomposition of $\mathcal{K}$ into $\{\mathcal{K}_j\}_{j \in \mathcal{J}_0}$ such that $\{f_k\}_{k \in \mathcal{K}_j} \subseteq \Phi(E_j)$, for all $j \in \mathcal{J}_0$. We first show that 
    \begin{align}\label{eq:sep-rect-1}
        \dimR{\spanR{\{f_k\}_{k \in \mathcal K}}} = \Upsilon(\underline{\nu}),
    \end{align} 
    where $\nu_j \coloneqq \lvert \mathcal{K}_j \rvert$ and $\underline{\nu} = (\nu_j)_{j \in \mathcal{J}_0} \in \mathbb N_0^{\mathcal{J}_0}$. To this end, we proceed by induction on $\lvert \supp{\underline{\nu}}\rvert$. For the base case $\lvert \supp{\underline{\nu}}\rvert = 1$, we note that \cref{A:sep-general-case}\ref{it:sep-gen-i} implies $\dimRinline{\spanRinline{\{f_k\}_{k \in \mathcal K}}} = \min\{\nu_{j_0},s_{j_0}\}$, where $j_0 \in \supp{\underline{\nu}}$. As \begin{align*}
        \min_{\mathfrak{s} \subseteq \{j_0\}} \left\{\sum_{j \in \mathfrak{s}} \nu_j+\dimR{\spanR{\bigcup_{j \in \mathfrak{s}^c} \Phi(E_j)}} \right\} = \min\{\nu_{j_0},s_{j_0}\},
    \end{align*}
    the base case follows. Next, for the induction step we need to show that if \cref{eq:sep-rect-1} holds for all $\underline \nu \in \mathbb N_0^{\mathcal{J}_0}$ with $\lvert \supp{\underline{\nu}} \rvert \leq r$ for some $r \in \{1,\ldots,\lvert \mathcal{J}_0\rvert-1\}$, then it also holds for all $\underline{\nu}$ with $\lvert \supp{\underline{\nu}} \rvert = r+1$. Now, by \cref{A:sep-general-case}, for every $\mathfrak{s} \subseteq \mathcal J_0$ and every $\{\mathfrak{k}_j\}_{j \in \mathfrak{s}}$ with $\mathfrak{k}_j \subseteq \mathcal{K}_j$ and $\lvert \mathfrak{k}_j \rvert \leq s_j$, 
        \begin{align*}
            \bigcup_{j \in \mathfrak{s}} \{f_k\}_{k \in \mathfrak{k}_j} \text{ is in $\pi_\mathfrak{s}$-general position,}
        \end{align*}
        where $\pi_\mathfrak{s}\colon \bigcup_{j \in \mathfrak{s}}\Phi(E_j) \to \mathbb R^{d_{\mathfrak{s}}}$ is the linear embedding of $\bigcup_{j \in \mathfrak{s}}\Phi(E_j)$ into $\mathbb R^{d_{\mathfrak{s}}}$ with $d_\mathfrak{s} \coloneqq \dimRinline{\spanRinline{\bigcup_{j \in \mathfrak{s}}\Phi(E_j)}}$. We thus have for every $j_0 \in \supp{\underline{\nu}}$,
        \begin{align*}
            \dimR{\spanR{\{f_k\}_{k \in \mathcal K}}}
            &= \min\!\left\{\dimR{\spanR{\{f_k\}_{k \in \mathcal K\setminus\mathcal{K}_{j_0}}}}+ \dimR{\spanR{\{f_k\}_{k \in \mathcal K_{j_0}}}}\!,\!\tallphantom\right.\\ &\quad\quad\quad\quad\left.\dimR{\spanR{\bigcup_{j \in \suppinline{\underline{\nu}}} \Phi(E_j)}}\right\}.
        \end{align*}
        Denoting by $\underline{\nu}^{\setminus j_0}$ the multi-index $\underline{\nu}$ with $\nu_{j_0}$ set to zero, we obtain, under the induction hypothesis,
        \begin{align*}
            \dimR{\spanR{\{f_k\}_{k \in \mathcal K}}} &= \min\!\left\{\Upsilon(\underline{\nu}^{\setminus j_0})+ \min\{\nu_{j_0},s_{j_0}\},\dimR{\spanR{\bigcup_{j \in \suppinline{\underline{\nu}}} \Phi(E_j)}}\right\},
        \end{align*}
        for all $j_0 \in \supp{\underline{\nu}}$. Now, by the same argument as in the proof of \cref{eq:sep-sparse-frame-2}, we obtain 
        \begin{align*}
            \dimR{\spanR{\{f_k\}_{k \in \mathcal K}}} = \Upsilon(\underline{\nu}),
        \end{align*}
        as desired. Thus, the set of hyperplanes $\{H_k\}_{k \in \mathcal{K}}$, where $H_k \coloneqq \{f_k\}^\perp$, is odd-degenerate if and only if 
        \begin{align*}
            \Upsilon(\underline \nu) \not\equiv t \Mod{2}.
        \end{align*}
        Summing over all possible configurations, we can deduce that the number of odd-degenerate sets of the hyperplanes $\{H_k\}_{k=1}^N$ is given by 
        \begin{align*}
            \sum_{t=s^*+1}^N \sum_{\underline{\nu} \in \mathcal{I}_t} \prod_{j \in \mathcal{J}_0}\binom{N_j}{\nu_j}.
        \end{align*}
        Finally, application of \cref{thm:winder}, together with \cref{rem:number-even-odd-deg}, completes the proof.  
\end{proof}
\begin{remark}\label{rem:decomposition-sparse-rect}
    Let us revisit the setup from \Cref{subsec:sep-sparse}, i.e., the case where $E$ is the set of $s$-sparse vectors in a certain basis or frame $\Xi$ and $\Phi = \mathrm{Id}$. The set of $s$-sparse vectors $E = \bigcup_{j \in \mathcal{J}} E_j$ naturally admits the decomposition \ref{it:i}--\ref{it:iii}, where $E_j \coloneqq \spanRinline{\{\xi_k\}_{k \in S_j}}$. Indeed, \ref{it:i} is obvious, and \ref{it:ii} holds because, by assumption, $\{\xi_k\}_{k \in S_j}$ are linearly independent for all index subsets $S_j$ with $\lvert S_j \rvert =s$, so that the linear subspace $E_i \cap E_j$ is at most $(s-1)$-dimensional for all distinct $i,j \in \mathcal{J}$. Property \ref{it:iii} is satisfied because for every $\mathcal{H}^s$-measurable set $A \subseteq E_j$ with $\mathcal{H}^s(A)>0$, $\dimR{\spanR{A}} = s$. To see this, suppose, for the sake of contradiction, that $\dimR{\spanR{A}} =s'< s$, then $A$ is a subset of an $s'$-dimensional subspace $\mathcal{U}$ of $\mathbb R^M$. But $\mathcal{H}^s(\mathcal U)=0$ since $s'<s$, see \cref{prop:hausdorff-meas}, which establishes the contradiction. 
    
    Likewise, this is the case for rectifiable sets discussed in \Cref{subsec:sep-rect}, namely, the bi-Lipschitz parametrization in the sense of \cref{L:bi-Lip} can be chosen such that \ref{it:i}--\ref{it:iii} are satisfied. First note that the properties \ref{it:i} and \ref{it:ii} hold trivially for any such parametrization. Regarding \ref{it:iii}, consider the bi-Lipschitz parametrization $\{\psi_j \colon K_j \to \psi_j(K_j)\}_{j \in \mathcal J}$. If \ref{it:iii} does not hold for $\psi_j(K_j)$, decompose $\psi_j(K_j)$ into $\{\psi_j(K_{j,k})\}_{k \in \mathcal K_j}$ such that $\{\psi_j(K_{j,k})\}_{k \in \mathcal K_j}$ satisfies \ref{it:i}--\ref{it:iii}. Since $\mathcal{H}^s(\psi_j(K_j)) \leq \mathrm{Lip}(\psi_j)^s \mathcal{L}^s(K_j)< \infty$ and $\mathcal{H}^s(\psi_j(K_{j,k}))>0$ for all $k \in \mathcal K_j$, the index set of the decomposition, $\mathcal K_j$, is (at most) countable, e.g., we may set $\mathcal K_j = \mathbb N_0$ or $\mathcal K_j = \{0,\ldots,L_j\}$ for some $L_j \in \mathbb N_0$. Now exhaust $K_{j,0}$ and each $K_{j,k} \setminus \bigcup_{\ell \colon \ell<k} K_{j,\ell}$, $k \in \mathcal K \setminus \{0\}$, up to an $\mathcal{L}^s$-nullset by a countable family of pairwise disjoint compact sets. By doing so, we obtain a bi-Lipschitz parametrization, denoted by $\{\tilde \psi_j \colon \tilde K_j \to \tilde\psi_j(\tilde K_j)\}_{j \in \tilde{\mathcal J}}$, such that $\{\tilde\psi_j(\tilde K_j)\}_{j \in \tilde{\mathcal J}}$ satisfies \ref{it:i}--\ref{it:iii}.   
\end{remark}

\section{Separation capacity for low-dimensional data structures}\label{sec:sep-cap}
A natural question one may ask about transformations $\Phi \colon E \to \mathbb R^{M'}$, where $E \subseteq \mathbb R^{M}$ with $M,M' \in \mathbb N$, is: How can we efficiently compare different transformations with respect to their classification capabilities on a given dataset $E$? 
The so-called \emph{separation capacity} of $\Phi$ provides a measure for this.
In particular, the separation capacity characterizes the classification capabilities of the function class $\{f \mapsto \mathrm{sign}(\innerprod{\Phi(f)}{w}) \colon w \in \mathbb R^{M'}\}$ induced by the transformation $\Phi \colon E \to \mathbb R^{M'}$.
However, the standard definition of separation capacity in the sense of \cite{cover1965geometrical}, formally stated in \cite{kowalczyk1994separating,haberle2026scattering}, only encompasses maps between Euclidean spaces, i.e., $\Phi \colon \mathbb R^M \to \mathbb R^{M'}$.
In this section, we will generalize the notion of separation capacity to maps on $\mathcal{H}^s$-measurable sets $E$ with $\mathcal{H}^s(E)>0$, for some $s\geq 0$. This extension encompasses a broad class of datasets $E \subset \mathbb R^M$ of $\mathcal L^M$-measure zero, thereby including many practically relevant examples such as the set of $s$-sparse signals. 
 
\begin{definition}[$s$-separation capacity]\label{def:s-sc}
    Let $M,M' \in \mathbb N$, and let $E \subseteq \mathbb R^M$ be $\mathcal{H}^s$-measurable with $\mathcal{H}^s(E)>0$ for $s\geq 0$. Let $\Phi \colon E \to \mathbb R^{M'}$. Denote by $\sepcaps{s}{\Phi}$ the largest $N \in \mathbb N$ such that for $(\mathcal{H}^{s})^N$-a.e. $N$-tuple $F\coloneqq(f_1,\ldots,f_N) \in E^N$ at least $50\%$ of all possible dichotomies of $F$ are $\Phi$-separable. If there is no such $N \in \mathbb N$, set $\sepcaps{s}{\Phi} \coloneqq 0$. We call $\sepcaps{s}{\Phi}$ the \emph{$s$-separation capacity} of $\Phi$. If $s = M$, we write $\sepcap{\Phi} \coloneqq \sepcaps{M}{\Phi}$ and call $\sepcap{\Phi}$ the \emph{separation capacity} of $\Phi$. 
\end{definition}
\begin{remark}[Notation]
    With slight abuse of notation, we will use from now on $F$ to denote the set $\{f_1,\ldots,f_N\} \subseteq E$ as well as the tuple $(f_1,\ldots,f_N) \in E^N$. It should be clear from the context in which sense $F$ has to be understood.
\end{remark} 
\begin{remark}
    On $\mathbb R^M$, we have by \cref{prop:hausdorff-meas}, $\mathcal{H}^M = \mathcal{L}^M$. Thus, this definition of $\sepcap{\Phi}$ coincides with the standard definition of separation capacity of $\Phi$ in \cite{haberle2026scattering,kowalczyk1994separating}.
\end{remark}
\begin{remark}[VC dimension]
    The concept of $s$-separation capacity closely relates to another, well-known measure of classification capabilities, namely the Vapnik--Chervonenkis (VC) dimension \cite{vapnik1971uniform}. Concretely, the VC dimension of the function class $\{f \mapsto \mathrm{sign}(\innerprod{\Phi(f)}{w}) \colon w \in \mathbb R^{M'}\}$ is defined to be the largest $N \in \mathbb N$ for which there exists an $N$-point set $F \subseteq E$ such that all possible $2^N$ dichotomies of $F$ are $\Phi$-separable. By contrast, the $s$-separation capacity is a measure-theoretic, average-case notion, which evaluates separability on the whole dataset $E$ up to $\mathcal{H}^s$-nullsets, but requiring only that at least $50\%$ of all dichotomies be $\Phi$-separable.   
    In general, neither quantity uniformly bounds the other without additional assumptions on $E$ and $\Phi$ \cite{haberle2026scattering}.
    Furthermore, the $s$-separation capacity is therefore, intuitively, more strongly governed by the geometry of $E$ than the VC dimension. Consequently, $s$-separation capacity provides a more natural framework for analyzing how the dataset $E$ affects the classification capabilities of $\Phi$. 
\end{remark}
Let now $E \subseteq R^M$ be $\mathcal H^s$-measurable with $\mathcal H^s(E)>0$, and suppose that $\mathcal{H}^s\llcorner E$ is $\sigma$-finite. 
In the following, we will derive a ready-to-use expression for the $s$-separation capacity of $\Phi$. 
To this end, let $F\in E^N$, and recall the number of $\Phi$-separable dichotomies of $F$ given in \cref{prop:sep-general-case}, be denoted by $C(\underline{N},\Phi,s)$. Here, $\underline{N} =(N_j)_{j \in \mathcal J} \in \mathbb N_0^{\mathcal{J}}$ with $N_j$ being the number of points of $F$ in $E_j$, $j \in \mathcal J$, and $(E_j)_{j \in \mathcal{J}}$ constitute the decomposition in the sense of \ref{it:i}--\ref{it:iii} in \Cref{subsec:sep-gen-case}. Suppose for now that for $(\mathcal H^s)^N$-a.e. $F \in E^N$, \cref{A:sep-general-case} is satisfied. Then, by \cref{def:s-sc}, the $s$-separating of $\Phi$ is the largest $N \in \mathbb N$ such that
\begin{align*}
    \min_{\substack{\underline{N} \in \mathbb N_0^{\mathcal{J}} \\ \lvert \underline{N} \rvert =N}}\frac{C(\underline{N},\Phi,s)}{2^N} \geq \frac 12.
\end{align*}
It holds that
\begin{align}\label{eq:min-C-gen-case}
    \min_{\substack{\underline{N} \in \mathbb N_0^{\mathcal{J}} \\ \lvert \underline{N} \rvert =N}} C(\underline N, \Phi, s) = C\left(N,\min_{j \in \mathcal J} s_j\right),
\end{align}
where $s_j \coloneqq \dimR{\spanR{\Phi(E_j)}}$. 
Indeed, we first show that \cref{eq:min-C-gen-case} holds with ``$\leq$''. To this end, let $j^* \in \mathcal J$ be such that $s_{j^*} = \min_{j \in \mathcal J} s_j$, and let $\underline N^* = (N_j^*)_{j \in \mathcal{J}} \in \mathbb N_0^{\mathcal{J}}$, where $N_j^* = 0$ for all $j \in \mathcal{J} \setminus \{j^*\}$ and $N^*_{j^*} = N$. By \cref{prop:sep-general-case},
\begin{align}\label{eq:min_N_upperbound_1}
    C(\underline N^*, \Phi, s) &= 2^N - 2\sum_{\substack{t=s_{j^*}+1 \\ (t-s_{j^*})\text{ is odd}}}^N \binom{N}{t}.
\end{align}
Following the derivation in \cref{rem:all-in-one}, we obtain 
\begin{align}\label{eq:min_N_upperbound_2}
    C(\underline N^*, \Phi, s) = C\left(N,s_{j^*}\right).  
\end{align}
The reverse inequality (i.e., ``$\geq$'') follows from \cref{thm:lower-upper-bound-C_F} upon noting that under \cref{A:sep-general-case} every subset of $s_{j^*}$ elements of $\{\Phi(f_1),\ldots,\Phi(f_N)\}$ is linearly independent.

Thus, to determine the $s$-separation capacity of $\Phi$, one needs to find the largest $N\in \mathbb N$ such that
\begin{align*}
    \frac{C(s_{j_*},N)}{2^N} \geq \frac 12.
\end{align*}
By a symmetry argument (carried out in, e.g., \cite{cover1965geometrical}), it follows that the $s$-separation capacity of $\Phi$ is given by
\begin{align}\label{eq:sc-formula}
    \sepcaps{s}{\Phi} = 2s_{j_*} = 2 \min_{j \in \mathcal J} \dimR{\spanR{\Phi(E_j)}}.
\end{align}

Note that in this derivation, it was assumed that 
\begin{align}\label{eq:ae-tupel-assumption}
    \text{for $(\mathcal H^s)^N$-a.e. $F \in E^N$, \cref{A:sep-general-case} is satisfied.}
\end{align}
It is natural to ask whether \cref{eq:ae-tupel-assumption}, and hence \cref{eq:sc-formula}, is indeed valid in general. This will now be analyzed.  
To do so, let us first extend the notion of $\Phi$-general position.
\begin{definition}
    For $M,M',N \in \mathbb N$, let $\Phi \colon E \to \mathbb R^{M'}$, where $E \subseteq \mathbb R^M$, and let $M^\natural \in \mathbb N$ with $M^\natural \leq M'$.
    The set $F \coloneqq \{f_1, \ldots, f_N\} \subseteq E$  is said to be in \emph{$(M^\natural,\Phi)$-general position} if every subset of $k$ elements of $\left\{\Phi(f_1), \ldots, \Phi(f_N)\right\} \subseteq \mathbb R^{M'}$ is linearly independent for all $k \leq \min\{M^\natural,N\}$.
\end{definition}
Note that if $M^\natural = M'$, then $F$ is in $\Phi$-general position. Consider the following key lemma.
\begin{lemma}\label{L:equiv_span_phi-gen_pos}
    Fix $M,M' \in \mathbb N$, let $E \subseteq \mathbb R^M$ be $\mathcal{H}^s$-measurable with $\mathcal{H}^s(E)>0$ for $s\geq 0$, and consider the measurable function $\Phi \colon E \to \mathbb R^{M'}$. Let $M^\natural,N \in \mathbb N$ with $M^\natural \leq M' \leq N$. The set of $N$-tuples $F \coloneqq (f_1, \ldots, f_N) \in E^N$ which are not in $(M^\natural,\Phi)$-general position has $(\mathcal{H}^{s})^N$-measure zero if and only if there is no $\mathcal{H}^s$-measurable set $A \subseteq E$ with $\mathcal{H}^s(A)>0$ such that 
    \begin{align}\label{eq:cover_equiv_phi-gen_pos}
        \dimR{\spanR{\Phi(A)}} < M^\natural.
    \end{align}
\end{lemma}
\begin{remark}\label{rem:N_less_M_nat}
    The ``if'' statement remains valid in the regime $N\leq M^\natural\leq M'$. 
\end{remark}
\begin{remark}[Measurability]
    The subset of $N$-tuples $F= (f_1,\ldots,f_N)$ not in $(M^\natural,\Phi)$-general position, denoted $P_{M^\natural,\Phi}$, is given by 
    \begin{align*}
        P_{M^\natural,\Phi} = \bigcup_{1 \leq j_1 < \cdots < j_{L} \leq N} \pi^{-1}_{j_1, \ldots, j_{L}} \left( \bigcap_{1 \leq k_1 < \cdots < k_{L} \leq M'}\delta_{k_1,\ldots,k_L}^{-1}(\{0\})\right),
    \end{align*}
    where $L \coloneqq \min\{M^\natural,N\}$, $\pi_{j_1, \ldots, j_{L}} \colon E^{N} \to E^{L}, (f_1,\ldots,f_N) \mapsto (f_{j_1}, \ldots, f_{j_{L}})$ is the canonical projection, and where
    \begin{align*}
        \delta_{k_1,\ldots,k_L} \colon E^L \to \mathbb R, (f_1,\ldots,f_L) \mapsto \det \begin{pmatrix}
            \Phi_{k_1}(f_1) &\cdots & \Phi_{k_1}(f_L) \\
            \vdots & \ddots & \vdots \\
            \Phi_{k_L}(f_1) &\cdots & \Phi_{k_L}(f_L)
        \end{pmatrix}.    
    \end{align*}
    It follows that $P_{\Phi} \subseteq E^N$ is $(\mathcal H^s)^{N}$-measurable whenever $\Phi$ is measurable.
\end{remark}
\begin{proof}[Proof of \cref{L:equiv_span_phi-gen_pos}]
We first show the contrapositive of the ``only if'' statement. Namely, suppose there is an $\mathcal{H}^s$-measurable set $A \subseteq E$ with $\mathcal{H}^s(A)>0$ such that \cref{eq:cover_equiv_phi-gen_pos} holds. Then, all $N$-tuples in $A^N$ are not in $(M^\natural,\Phi)$-general position for $N\geq M^\natural$. Indeed, if $F \in A^N$, then $\dimR{\spanRinline{\{\Phi(f_k)\}_{k=1}^N}}< M^\natural$, which implies that every subset of $M^\natural$ elements of $\{\Phi(f_k)\}_{k=1}^N$ is linearly dependent. Thus, the ``only if'' part follows as $(\mathcal H^s)^N(A^N) = N \mathcal H^s(A)>0$.

Next, consider the ``if'' statement. That is, assume that there is no $\mathcal{H}^s$-measurable set of positive $\mathcal{H}^s$-measure such that \cref{eq:cover_equiv_phi-gen_pos} holds. We prove the claim (in the form of \cref{rem:N_less_M_nat}) by induction on $N$. For $N=1$, $f_1 \in E$ is in $(M^\natural,\Phi)$-general position if and only if $\Phi(f_1) \neq 0$. Set $A \coloneqq \{f \in E \colon \Phi(f)=0\}$. Then, by assumption, we must have $\mathcal{H}^s(A) =0$ since $\dimR{\spanR{\Phi(A)}} = 0$. Now suppose the claim is true for $N-1$, i.e., $(\mathcal{H}^s)^{N-1}$-a.e. $(f_1,\ldots,f_{N-1}) \in E^{N-1}$ is in $(M^\natural,\Phi)$-general position. Fix such an $(N-1)$-tuple which is in $(M^\natural,\Phi)$-general position, and let $f_N \in E$. Then, $(f_1,\ldots,f_N)$ is in $(M^\natural,\Phi)$-general position if and only if $\Phi(f_N) \notin \spanRinline{\{\Phi(f_{j_k})\}_{k=1}^{L-1}}$ for every $1\leq j_1 < \cdots < j_{L-1}\leq N-1$, where $L \coloneqq \min\{M^\natural,N\}$. Define $A_{j_1,\ldots,j_{L-1}} \coloneqq \{f \in E \colon \Phi(f) \in \spanRinline{\{\Phi(f_{j_\ell})\}_{\ell=1}^{L-1}}\}$. But $\mathcal{H}^s(A_{j_1,\ldots,j_{L-1}})=0$ as $\dimR{\spanR{\Phi(A_{j_1,\ldots,j_{L-1}})}} \leq L-1<M^\natural$. Consequently, $(\mathcal{H}^s)^{N}$-a.e. $(f_1,\ldots,f_{N}) \in E^{N}$ is in $(M^\natural,\Phi)$-general position.   
\end{proof}

Building on \cref{L:equiv_span_phi-gen_pos}, we proceed to demonstrate the validity of \cref{eq:ae-tupel-assumption} in the next remark.

\begin{remark}\label{rem:assumptions-almost-surely-satisfied}
    In the following, we prove that \cref{A:sep-general-case} holds for $(\mathcal{H}^s)^N$-a.e. $F\in E^N$. 
    First note that by \ref{it:i} of the decomposition from \Cref{subsec:sep-gen-case}, $\mathcal{H}^s(E\setminus \bigcup_{j \in \mathcal{J}} E_j)=0$, and hence $(\mathcal{H}^s)^N$-a.e. $(f_1,\ldots,f_N) \in E^N$ is such that $\{f_1,\ldots,f_N\} \subset \bigcup_{j \in \mathcal{J}} E_j$. Furthermore, \ref{it:ii} ensures that for $(\mathcal{H}^s)^N$-a.e. $(f_1,\ldots,f_N) \in E^N$, we have $\{f_1,\ldots,f_N\}\cap E_i \cap E_j = \emptyset$ whenever $i,j \in \mathcal{J}$ with $i\neq j$.  
    Next, we show that $(\mathcal{H}^s)^{N}$-a.e. $F \in E^N$ satisfies \cref{it:sep-gen-i,it:sep-gen-ii,it:sep-gen-iii} of \cref{A:sep-general-case}. To this end, fix $j \in \mathcal{J}$. 
    \begin{enumerate}[label=(\roman*)]
        \item By \ref{it:iii}, there is no $A \subseteq E_j$ of positive $\mathcal{H}^s$-measure such that $\dimR{\spanR{\Phi(A)}}<s_j$, so that application of \cref{L:equiv_span_phi-gen_pos} yields that $(\mathcal{H}^s)^{N}$-a.e. $F \in E_j^{N}$ is in $(s_j,\Phi)$-general position. Consequently, for $(\mathcal{H}^s)^{N}$-a.e. $F \in E^N$, $F_j$ is in $(\pi_j\circ\Phi)$-general position if $F_j \neq \emptyset$. This establishes \cref{it:sep-gen-i} of \cref{A:sep-general-case}. 
        \item Similarly, for \cref{it:sep-gen-ii} of \cref{A:sep-general-case}, note that, by \ref{it:iii}, there is no $A \subseteq \bigcup_{i \in \mathcal{J}'}E_i$ of positive $\mathcal{H}^s$-measure such that $\dimR{\spanR{\Phi(A)}}<s_j$. Using \cref{L:equiv_span_phi-gen_pos}, we obtain that $(\mathcal{H}^s)^{N}$-a.e. $F \in (\bigcup_{i\in\mathcal{J}'}E_i)^{N}$ is in $(s_j,\Phi)$-general position. Thus, for $(\mathcal{H}^s)^{N}$-a.e. $F \in E^N$, $\bigcup_{i\in \mathcal{J}'}F_i$ is in $(\widetilde\pi_j\circ\Phi)$-general position whenever $\bigcup_{i\in \mathcal{J}'}F_i \neq \emptyset$.
        \item Finally, for \cref{it:sep-gen-iii} of \cref{A:sep-general-case}, let $\mathcal{J}''\subset \mathcal{J}$ and set 
        \begin{align*}
            A \coloneqq \left\{f \in E_j \colon \Phi(f) \in \sum_{i \in \mathcal{J}''}\spanR{\Phi(E_i)}\right\},    
        \end{align*}
        then $\mathcal{H}^s(A)=0$ whenever $\Phi(E_j) \not \subseteq \sum_{i \in \mathcal{J}''}\spanR{\Phi(E_i)}$. Indeed, if $\Phi(E_j) \not \subseteq \sum_{i \in \mathcal{J}''}\spanR{\Phi(E_i)}$, then 
        \begin{align}\label{eq:verifying-iii}
            \dimR{\spanR{\Phi(E_j)}\cap \sum_{i \in \mathcal{J}''}\spanR{\Phi(E_i)}}<s_j.
        \end{align} 
        But the LHS of \cref{eq:verifying-iii} equals $\dimRinline{\spanR{\Phi(A)}}$. Hence, by \ref{it:iii}, $\mathcal{H}^s(A)=0$. Consequently, it holds that for $(\mathcal{H}^s)^N$-a.e. $F$, $\Phi(F_j) \cap \sum_{i \in \mathcal{J}''}\spanR{\Phi(E_i)} = \emptyset$ whenever $\Phi(E_j) \not \subseteq \sum_{i \in \mathcal{J}''}\spanR{\Phi(E_i)}$ and $F_j \neq \emptyset$, as desired.
    \end{enumerate}
    This also establishes the analogous results for \cref{A:sep-sparse-basis,A:sep-sparse-frame,A:sep-rect}. Note that, as discussed in \cref{rem:decomposition-sparse-rect}, the bi-Lipschitz parametrization of a countably $\mathcal{H}^s$-rectifiable set can be chosen such that it admits the decomposition \ref{it:i}--\ref{it:iii} from \Cref{subsec:sep-gen-case}.
\end{remark}
In the next theorem, the obtained expression for the $s$-separation capacity is stated. Additionally, we provide an alternative proof which does not rely on the decomposition \ref{it:i}--\ref{it:iii}, so that the assumption of $\mathcal{H}^s\llcorner E$ being $\sigma$-finite can be dropped.  
\begin{theorem}\label{thm:sep-cap-formula}
    Let $M,M' \in \mathbb N$, and let $E \subseteq \mathbb R^M$ be $\mathcal{H}^s$-measurable with $\mathcal{H}^s(E)>0$ for $s\geq 0$. Let $\Phi \colon E \to \mathbb R^{M'}$ be measurable. The $s$-separation capacity of $\Phi$ is given by
    \begin{align*}
        \sepcaps{s}{\Phi} = 2 \min_{\substack{A \subseteq E \\ \mathcal H^s(A)>0}} \dimR{\spanR{\Phi(A)}}.
    \end{align*}
\end{theorem}
\begin{proof}
     Let $E^{\natural} \subseteq E$ be an $\mathcal H^s$-measurable set of positive $\mathcal{H}^s$-measure such that 
    \begin{align}\label{eq:def-M_nat}
        \dimR{\spanR{\Phi(E^{\natural})}} = \min_{\substack{A \subseteq E \\ \mathcal H^s(A)>0}} \dimR{\spanR{\Phi(A)}} \eqqcolon M^{\natural}.
    \end{align}
    There exists a linear map $\pi^\natural \colon \mathbb R^{M'} \to \mathbb R^{M^\natural}$ such that
    $\Phi^\natural \coloneqq \pi^\natural \circ \Phi \colon E^\natural \to M^\natural$ satisfies
    \begin{align*}
        \min_{\substack{A \subseteq E^\natural \\ \mathcal H^s(A)>0}} \dimR{\spanR{\Phi^\natural(A)}} = M^\natural.
    \end{align*}
    Thus, by \cref{L:equiv_span_phi-gen_pos}, $(\mathcal{H}^s)^N$-a.e. $F \in (E^\natural)^N$ is in $\Phi^\natural$-general position. It follows from \cref{thm:fct-counting}, that the number of $\Phi^\natural$-separable dichotomies is $C(N,M^\natural)$, and hence $\sepcaps{s}{\Phi^\natural}=2M^\natural$. Note that on $E^\natural$, $\Phi$-separability is equivalent to $\Phi^\natural$-separability. Thus, we have $\sepcaps{s}{\Phi} \leq \sepcaps{s}{\Phi^\natural} = 2M^\natural$. 

    To show that equality holds, recall \cref{eq:def-M_nat} and apply \cref{L:equiv_span_phi-gen_pos} to deduce that $(\mathcal{H}^s)^N$-a.e. $F=(f_1,\ldots,f_N) \in E^N$ is in $(M^\natural,\Phi)$-general position. That is, every subset of $\{\Phi(f_1),\ldots,\Phi(f_N)\}$ containing $M^\natural$ elements is linearly independent for $N\geq M^\natural$. Then the number of $\Phi$-separable dichotomies is at least $C(N,M^\natural)$ by \cref{thm:lower-upper-bound-C_F}, and $\sepcaps{s}{\Phi}\geq 2M^\natural$. This completes the proof.
\end{proof}

In the following, we apply the expression in \cref{thm:sep-cap-formula} to two specific cases to analyze the effective dimension which determines the separation capacity: first, when $E$ is the set of $s$-sparse vectors, and second, when $E$ is a countably $\mathcal{H}^s$-rectifiable set.

\subsection{Sparse vectors}
We begin by expressing the $s$-separation capacity in terms of the standard separation capacity $\sepcap{\cdot}$. This reformulation allows a direct application of the computational framework developed in \cite{haberle2026scattering} for the standard separation capacity $\sepcap{\cdot}$. In doing so, we therefore obtain a method to compute the $s$-separation capacity of a transformation $\Phi\colon E \to \mathbb R^{M'}$. 

\begin{proposition}[$s$-sparse vectors]\label{prop:sc-sparse}
    Let $\Xi = \{\xi_k\}_{k \in \mathcal K}$ be a frame for $\mathbb R^M$, $M\in \mathbb N$, and fix $s \in \{1,\ldots,M\}$. Assume that $\{\xi_k\}_{k\in S}$ is linearly independent for every $S \subseteq \mathcal{K}$ with $\lvert S \rvert =s$, and set $E \coloneqq \bigcup_{S\subseteq \mathcal K\colon \lvert S \rvert =s}\left(\spanRinline{\{\xi_k\}_{k\in S}}\right)$. For $\Phi \colon E \to \mathbb R^{M'}$ measurable, we have
    \begin{align}\label{eq:sc-sparse1}
        \sepcaps{s}{\Phi} = \min_{\substack{S \subseteq \mathcal{K} \\ \lvert S \rvert =s}} \sepcap{\Phi \circ \sigma_S},
    \end{align}
    where $\sigma_S \colon \mathbb R^{s} \to \mathbb R^M, c \mapsto \sum_{i=1}^s c_i \xi_{k_{S,i}}$ with the labeling $S=\{k_{S,i}\}_{i=1}^s$. 
    If, moreover, $\Phi$ is real-analytic, then 
    \begin{align}\label{eq:sc-sparse2}
        \sepcaps{s}{\Phi} = 2\min_{\substack{S \subseteq \mathcal{K} \\ \lvert S \rvert =s}}\dimR{\spanR{(\Phi \circ \sigma_S)(\mathbb R^s)}}.
    \end{align}
\end{proposition}
\begin{proof} We first show \cref{eq:sc-sparse1}. Note that 
\begin{align}
    \sepcaps{s}{\Phi} &= 2 \min_{\substack{A \subseteq E \\ \mathcal H^s(A)>0}} \dimR{\spanR{\Phi(A)}} \label{eq:sc-form-sparse-1}\\
    &= 2 \min_{\substack{S \subseteq \mathcal{K} \\ \lvert S \rvert =s}} \min_{\substack{A\subseteq \spanRinline{\{\xi_k\}_{k \in S}}\\\mathcal{H}^s(A)>0}} \dimR{\spanR{\Phi(A)}} \label{eq:sc-form-sparse-2}\\
    &= 2 \min_{\substack{S \subseteq \mathcal{K} \\ \lvert S \rvert =s}} \min_{\substack{A\subseteq \mathbb R^s\\ \mathcal{L}^s(A)>0}} \dimR{\spanR{(\Phi\circ \sigma_S)(A)}} \label{eq:sc-form-sparse-3}\\
    &= 2 \min_{\substack{S \subseteq \mathcal{K} \\ \lvert S \rvert =s}} \sepcap{\Phi \circ \sigma_S}, \label{eq:sc-form-sparse-4} 
\end{align}
where \cref{eq:sc-form-sparse-1} is by \cref{thm:sep-cap-formula}. We next note that \cref{eq:sc-form-sparse-2} holds with ``$\leq$'', as $\spanRinline{\{\xi_k\}_{k \in S}} \subseteq E$, for every $S \subseteq \mathcal{K}$ with $\lvert S \rvert =s$. To see that the reverse inequality in \cref{eq:sc-form-sparse-2}, i.e., ``$\geq$'', is also satisfied, observe that if $A\subseteq E$ with $\mathcal{H}^s(A)>0$, there is an $S \subseteq \mathcal{K}$ with $\lvert S \rvert =s$ such that $A \cap \spanRinline{\{\xi_k\}_{k\in S}} \eqqcolon A'$ is of positive $\mathcal{H}^s$-measure. As $A' \subseteq A$, $\dimRinline{\spanRinline{\Phi(A')}} \leq \dimR{\spanRinline{\Phi(A)}}$. Consequently, \cref{eq:sc-form-sparse-2} holds also with ``$\geq$'', establishing \cref{eq:sc-form-sparse-2}. To show \cref{eq:sc-form-sparse-3}, we again prove both inequalities. For the inequality ``$\leq$'', first note that for every $\mathcal{L}^s$-measurable $A\subseteq \mathbb R^s$ with $\mathcal{L}^s(A)>0$,  
\begin{align}\label{eq:sc-form-sparse-3.1} 
\begin{split}
\begin{array}{c}
     \text{$\sigma_S(A)$ is $\mathcal{H}^s$-measurable and} \\[1mm]
     0<\mathcal{L}^s(A) = \mathcal{H}^s\!\left(\sigma_S^{-1}(\sigma_S(A))\right) \leq \mathrm{Lip}\left(\sigma_S^{-1}\right) \mathcal{H}^s(\sigma_S(A)),
\end{array}
\end{split}
\end{align}
since $\sigma_S\colon \mathbb R^s \to \spanRinline{\{\xi_k\}_{k \in S}}$ is linear and bijective with linear inverse. It now follows from \cref{eq:sc-form-sparse-3.1} that \cref{eq:sc-form-sparse-3} holds with ``$\leq$''. For the reverse inequality, i.e., ``$\geq$'', let $A\subseteq \spanRinline{\{\xi_k\}_{k \in S}}$ be $\mathcal{H}^s$-measurable with $\mathcal{H}^s(A)>0$. By \cref{prop:regularity}, there is a closed set $A' \subseteq A$ such that $\mathcal{H}^s(A')>0$. As $\sigma_S$ has a linear inverse and $A'$ is Borel,
\begin{align}\label{eq:sc-form-sparse-3.2} 
\begin{array}{c}
     \text{$\sigma_S^{-1}(A')$ is $\mathcal{L}^s$-measurable and} \\[1mm]
     0<\mathcal{H}^s(A') = \mathcal{H}^s\!\left(\sigma_S(\sigma_S^{-1}(A'))\right) \leq \mathrm{Lip}(\sigma_S) \mathcal{H}^s(\sigma_S^{-1}(A')) = \mathrm{Lip}(\sigma_S) \mathcal{L}^s(\sigma_S^{-1}(A')). 
\end{array}
\end{align}
Since $A' \subseteq A$ implies $\dimRinline{\spanRinline{\Phi(A')}} \leq \dimRinline{\spanRinline{\Phi(A)}}$, \cref{eq:sc-form-sparse-3.2} yields the inequality ``$\geq$'', thereby establishing \cref{eq:sc-form-sparse-3}.    
Finally, \cref{eq:sc-form-sparse-4} is again by \cref{thm:sep-cap-formula}.

Application of the result from \cite{kowalczyk1994separating} to \cref{eq:sc-sparse1} yields \cref{eq:sc-sparse2}. This completes the proof.
\end{proof}

Considering the identity map $\Phi = \mathrm{Id} \colon \mathbb R^M \to \mathbb R^{M}$ and applying \cref{prop:sc-sparse} to the restriction of $\Phi$ to $E$, denoted $\restrinline{\Phi}{E}$, we obtain $\sepcapsinline{s}{\restrinline{\mathrm{Id}}{E}} = 2s$. Thus, in this case, the $s$-separation capacity is determined entirely by the sparsity parameter $s$.
In comparison with the function-counting results for the homogeneous linear case presented in \Cref{subsec:sep-sparse}, we note that $\sepcapsinline{s}{\restrinline{\mathrm{Id}}{E}}$ is independent of the frame $\Xi$, whereas $C_{\mathrm{sp,f}}(\underline{N},\Xi,s)$ depends on $\Xi$ (see also \cref{L:gram-nondecreasing}). Intuitively, $\sepcapsinline{s}{\restrinline{\mathrm{Id}}{E}}$ is as a coarser, summary measure of separation capabilities, while $C_{\mathrm{sp,f}}(\underline{N},\Xi,s)$ is a finer combinatorial quantity containing information about the separation behavior for each configuration $\underline{N}$.

In general, however, if $\Phi\colon \mathbb R^{M} \to \mathbb R^{M'}$ is not a full-rank linear map, the union-of-linear-subspaces structure is no longer preserved, and $\sepcaps{s}{\restrinline{\Phi}{E}}$ depends on $\Xi$. For instance, let $\Xi' = \{\xi_k'\}_{k \in \mathcal{K}}$ be another frame for $\mathbb R^M$, and suppose that $\Phi \colon \mathbb R^M \to \mathbb R^{M'}$ vanishes on a set containing the $s$-dimensional linear subspace $\spanRinline{\{\xi_k\}_{k \in S_0}}$ for some $S_0 \subseteq \mathcal{K}$ with $\lvert S_0 \rvert =s$, while it does not vanish on any of the subspaces $\spanRinline{\{\xi_k'\}_{k \in S}}$ for $S\subseteq \mathcal{K}$ with $\vert S\rvert =s$. Then, $0 = \sepcapsinline{s}{\restrinline{\Phi}{E}} < \sepcapsinline{s}{\restrinline{\Phi}{E'}}$, where  $E' \coloneqq \bigcup_{S\subseteq \mathcal K\colon \lvert S \rvert =s}\left(\spanRinline{\{\xi'_k\}_{k\in S}}\right)$. This example illustrates that $\sepcapsinline{s}{\restrinline{\Phi}{E}}$ depends critically on how $\Phi$ interacts with all $s$-dimensional linear subspaces spanned by elements of the frame. In particular, to maximize $\sepcapsinline{s}{\restrinline{\Phi}{E}}$, one needs to ensure that $\sepcap{\restrinline{\Phi}{E} \circ \sigma_S}$ is maximized for all $S\subseteq \mathcal{K}$ with $\lvert S\rvert =s$.   

\subsection{Rectifiable sets}
Let us now investigate how to compute the $s$-separation capacity of transformations on countably $\mathcal{H}^s$-rectifiable sets.
\begin{proposition}\label{prop:sc-rect}
    Let $M,M' \in \mathbb N$, and let $E \subseteq \mathbb R^M$ be $\mathcal{H}^s$-measurable with $\mathcal{H}^s(E)>0$ for $s\geq 0$ and countably $\mathcal{H}^s$-rectifiable. Consider the measurable map $\Phi \colon E \to \mathbb R^{M'}$. Let $\{\psi_j \colon K_j \to \psi_j(K_j)\}_{j \in \mathcal J}$ be a bi-Lipschitz parametrization of $E$ according to \cref{L:bi-Lip}. It holds that
    \begin{align*}
        \sepcaps{s}{\Phi} = \min_{j \in \mathcal J} \sepcap{\Phi \circ \psi_j},
    \end{align*}
    where $\Phi \circ \psi_j \colon K_j \subset \mathbb R^s \to \mathbb R^{M'}$.
\end{proposition}
\begin{remark}\label{rem:bi-Lip-H-meas}
    Note that $\sepcap{\Phi \circ \psi_j}$ is well-defined because $\mathcal{L}^s(K_j)>0$. Indeed, if $\psi_j \colon K_j \subset \mathbb R^s \to \psi_j(K_j) \subset\mathbb R^M$ is bi-Lipschitz, then $\psi_j(A)$ is $\mathcal{H}^s$-measurable if $A \subset K_j$ is $\mathcal{L}^s$-measurable, and for every $\mathcal{L}^s$-measurable set $A \subseteq K_j$, 
    \begin{align*}
        \left( \mathcal{H}^s(\psi_i(A)) = 0 \right) \iff \left(\mathcal{H}^s(A)=\mathcal{L}^s(A)=0\right).
    \end{align*}
    Thus, the compact sets $\{K_j\}_{j \in \mathcal J}$ can assumed to be of positive $\mathcal{L}^s$-measure. 
\end{remark}
\begin{proof}[Proof of \cref{prop:sc-rect}]
    By \cref{thm:sep-cap-formula}, we have
    \begin{align}
        \sepcaps{s}{\Phi} &= 2 \min_{\substack{A \subseteq E \\ \mathcal H^s(A)>0}} \dimR{\spanR{\Phi(A)}} \nonumber\\
        &\leq 2 \min_{\substack{A \subseteq \bigcup_{j \in \mathcal J} \psi_j(K_j) \\ \mathcal H^s(A)>0}} \dimR{\spanR{\Phi(A)}}, \label{eq:sc-rect1}
    \end{align}
    where the inequality holds since $\bigcup_{j \in \mathcal J} \psi_j(K_j) \subseteq E$. On the other hand, for every $\mathcal{H}^s$-measurable set $A \subseteq E$ with $\mathcal{H}^s(A)>0$, it holds, as a consequence of \cref{L:bi-Lip}, that $A' \coloneqq A \cap \bigcup_{j \in \mathcal J} \psi_j(K_j)$ is $\mathcal{H}^s$-measurable with $\mathcal{H}^s(A')>0$. Since $A'\subseteq A$, we have
    \begin{align}\label{eq:sc-rect1-a}
        \dimR{\spanR{\Phi(A')}} \leq \dimR{\spanR{\Phi(A)}}.
    \end{align}
    Furthermore, observe that 
    \begin{align}\label{eq:sc-rect1-b}
        \min_{\substack{B \subseteq \bigcup_{j \in \mathcal J} \psi_j(K_j) \\ \mathcal H^s(B)>0}} \dimR{\spanR{\Phi(B)}} \leq \dimR{\spanR{\Phi(A')}}.
    \end{align}
    Since $A \subseteq E$ was an arbitrarily chosen $\mathcal{H}^s$-measurable set with $\mathcal{H}^s(A)>0$, it follows, by combining \cref{eq:sc-rect1-a,eq:sc-rect1-b}, that
    \begin{align*}
        \min_{\substack{B \subseteq \bigcup_{j \in \mathcal J} \psi_j(K_j) \\ \mathcal H^s(B)>0}} \dimR{\spanR{\Phi(B)}} \leq \min_{\substack{A \subseteq E \\ \mathcal H^s(A)>0}} \dimR{\spanR{\Phi(A)}},
    \end{align*}
    and consequently, \cref{eq:sc-rect1} holds with equality. Therefore, we have
    \begin{align}
        \sepcaps{s}{\Phi} &= 2 \min_{\substack{A \subseteq \bigcup_{j \in \mathcal J} \psi_j(K_j) \\ \mathcal H^s(A)>0}} \dimR{\spanR{\Phi(A)}} \nonumber\\
        &\leq 2 \min_{j \in \mathcal J} \min_{\substack{A \subseteq \psi_j(K_j) \\ \mathcal H^s(A)>0}} \dimR{\spanR{\Phi(A)}} \label{eq:sc-rect2}\\
        &\leq 2 \min_{j \in \mathcal J} \min_{\substack{A \subseteq K_j \\ \mathcal L^s(A)>0}} \dimR{\spanR{(\Phi\circ\psi_j)(A)}} \label{eq:sc-rect3}\\
        &= \min_{j \in \mathcal J} \sepcap{\Phi \circ \psi_j}, \nonumber
    \end{align}
    where \cref{eq:sc-rect3} follows from \cref{rem:bi-Lip-H-meas}. We next show that \cref{eq:sc-rect2} holds with equality. To this end, let now $A \subseteq \bigcup_{j \in \mathcal J} \psi_j(K_j)$ be $\mathcal{H}^s$-measurable with $\mathcal{H}^s(A)>0$. Then, there must be an index $i \in \mathcal J$ such that the $\mathcal{H}^s$-measurable set $A' \coloneqq A \cap \psi_i(K_i)$ satisfies $\mathcal{H}^s(A')>0$. Since $A' \subseteq A$, 
    \begin{align*}
        \dimR{\spanR{\Phi(A')}} \leq \dimR{\spanR{\Phi(A)}},
    \end{align*}
    from which we deduce, following the same argument as above,  
    \begin{align*}
        \min_{j \in \mathcal J} \min_{\substack{A \subseteq \psi_j(K_j) \\ \mathcal H^s(A)>0}} \dimR{\spanR{\Phi(A)}} \leq \min_{\substack{A \subseteq \bigcup_{j \in \mathcal J} \psi_j(K_j) \\ \mathcal H^s(A)>0}} \dimR{\spanR{\Phi(A)}}.
    \end{align*}
    Finally, to show that \cref{eq:sc-rect3} holds with equality, we use again the same argument. Consider an arbitrary $\mathcal{H}^s$-measurable set $A \subseteq \psi_j(K_j)$ be  with $\mathcal{H}^s(A)>0$. By \cref{prop:regularity}, there is a closed set $A'\subset A$ with $\mathcal{H}^s(A')>0$. Then, 
    \begin{align*}
        \dimR{\spanR{\Phi(A')}} \leq \dimR{\spanR{\Phi(A)}},
    \end{align*}
    and since $\psi_j$ is bi-Lipschitz, $\psi_i^{-1}(A')$ is $\mathcal{L}^s$-measurable with $\mathcal{L}^s(\psi_j^{-1}(A'))>0$. It follows that
    \begin{align*}
        \min_{\substack{A \subseteq K_j \\ \mathcal L^s(A)>0}} \dimR{\spanR{(\Phi\circ\psi_j)(A)}} \leq \min_{\substack{A \subseteq \psi_j(K_j) \\ \mathcal H^s(A)>0}} \dimR{\spanR{\Phi(A)}},
    \end{align*}
    which completes the proof.
\end{proof}
Let us also study the identity map restricted to $E$, when $E$ countably $\mathcal{H}^s$-rectifiable. Recalling \cref{rem:decomposition-sparse-rect}, the bi-Lipschitz parametrization can be chosen such that the decomposition from \Cref{subsec:sep-gen-case} holds. Using \cref{prop:sc-rect}, we compute
\begin{align}\label{eq:s-sep-cap-id-rect}
    \sepcaps{s}{\restrinline{\mathrm{Id}}{E}} = 2\min_{j \in \mathcal{J}} s_j \geq 2s,
\end{align}
where $s_j \coloneqq \dimRinline{\spanRinline{\psi_j(K_j)}}$, and where the inequality follows from \cref{rem:lower-bound-by-rect-para}. Thus, the rectifiability parameter $s$ determines a lower bound for the $s$-separation capacity of the identity map on $E$. Note that \cref{eq:s-sep-cap-id-rect} holds with equality if one of the parametrization maps in $\{\psi_j\colon K_j \to \psi_j(K_j)\}_{j\in \mathcal J}$ is linear, i.e., if on a set of positive $\mathcal{H}^s$-measure, $E$ coincides with an $s$-dimensional linear subspace. For nonlinear parametrization maps $\{\psi_j\colon K_j \to \psi_j(K_j)\}_{j\in \mathcal J}$, the images $\{\psi_j(K_j)\}_{j \in \mathcal{J}}$ may exceed $s$-dimensional linear structure, and $\sepcapsinline{s}{\restrinline{\mathrm{Id}}{E}}$ can become strictly larger than $2s$. 

Intuitively, we can therefore conclude that datasets $E$ with nonlinear parametrizations and large rectifiability parameter $s$ (i.e., rich geometric structure and high intrinsic dimension) tend to yield high $s$-separation capacity $\sepcapsinline{s}{\restrinline{\mathrm{Id}}{E}}$. 

\section{Generalization and learning on low-dimensional datasets}\label{sec:generalization-learning}
In this section, we study another measure of classification capabilities closely related to the $s$-separation capacity: the probability of ambiguous generalization introduced by Cover \cite{cover1965geometrical}. It characterizes the ability of a transformation $\Phi$ to generalize beyond the points it has already separated. More precisely, given a $\Phi$-separable dichotomy $\{F_+, F_-\}$ of an $N$-point set $F \subset E$, the question is whether the realized dichotomy uniquely determines the label of a new point $g \in E$, or whether both assignments $g \in F_+$ and $g \in F_-$ remain compatible with $\{F_+, F_-\}$. It is clear that for certain dichotomies of $F$, the classification of $g$ will not be unique. In general, one may expect, however, that for $N$ large enough, the labeling of $g$ is unique.  
In \cite{cover1965geometrical}, the question of when unique generalization becomes probable was studied under the assumption that $F$ is in $\Phi$-general position. However, as previously noted, in general this assumption may not hold, specifically when $E$ exhibits low-dimensional structure in a measure-theoretic sense, i.e., when $\mathcal{L}^M(E)=0$.   
The goal of this section is to investigate when unique generalization becomes probable in the setting where $E$ is $\mathcal{H}^s$-measurable with positive and $\sigma$-finite $\mathcal{H}^s$-measure for some $s\geq 0$, and where the $\Phi$-general position assumption may fail. To this end, we will identify which results from \cite{cover1965geometrical} carry over directly to this setting and which require modification. 
We start with the formal definition of ambiguous generalization as introduced in \cite{cover1965geometrical}. 
\begin{definition}[Ambiguous generalization, \cite{cover1965geometrical}]\label{def:ambiguous}
    For $M,M',N \in \mathbb N$, let $F \coloneqq \{f_1,\ldots,f_N\} \subseteq E$, where $E \subseteq \mathbb R^M$, and let $\Phi \colon E \to \mathbb R^{M'}$. Suppose the dichotomy $\{F_+,F_-\}$ is $\Phi$-separable. We call $g \in E$ \emph{ambiguous with respect to $\{F_+,F-\}$} if both dichotomies $\{F_+ \cup \{g\},F_-\}$ and $\{F_+, F_- \cup \{g\}\}$ are $\Phi$-separable. Otherwise, $g \in E$ is said to be \emph{unambiguous with respect to $\{F_+,F-\}$}.  
\end{definition}
The concept of ambiguous generalization is illustrated in \cref{fig:amb-gen}.
\begin{figure}
    \centering
    \definecolor{blue2}{RGB}{20,99,178}
\definecolor{green2}{RGB}{0,95,1}
\colorlet{red2}{red!70!black}
\begin{tikzpicture}[scale=1, every node/.style={scale=1}]
    \begin{axis}[
        unit vector ratio*=1 1 1,
        grid=none,
        ticks =none,
        axis lines=middle,
        xmin=-3.1,
        xmax=3.1,
        ymin=-3.1,
        ymax=3.1,
        xticklabels={,,},
        yticklabels={,,}
    ]
    \addplot [only marks, mark=*, color=blue2] coordinates {
        (-1.51, -1.15)
        (-0.95, -0.95)
        (-2.1, -0.9)
        (-1.1,-0.35)
        (-2,-0.33)
        (-1.7,-0.64)
    };
    \addplot [only marks, mark=*, color=red2] coordinates {
        (2, 2)
        (1.2, 1.5)
        (2.5, 1)
        (1.24,0.75)
        (1.8,0.53)
    };
    \addplot [only marks, mark=*, color=green2, nodes near coords, nodes near coords align={north west}, point meta=explicit symbolic,] coordinates {
        (0.25, -1.1) [$g_2$]
        (-2.3, -1.8) [$g_1$]
    };
    \draw[color=blue2, dotted] (-1.56, -0.72) circle[radius=0.8];
    \node[color=blue2] at (-2, 0.25) {$F_+$};
    \draw[color=red2, dotted] (1.748, 1.156) circle[radius=1.1];
    \node[color=red2] at (1.748, -0.25) {$F_-$};
    \draw[thick,dashed]{} (2,6)--(-2,-6);
    \draw[thick,dashdotted]{} (3.6,-6)--(-3.6,6);
    \node[] at (2.9,2.9){$\mathbb R^2$};
    \end{axis}
    \end{tikzpicture}
    \vspace{-10mm}
    \caption{Ambiguous generalization with respect to the homogeneously linearly separable dichotomy $\{F_+,F_-\}$. The point $g_1$ is unambiguous and the point $g_2$ is ambiguous with respect to $\{F_+,F_-\}$. Indeed, the dashed separating surface assigns $g_2$ to $F_-$, while the dichotomy $\{F_+ \cup \{g_2\},F_-\}$ is realized by the dashed-dotted separating surface.}
    \label{fig:amb-gen}
\end{figure}
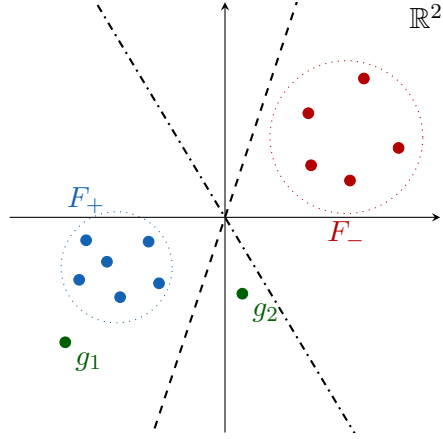
To determine whether a point is ambiguous or unambiguous with respect to a given dichotomy, the following lemma, established in \cite{cover1965geometrical}, is particularly useful and can be applied directly in our setting. It provides a necessary and sufficient condition for ambiguous generalization.
\begin{lemma}[\cite{cover1965geometrical}]\label{L:amb-gen}
    Let $F \coloneqq \{f_1,\ldots,f_N\} \subseteq E$, where $E \subseteq \mathbb R^M$ with $M,M',N \in \mathbb N$. Suppose $\Phi \colon E \to \mathbb R^{M'}$ is such that the dichotomy $\{F_+,F_-\}$ is $\Phi$-separable. The point $g \in E$ is ambiguous with respect to $\{F_+,F_-\}$ if and only if there is a $\Phi$-surface containing $g$ which realizes the dichotomy $\{F_+,F_-\}$.
\end{lemma}

Fix now an $N$-point set $F \subset E$, and let $g \in E$. We wish to compute the probability that $g$ is ambiguous with respect to a uniformly at random chosen $\Phi$-separable dichotomy of $F$, denoted $P(\Phi,F,g)$. It follows from \cref{L:amb-gen} that 
    \begin{align}\label{eq:prob-ambg}
        P(\Phi,F,g) = \frac{\# \text{ of $\Phi$-sep. dichotomies of $F$ s.t. sep. $\Phi$-surface contains $g$}}{\# \text{ of $\Phi$-sep. dichotomies of $F$}}.
    \end{align}
Note that the separating $\Phi$-surface containing $g$ achieves the dichotomy $\{F_+,F_-\}$ of $F$ if and only if there is a separating vector $w \in \mathbb R^{M'}$ such that
\begin{align}
    \innerprod{\Phi(f)}{w} &\geq 0, \text{ if $f \in F_+$}, \label{eq:sep-surf-g-+}\\
    \innerprod{\Phi(f)}{w} &< 0, \text{ if $f \in F_-$}, \label{eq:sep-surf-g--}\\
    \innerprod{\Phi(g)}{w} &= 0. \nonumber
\end{align}
In other words, there exists $w \in \{\Phi(g)\}^\perp$ satisfying \cref{eq:sep-surf-g-+,eq:sep-surf-g--}. Thus, we may also write $\innerprodinline{\Phi(f)}{P_{\{\Phi(g)\}^\perp} w}$ in \cref{eq:sep-surf-g-+,eq:sep-surf-g--}, and take $w\in \mathbb R^{M'}$, where $P_{\{\Phi(g)\}^\perp} \colon \mathbb R^{M'} \to \mathbb R^{M'}$ denotes the orthogonal projection onto the linear subspace $\{\Phi(g)\}^\perp$. But orthogonal projections are self-adjoint, which implies $\innerprodinline{\Phi(f)}{P_{\{\Phi(g)\}^\perp} w} = \innerprodinline{P_{\{\Phi(g)\}^\perp}\Phi(f)}{w}$.
Hence, setting $\widetilde \Phi_g \coloneqq P_{\{\Phi(g)\}^\perp} \circ \Phi$, one can write \cref{eq:prob-ambg} as
    \begin{align}\label{eq:prob-ambg-1}
        P(\Phi,F,g) = \frac{\# \text{ of $\widetilde\Phi_g$-sep. dichotomies of $F$}}{\# \text{ of $\Phi$-sep. dichotomies of $F$}}.
    \end{align}
We emphasize that for \cref{eq:prob-ambg-1} to hold, $F\cup \{g\}$ need not be in $\Phi$-general position.  

Let us analyze \cref{eq:prob-ambg-1} first under the assumptions made in \cite{cover1965geometrical}, namely, when $E = \mathbb R^M$ and $\Phi \colon \mathbb R^M \to \mathbb R^{M'}$ is such that $(\mathcal{L}^M)^{N'}$-a.e. $N'$-tuple is in $\Phi$-general position for every $N' \in \mathbb N$, i.e., $\sepcap{\Phi}=2M'$. In particular, $F\cup \{g\}$ can assumed to be in $\Phi$-general position. It follows that the points $\{\widetilde\Phi_g(f_k)\}_{k=1}^N$ lie in an $(M'-1)$-dimensional linear subspace, satisfying \cref{A:sep-sparse-basis} with $s=M'-1$. Consequently, \cref{rem:all-in-one} can be leveraged to compute the number of $\widetilde\Phi_g$-separable dichotomies of $F$, yielding the result of $C(N,M'-1)$. 
Thus, assuming $F\cup \{g\}$ is in $\Phi$-general position, we obtain
\begin{align*}
        P(\Phi,F,g) = \frac{C(N,M'-1)}{C(N,M')}.
\end{align*}
In \cite{cover1965geometrical}, the asymptotic properties of this quantity are studied as $M' \to \infty$. Specifically, it is shown that  
\begin{align}\label{eq:asmyp_abg_gen}
    P^*(\beta) \coloneqq\lim_{\substack{N = \lfloor \beta M'\rfloor \\ M' \to \infty}} \frac{C(N,M'-1)}{C(N,M')} = \begin{cases}
        1, & \text{if $\beta \in [0,2]$,} \\
        \frac{1}{\beta-1}, & \text{if $\beta \in (2, \infty)$,}
    \end{cases} \quad \beta \in  \mathbb R_0^+.
\end{align}
We refer to $P^*(\beta)$, $\beta \in  \mathbb R_0^+$, as the asymptotic probability of ambiguous generalization under $\Phi$-general position assumption. See \cref{fig:prob-amb-gen-pos} for an illustration. Observe that $P^*$ exhibits a decline at $\beta = 2$. Unambiguous generalization occurs with positive probability if $N>2M'$. But recall that $\sepcap{\Phi}=2M'$ if we assume $\Phi \colon \mathbb R^M \to \mathbb R^{M'}$ is such that $(\mathcal{L}^M)^{N'}$-a.e. $N'$-tuple is in $\Phi$-general position for every $N' \in \mathbb N$. Thus, the separation capacity serves as the threshold indicating when unambiguous generalization becomes probable.
\begin{figure}
    \centering
    \definecolor{blue2}{RGB}{20,99,178}
\definecolor{green2}{RGB}{0,95,1}
\colorlet{red2}{red!70!black}
\begin{tikzpicture}[scale=0.75, every node/.style={scale=0.75}]
    \begin{axis}[
        at={(0,0)},
        xmin=-0.1,
        xmax=6,
        ymin=-0.1,
        ymax=1.1,
        clip=false,
        xtick={0,1,2,3,4,5,6},
        ytick={0,0.25,...,1},
        yticklabels={0,0.25,...,1},
        xticklabels={0,1,2,3,4,5,6},
        ylabel={$P^*(\beta)$},
        xlabel={$\beta$},
        legend cell align={left}
    ]
    \addplot[thick,color=blue2,samples at={0,2},variable=\x] {1};
    \addplot[thick,color=blue2,domain=2:6,variable=\x] {1/(\x-1)};
    \draw[thick,dotted,gray] (2,-0.1)--(2,1);
\end{axis}
\end{tikzpicture}
    \vspace{-10mm}
    \caption{Asymptotic probability of ambiguous generalization under $\Phi$-general position assumption.}
    \label{fig:prob-amb-gen-pos}
\end{figure}
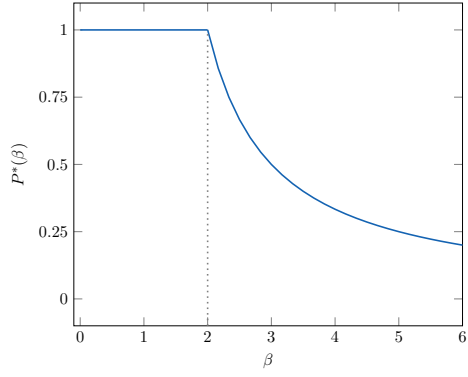

Let us now generalize the results of \cite{cover1965geometrical} by analyzing \cref{eq:prob-ambg-1} in the broader setting where $E$ is $\mathcal{H}^s$-measurable with positive and $\sigma$-finite $\mathcal{H}^s$-measure, $\Phi \colon E \to \mathbb R^{M'}$ is arbitrary, and $F \cup \{g\}$ need not be $\Phi$-general position. To this end, we recall the framework introduced in \Cref{subsec:sep-gen-case} and decompose $E$ into $\{E_j\}_{j \in \mathcal{J}}$ according to \ref{it:i}--\ref{it:iii}. We next extend \cref{prop:sep-general-case} by imposing the additional constraint that the separating surface must pass through a prescribed point $g$.
\begin{proposition}\label{prop:sep-constraint}
     Let $g \in E_\ell$, for some $\ell \in \mathcal{J}$, be such that for all $\mathfrak{s} \subseteq \mathcal{J}$, 
     \begin{align}\label{eq:assumption-g}
         \Phi(g) \notin \sum_{j \in \mathfrak{s}} \spanR{\Phi(E_j)},\text{ whenever }\Phi(E_\ell)\not\subseteq \sum_{j \in \mathfrak{s}} \spanR{\Phi(E_j)}.
     \end{align}
     Moreover, suppose that the $N$-point set $F \subseteq E$ satisfies \cref{A:sep-general-case} with respect to $\widetilde \Phi_g$. Then, the number of $\Phi$-separable dichotomies of $F$ subject to the condition that the $\Phi$-separating surface contains $g$ 
     is given by
     \begin{align*}
            C(\underline{N},\ell,\Phi,s) \coloneqq 2^N - 2 \sum_{t=s^*+1}^{N+1} \sum_{\underline \nu \in \mathcal{I}_{t}} \binom{N_\ell}{\nu_\ell -1}\prod_{j \in \mathcal{J}_0 \setminus \ell}\binom{N_j}{\nu_j}.
        \end{align*}
        Here, $\mathcal{J}_{0} \coloneqq \supp{\underline{N}} \cup \{\ell\}$, $s^* \coloneqq \min_{j \in \mathcal{J}_0} s_j$, and $\mathcal{I}_{t}\coloneqq\{\underline \nu \in \mathbb N_0^{\mathcal{J}_0} \colon \lenin{\underline{\nu}} = t, \Upsilon(\underline \nu) \not\equiv t \Mod{2}\}$ with 
        \begin{align*}
            \Upsilon(\underline \nu) &\coloneqq \min_{\mathfrak{s} \subseteq \suppinline{\underline{\nu}}} \left\{\sum_{j \in \mathfrak{s}} \nu_j + \dimR{\spanR{\bigcup_{j \in \mathfrak{s}^c} \Phi(E_j)}}\right\},
        \end{align*}
        for all $\underline \nu = (\nu_j)_{j\in\mathcal{J}_0} \in \mathbb N_0^{\mathcal{J}_0}$. 
\end{proposition}
\begin{remark}\label{rem:assump-sep-constraint}
    For $(\mathcal{H}^s)^N$-a.e. $N$-tuple $F \in E^N$ and $\mathcal{H}^s$-a.e. $g \in E$, the assumptions of \cref{prop:sep-constraint} hold. Indeed, from \cref{rem:assumptions-almost-surely-satisfied}, we know that $\mathcal{H}^s$-a.e. $g \in E$ satisfies \cref{A:sep-general-case}, but \cref{it:sep-gen-iii} in \cref{A:sep-general-case} coincides with \cref{eq:assumption-g}. Furthermore, we have, for every $\mathcal{H}^s$-measurable $A\subseteq E_j$, $j \in \mathcal{J}$, with positive $\mathcal{H}^s$-measure,  
    \begin{align}
        &\dimR{\spanR{\widetilde \Phi_g(A)}}\nonumber\\ 
        &= \dimR{\spanR{\Phi(A)}} - \dimR{\spanR{\{\Phi(g)\}}\cap \spanR{\Phi(A)}} \label{eq:decomp-valid-Phi_g-1}   \\
        &= \dimR{\spanR{\Phi(E_j)}} - \dimR{\spanR{\{\Phi(g)\}}\cap \spanR{\Phi(E_j)}}, \label{eq:decomp-valid-Phi_g-2}
    \end{align}
    where \cref{eq:decomp-valid-Phi_g-1} is by the rank--nullity theorem, and in \cref{eq:decomp-valid-Phi_g-2}, we used that for $\{E_j\}_{j \in \mathcal{J}}, $\ref{it:iii} holds. Note that the RHS in \cref{eq:decomp-valid-Phi_g-2} does not depend on $A$. Thus, $\{E_j\}_{j \in \mathcal{J}}$ also constitutes a valid decomposition in the sense of \ref{it:i}--\ref{it:iii} with respect to the map $\widetilde \Phi _g$. Application of \cref{rem:assumptions-almost-surely-satisfied} then establishes that for $(\mathcal{H}^s)^N$-a.e. $N$-tuple $F \in E^N$, \cref{A:sep-general-case} holds with respect to the map $\widetilde{\Phi}_g$.    
\end{remark}
\begin{proof}[Proof of \cref{prop:sep-constraint}]
    From the derivation of \cref{eq:prob-ambg-1}, we know that $C(\underline{N},\ell,\Phi,s)$ equals the number of $\widetilde \Phi_g$-separable dichotomies of $F$. As $F$ satisfies \cref{A:sep-general-case} with respect to the map $\widetilde \Phi_g$, application of \cref{prop:sep-general-case} particularized to $\widetilde \Phi_g$ yields
    \begin{align}\label{eq:fct-count-contr-sep-surface-1}
        C(\underline{N},\ell,\Phi,s) = 2^N -2 \sum_{t=s_*}^N \sum_{\underline{\nu} \in \widetilde{ \mathcal{I}}_{g,t}} \prod_{j \in \mathcal{J}_0} \binom{N_j}{\nu_j}, 
    \end{align}
    where we used that, by the rank--nullity theorem, \begin{align*}
    \dimR{\spanR{\widetilde\Phi_g(E_j)}} &= \dimR{\spanR{\Phi(E_j)}} \\& \quad - \dimR{\spanR{\{\Phi(g)\}} \cap \spanR{\Phi(E_j)}} \\
    &\geq s^* -1, \quad \text{for all $j \in \mathcal{J}_0$}.
    \end{align*}
    Here, $\widetilde{\mathcal I}_{g,t} \coloneqq \{\underline \nu \in \mathbb N_0^{\mathcal{J}_0} \colon \lenin{\underline{\nu}} = t, \widetilde\Upsilon_g(\underline \nu) \not\equiv t \Mod{2}\}$ and
    \begin{align*}
        \widetilde\Upsilon_g(\underline \nu) \coloneqq \min_{\mathfrak{s} \subseteq \suppinline{\underline{\nu}}} \left\{\sum_{j \in \mathfrak{s}} \nu_j + \dimR{\spanR{\bigcup_{j \in \mathfrak{s}^c} \widetilde\Phi_g(E_j)}}\right\}.
    \end{align*}
    Note that the constraint in the minimum $\mathfrak{s} \subseteq \suppinline{\underline{\nu}}$ can equivalently be replaced by $\mathfrak{s} \subseteq \mathcal{J}_0$, where $\mathfrak{s}^c$ denotes the complement with respect to the indexing set appearing in the minimization (i.e., $\suppinline{\underline{\nu}}$ or $\mathcal{J}_0$). We next compute
    \begin{align}
        \widetilde\Upsilon_g(\underline{\nu}) &= \min_{\mathfrak{s} \subseteq \mathcal{J}_0} \left\{\sum_{j \in \mathfrak{s}} \nu_j + \dimR{\sum_{j \in \mathfrak{s}^c} \spanR{\widetilde\Phi_g(E_j)}}\right\} \nonumber\\
        &=\min_{\mathfrak{s} \subseteq \mathcal{J}_0} \left\{\sum_{j \in \mathfrak{s}} \nu_j + \dimR{P_{\{\Phi(g)\}^\perp}\!\left(\sum_{j \in \mathfrak{s}^c} \spanR{\Phi(E_j)}\right)}\right\} \label{eq:upsilon-g-0}\\
        &= \min_{\mathfrak{s} \subseteq \mathcal{J}_0} \left\{\sum_{j \in \mathfrak{s}} \nu_j + \dimR{\sum_{j \in \mathfrak{s}^c} \spanR{\Phi(E_j)}} \right.- \label{eq:upsilon-g-1}
        \\
        &\quad\quad\quad\quad\quad\quad\quad\left.\dimR{\spanR{\{\Phi(g)\}} \cap \sum_{j \in \mathfrak{s}^c} \spanR{\Phi(E_j)} }\right\} \nonumber \\
        &= \min_{\mathfrak{s} \subseteq \mathcal{J}_0} \left\{\sum_{j \in \mathfrak{s}} \nu_j + \dimR{\sum_{j \in \mathfrak{s}^c} \spanR{\Phi(E_j)}} - \mathbbm{1}_{\left\{\Phi(E_\ell)  \subseteq \sum_{j \in \mathfrak{s}^c} \spanR{\Phi(E_j)}\right\}}\right\} \label{eq:upsilon-g-2} \\
        &= \min_{\mathfrak{s} \subseteq \mathcal{J}_0} \left\{\sum_{j \in \mathfrak{s}} \nu_j + \dimR{\sum_{j \in \mathfrak{s}^c} \spanR{\Phi(E_j)}} - \mathbbm{1}_{\left\{\ell \in \mathfrak{s}^c\right\}}\right\} \label{eq:upsilon-g-3} \\
        &= \min_{\mathfrak{s} \subseteq \mathcal{J}_0} \left\{\sum_{j \in \mathfrak{s}} \nu_j + \dimR{\sum_{j \in \mathfrak{s}^c} \spanR{\Phi(E_j)}} + \mathbbm{1}_{\left\{\ell \in \mathfrak{s}\right\}}\right\} -1 \nonumber\\
        &= \Upsilon(\underline{\nu} +\underline{e}_\ell) -1, \nonumber
    \end{align}
    where \cref{eq:upsilon-g-0} holds as $P_{\{\Phi(g)\}^\perp}$ is linear, \cref{eq:upsilon-g-1} follows from the rank--nullity theorem, and \cref{eq:upsilon-g-2} is by \cref{eq:assumption-g}. Finally, \cref{eq:upsilon-g-3} is valid because $\ell \in \mathfrak{s}^c$ implies $\Phi(E_\ell) \subseteq \sum_{j \in \mathfrak{s}^c} \spanRinline{\Phi(E_j)}$, and because, conversely, whenever $\Phi(E_\ell) \subseteq \sum_{j \in \mathfrak{s}^c} \spanRinline{\Phi(E_j)}$, the minimization allows us to include $\ell \in \mathfrak{s}^c$.

    Thus, $\underline{\nu} \in \widetilde{\mathcal{I}}_{g,t}$ if and only if $(\underline{\nu} +\underline{e}_\ell) \in \mathcal{I}_{t+1}$. Using \cref{eq:fct-count-contr-sep-surface-1}, we therefore obtain
    \begin{align*}
        C(\underline{N},\ell,\Phi,s) &= 2^N -2 \sum_{t=s^*}^N \sum_{\underline{\nu} \in \mathcal{I}_{t+1}} \binom{N_\ell}{\nu_\ell -1}\prod_{j \in \mathcal{J}_0 \setminus \{\ell\}}\binom{N_j}{\nu_j} \\
        &= 2^N -2 \sum_{t=s^*+1}^{N+1} \sum_{\underline{\nu} \in \mathcal{I}_{t}} \binom{N_\ell}{\nu_\ell -1}\prod_{j \in \mathcal{J}_0 \setminus \{\ell\}}\binom{N_j}{\nu_j},
    \end{align*}
    as desired.
\end{proof}

Combining \cref{prop:sep-general-case,prop:sep-constraint}, we have, 
by \cref{rem:assumptions-almost-surely-satisfied,rem:assump-sep-constraint}, for $(\mathcal{H}^s)^N$-a.e. $N$-tuple $F \in E^N$ and $\mathcal{H}^s$-a.e. $g \in E$,
\begin{align}\label{eq:prob-amb-gen}
    P(\Phi,F,g) = \frac{C(\underline{N},\ell,\Phi,s)}{C(\underline{N},\Phi,s)} = \frac{2^N -2 \sum_{t=s^*+1}^{N+1} \sum_{\underline{\nu} \in \mathcal{I}_{t}} \binom{N_\ell}{\nu_\ell -1}\prod_{j \in \mathcal{J}_0 \setminus \{\ell\}}\binom{N_j}{\nu_j}}{2^N -2 \sum_{t=s^*+1}^{N+1} \sum_{\underline{\nu} \in \mathcal{I}_{t}} \prod_{j \in \mathcal{J}_0 }\binom{N_j}{\nu_j}}.
\end{align}

In what follows, we will analyze the relation between $P(\Phi,F,g)$ and the $s$-separation capacity $\sepcaps{s}{\Phi}$ and particularize the RHS of \cref{eq:prob-amb-gen} to two extreme cases. To this end, set $s_{j^*} \coloneqq \min_{j \in \mathcal J} s_j$. Then, $C(\underline{N},\Phi,s) \geq C(N,s_{j^*})$ by \cref{thm:lower-upper-bound-C_F}. For $\eta \in (0,1)$, compute
\begin{align*}
    \lim_{\substack{M' \to \infty \\ \substack{s_{j^*}/M' \text{ fixed} \\ N=\lfloor 2s_{j^*}(1-\eta) \rfloor}}}\frac{C(\underline{N},\Phi,s)}{2^N} \geq \lim_{\substack{s_{j^*} \to \infty \\ N=\lfloor 2s_{j^*}(1-\eta) \rfloor}} \frac{C(N,s_{j^*})}{2^N} = 1,
\end{align*}
where the equality was shown in \cite{cover1965geometrical}. Thus, every dichotomy of a set of points in $E$ of cardinality less than $2s_{j^*}$ is asymptotically $\Phi$-separability with probability one as $M'\to \infty$ with $s_{j^*}/M'$ being fixed. But then $g$ can be assigned to any $F_+$ or $F_-$ to yield an asymptotic $\Phi$-separable dichotomy if the cardinality of $F$ satisfies $N<2s_{j^*}$. Indeed, in this case, we have 
\begin{align*}
    \frac{N+1}{s_{j^*}}=\underbrace{\frac{N}{s_{j^*}}}_{<2}+\frac{1}{s_{j^*}} <2, \quad \text{ for $s_{j^*}$ large enough.}
\end{align*}
Hence, every dichotomy of $F \cup \{g\}$ is asymptotically $\Phi$-separable with probability one if $N<2s_{j^*}$. Recalling that $\sepcaps{s}{\Phi}=2s_{j^*}$, we can infer that the $s$-separation capacity serves as bound to have ambiguous generalization with probability one for $(\mathcal{H}^s)^N$-a.e. $N$-tuple $F$ and $\mathcal{H}^s$-a.e. $g$ in the regime $M'\to \infty$ with fixed $s_{j^*}/M'$. In particular, if $N<2s_{j^*}$, then unambiguous generalization is of probability zero for $(\mathcal{H}^s)$-a.e. $N$-tuple $F$ and $\mathcal{H}^s$-a.e. $g$. Exemplifying this observation to the setting where $E$ is countably $\mathcal{H}^s$-rectifiable and $\Phi = \mathrm{Id}$, one can deduce that ambiguous generalization occurs with probability one if $N<2s$, as $s_{j^*} \geq s$ by \cref{rem:lower-bound-by-rect-para}. Thus, the bound to have ambiguous generalization with probability one increases with the rectifiability parameter $s$.      

We further note that if $F \subseteq E_{j^*}$ and $g \in E_{j^*}$, then $C(\underline{N},\Phi,s) = C(N,s_{j^*})$ by \cref{rem:all-in-one}, and
\begin{align*}
    C(\underline{N},\ell,\Phi,s) &= 2^N -2 \sum_{t=s_{j^*}+1}^{N+1} \sum_{\underline{\nu} \in \mathcal{I}_{t}} \binom{N_{j^*}}{\nu_{j^*} -1}\\
    &= 2^N - 2 \sum_{\substack{t = s_{j^*}+1 \\ (t-s_{j^*})\text{ odd}}}^{N+1} \binom{N}{t-1} \\
    &= 2^N - 2 \sum_{\substack{t = (s_{j^*}-1)+1 \\ (t-(s_{j^*}-1))\text{ odd}}}^{N} \binom{N}{t} \\
    &= C(N,s_{j^*}-1),
\end{align*}
using \cref{rem:all-in-one} again in the last step. 
Thus, if $F \subseteq E_{j^*}$ and $g \in E_{j^*}$,
\begin{align*}
    P(\Phi,F,g) = \frac{C(\underline{N},\ell,\Phi,s)}{C(\underline{N},\Phi,s)} = \frac{C(N,s_{j^*}-1)}{C(N,s_{j^*})}.
\end{align*}
Then, the asymptotic probability of ambiguous generalization takes the form \cref{eq:asmyp_abg_gen}, depicted in \cref{fig:prob-amb-gen-pos}, and unambiguous generalization occurs with positive probability if $N>2s_{j^*}$. 

Finally, let us emphasize that if $g\in E_\ell$ with $\Phi(E_\ell) \not \subseteq \spanRinline{\bigcup_{j \in \suppinline{\underline{N}}} \Phi(E_j)}$, then $N_\ell =0$ and for all $\underline{\nu} \in \mathbb N_0^{\mathcal{J}_0}$ with $\nu_\ell = 0$,
\begin{align*}
    \Upsilon(\underline{\nu} + \underline{e}_{\ell}) = \Upsilon(\underline{\nu}) + 1,
\end{align*}
by virtue of \cref{eq:upsilon-g-2}. Thus, 
\begin{align}\label{eq:p-amb-gen=1-1}
    \left((\underline{\nu}+\underline{e}_\ell) \in \mathcal{I}_{t}\right) \iff  \left(\underline{\nu} \in \mathcal{I}_{t-1}\right),     
\end{align}
whenever $\underline{\nu} \in \mathbb N_0^{\mathcal{J}_0}$ with $\nu_\ell = 0$. We then obtain 
\begin{align}
    C(\underline N,\ell,\Phi,s) &= 2^N -2 \sum_{t=s^*+1}^{N+1} \sum_{\underline{\nu} \in \mathcal{I}_{t}} \binom{N_\ell}{\nu_\ell -1}\prod_{j \in \mathcal{J}_0 \setminus \{\ell\}}\binom{N_j}{\nu_j} \nonumber\\
    &= 2^N -2 \sum_{t=s^*+1}^{N+1} \sum_{\substack{\underline{\nu} \in \mathcal{I}_{t}\\ \nu_\ell =1}} \prod_{j \in \mathcal{J}_0 \setminus \{\ell\}}\binom{N_j}{\nu_j} \nonumber\\
    &= 2^N -2 \sum_{t=s^*+1}^{N+1} \sum_{\substack{\underline{\nu} \in \mathcal{I}_{t-1}\\\nu_\ell = 0}} \prod_{j \in \mathcal{J}_0 \setminus \{\ell\}}\binom{N_j}{\nu_j} \label{eq:p-amb-gen=1-2}\\
    &= 2^N -2 \sum_{t=s^*+1}^{N} \sum_{\underline{\nu} \in \mathcal{I}_{t}} \prod_{j \in \mathcal{J}_0 \setminus \{\ell\}}\binom{N_j}{\nu_j} \label{eq:p-amb-gen=1-3} \\
    &=C(\underline N, \Phi,s), \nonumber
\end{align}
where \cref{eq:p-amb-gen=1-2} is by \cref{eq:p-amb-gen=1-1}, and in \cref{eq:p-amb-gen=1-3}, we used that $\mathcal{I}_{s^*} = \emptyset$, so that $P(\Phi,F,g) = 1$.  
In other words, in this case, unambiguous generalization is not probable (i.e., occurs with probability zero) for every $N \in \mathbb N$.

\appendix
\section{Notation}\label[appendix]{app:notation}
$\mathbb N$, $\mathbb N_0$, $\mathbb Z$, $\mathbb R$, and $\mathbb R_0^+$ denote the sets of natural numbers, nonnegative integers, integers, real numbers, and nonnegative real numbers, respectively.  
For $a,b\in \mathbb Z$ and $m \in \mathbb N$, we write $a \equiv b \Mod{m}$ whenever $m$ divides $(a-b)$. The binomial coefficient is defined as $\binom{n}{k}\coloneqq \frac{n!}{k!(n-k)!}$ for all $k,n \in \mathbb N_0$ with $0\leq k \leq n$. Moreover, if $n<k$, we set $\binom{n}{k}\coloneqq 0$. For multi-indices $\underline{N}=(N_j)_{j=1}^J,\underline{\nu}=(\nu_j)_{j=1}^J \in \mathbb N_0^J$ with $J\in \mathbb N$, we write $\binom{\underline{N}}{\underline{\nu}} \coloneqq \prod_{j=1}^J \binom{N_j}{\nu_j}$. Furthermore, the support and absolute value of $\underline{\nu}=(\nu_j)_{j=1}^J \in \mathbb N_0^J$ is given by $\supp{\underline{\nu}} = \{j\in\{1,\ldots,J\}\colon \nu_j \neq 0\}$ and $\lenin{\underline{\nu}} = \sum_{j=1}^J \nu_j$, respectively.
Let $\lfloor x \rfloor$ denote the largest $k \in \mathbb Z$ such that $k \leq x$, where $x \in \mathbb R$. 
To represent the indicator of a statement $S$, we write $\mathbbm{1}_{\{S\}}$, which equals $1$ if the statement $S$ is true, and 0 if $S$ is false. For a finite set $X$, let $\lvert X\rvert$ denote its cardinality.
We use $x^\mathsf{T}$ to denote the transpose of $x \in \mathbb R ^n$, $n \in \mathbb N$.
The standard Euclidean inner product of $x,y \in \mathbb R^n$ is $\langle x, y \rangle = y^\mathsf{T}x$, and its induced norm on $\mathbb R^n$ is given by $\lVert x \rVert \coloneqq \sqrt{\langle x, x \rangle}$. For a set $A \subseteq \mathbb R^n$, let $\spanR{A}$ stand for the set of all finite linear combinations of vectors in $A$ with scalars in the field $\mathbb R$. Given a linear space $V$ over $\mathbb R$, we write $\dimR{V}$ for its dimension. Moreover, if $\{U_k\}_{k \in \mathcal{K}}$ is a family of linear subspaces of $V$, then the sum of these linear subspaces is denoted by $\sum_{k \in \mathcal{K}} U_k \coloneqq \{\sum_{k \in \mathcal{K}} u_k\colon u_k \in U_k, k\in \mathcal{K}\}$. 
For a finite set $A \subset \mathbb R^n$, we denote by $\kr{A}$ the Kruskal rank of $A$, i.e., the largest integer $k$ such that every subset of $k$ elements of $A$ is linearly independent. 
The $n$-dimensional Lebesgue measure on $\mathbb R^n$ is denoted by $\mathcal{L}^n$. For $s\geq 0$, $\mathcal{H}^s$ stands for the $s$-dimensional Hausdorff measure on some metric space (see \Cref{app:Hausdorff}). A statement $S$ is said to hold for $\mu$-almost every $x \in A$ ($\mu$-a.e. $x \in A$ for short) if there exists a set $N \subset X$ with $\mu(N)=0$ such that $S$ is true for every $x \in A \setminus N$, where $\mu$ is a measure on some set $X$, and where $A \subseteq X$. The restriction of a measure $\mu$ to a subset $A$ is denoted by $\mu \llcorner A$. 

\section{Cover's framework and fundamentals of function-counting theory} \label[appendix]{app:Cover}
This section introduces Cover's framework \cite{cover1965geometrical} and reviews some key results from function-counting theory \cite{schlafli1950theorie,winder1961single,wendel1962problem,cover1965geometrical,winder1966partitions,harding1967number,kowalczyk1994separating}, using mostly the notation of \cite{haberle2026scattering}. A central ingredient in this framework is the \emph{pattern space}, represented as the $M$-dimensional Euclidean space $\mathbb R^M$ equipped with the standard inner product $\innerprod{\cdot}{\cdot}$. 
Let $E \subseteq \mathbb R^M$ be an arbitrary subset of the pattern space, and consider a set of $N$ points (patterns) $F\coloneqq\{f_1,\ldots,f_N\} \subseteq E$, where $N \in \mathbb N$. We are concerned with the problem of binary classification of the points in $F$, i.e., assigning the elements of the set $F$ to one of the two classes $F_+$ and $F_-$. Such a partition of $F$ into $F_+$ and $F_-$ is called a \emph{dichotomy}. The simplest way to implement a dichotomy is by using a hyperplane as separating surface, see \cref{fig:sep-dataset}. A dichotomy $\{F_+,F_-\}$ is said to be \emph{linearly separable} if there exist $w \in \mathbb R^M$ and $t \in \mathbb R$ such that
\allowdisplaybreaks
\begin{align*}
        \innerprod{f}{w} &> t, \quad \text{if $f\in F_+$}, \\
        \innerprod{f}{w} &< t, \quad \text{if $f\in F_-$}. 
\end{align*}
When $t=0$, we speak of \emph{homogeneous linear separation}. 
The surface $\{f \in \mathbb R^M \colon \innerprod{f}{w} = t\}$ is called the \emph{separating hyperplane}. 
\begin{figure}
    \centering
    \definecolor{blue2}{RGB}{20,99,178}
\definecolor{green2}{RGB}{0,95,1}
\colorlet{red2}{red!70!black}
\colorlet{colourE}{green2}
\begin{tikzpicture}[scale=1, every node/.style={scale=1}]
    \begin{axis}[
        unit vector ratio*=1 1 1,
        grid=none,
        ticks =none,
        axis lines=middle,
        xmin=-3.1,
        xmax=3.1,
        ymin=-3.1,
        ymax=3.1,
        xticklabels={,,},
        yticklabels={,,},
        at={(0,0)},
        axis on top
    ]
    \foreach \x in {-1,-0.4,-0.5}{
        \edef\temp{\noexpand\addplot[only marks, mark=*, color=red2] coordinates {({\x},{2.5*\x})};}
        \temp
    }
    \foreach \x in {0.8,0.6,1.25,1.5}{
        \edef\temp{\noexpand\addplot[only marks, mark=*, color=red2] coordinates {({\x},{-1.25*\x})};}
        \temp
    }
    \foreach \x in {0.6,0.8}{
        \edef\temp{\noexpand\addplot[only marks, mark=*, color=blue2] coordinates {({\x},{2.5*\x)})};}
        \temp
    }
    \foreach \x in {-0.4,-1.4,-1}{
        \edef\temp{\noexpand\addplot[only marks, mark=*, color=blue2] coordinates {({\x},{-1.25*\x)})};}
        \temp
    }
    \node[] at (2.9,2.9){$\mathbb R^2$};
    \node[colourE] at (2.7,-2.9) {$E$};
    
    \addplot[domain=-3.1:3.1,samples=2,color=colourE,thick] {2.5*x};
    \addplot[domain=-3.1:3.1,samples=2,color=colourE,thick] {-1.25*x};
    \addplot[domain=-3.1:3.1,samples=2,color=black,thick,dashed,name path=H] {0.5*x};
    \draw[->,thick,gray,dashed]{} (0,0)--(-1,2);
    \node[gray] at (-0.3,1.2) {$w$};
    \node[color=red2] at (-2.8, -2) {$F_-$};
    \node[color=blue2] at (-2.8, -1) {$F_+$};
    \addplot[name path=B,white] {-4};
    \addplot[name path=D,white] {4};
    \addplot[pattern color=blue2!5,pattern=crosshatch] fill between[of=H and D, soft clip={domain=-3.1:3.1}];
    \addplot[pattern color=red2!5,pattern=crosshatch] fill between[of=H and B, soft clip={domain=-3.1:3.1}];
    \end{axis}
\end{tikzpicture}
    \caption{Separation of points on a low-dimensional dataset $E$ by a hyperplane (dashed line) through the origin. Specifically, the dichotomy $\{F_+,F_-\}$ is homogeneously linearly separable.}
    \label{fig:sep-dataset}
\end{figure}
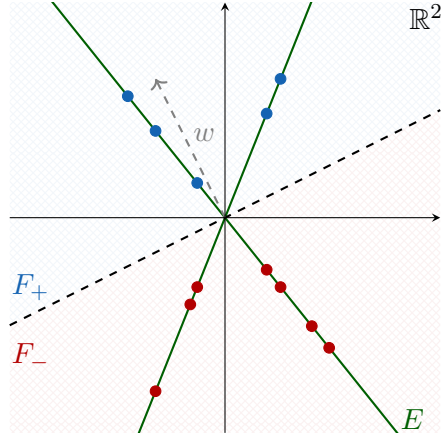
\sloppy In practice, however, most dichotomies we wish to realize are not linearly separable, i.e., they cannot be realized by separation through hyperplanes in the pattern space. Thus, more general nonlinear separating surfaces are required. To resolve this issue, one follows Cover's idea \cite{cover1965geometrical} of first mapping the points in $E \subseteq \mathbb R^M$ to another space, typically a higher-dimensional one, designated as \emph{feature space}, by employing a nonlinear transformation $\Phi \colon E \to \mathbb R^{M'}$. The goal is to choose $\Phi$ such that the dichotomies become linearly separable in the feature space while keeping the dimension of the feature space, $M'$, as small as possible. The separating surface in the pattern space then becomes a nonlinear surface characterized by $\Phi$, see \cref{fig:nonlinear_transf}. \emph{Homogeneous} linear separation in the feature space can always be achieved, provided that linear separation is possible in the feature space associated with $\Phi$, by considering the transformation $f \mapsto (1, \Phi(f))^{\mathsf{T}}$. Formally, the concept of obtaining homogeneous linear separation by employing a transformation $\Phi$ is captured by the notion of $\Phi$-separability.
\begin{definition}[$\Phi$-separability]\label{def:phi-sep}
    For $M,M',N \in \mathbb N$, let $F \coloneqq \{f_1,\ldots,f_N\} \subseteq E$, where $E \subseteq \mathbb R^M$, and let $\Phi \colon E \to \mathbb R^{M'}$. A dichotomy $F=\{F_+,F_-\}$ is called \emph{$\Phi$-separable} if there exists a vector $w \in \mathbb R^{M'}$ such that
    \begin{align*}
        \innerprod{\Phi(f)}{w} &> 0, \quad \text{if $f\in F_+$,} \\
        \innerprod{\Phi(f)}{w} &< 0, \quad \text{if $f\in F_-$}. 
    \end{align*}
    We call $\{f \in E \colon \innerprod{\Phi(f)}{w} = 0\}$ the \emph{separating $\Phi$-surface}.
\end{definition}
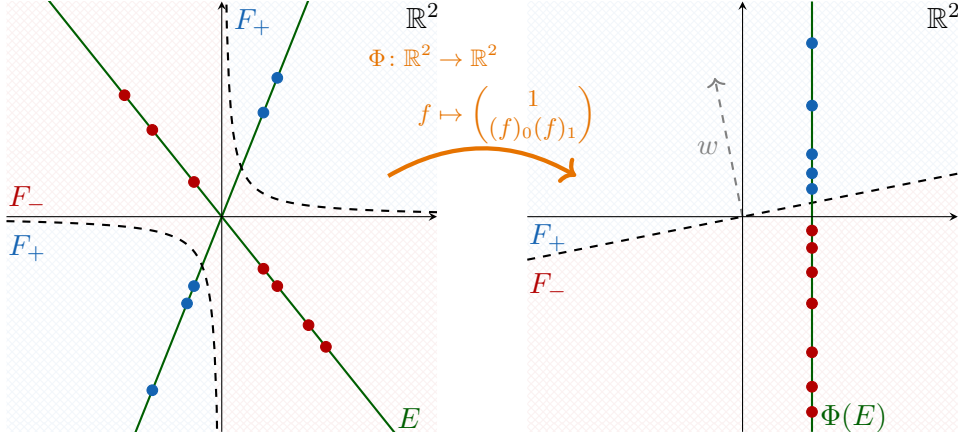
\begin{figure}
    \centering
    \definecolor{blue2}{RGB}{20,99,178}
\definecolor{green2}{RGB}{0,95,1}
\colorlet{red2}{red!70!black}
\colorlet{colourE}{green2}
\begin{tikzpicture}[scale=1, every node/.style={scale=1}]
    \begin{axis}[
        unit vector ratio*=1 1 1,
        grid=none,
        ticks =none,
        axis lines=middle,
        xmin=-3.1,
        xmax=3.1,
        ymin=-3.1,
        ymax=3.1,
        xticklabels={,,},
        yticklabels={,,},
        at={(0,0)},
        axis on top,
    ]
    \foreach \x in {-1,-0.4,-0.5,0.6,0.8}{
        \edef\temp{\noexpand\addplot[only marks, mark=*, color=blue2] coordinates {({\x},{2.5*\x})};}
        \temp
    }
    \foreach \x in {-0.4,-1.4,-1,1.25,0.8,0.6,1.5}{
        \edef\temp{\noexpand\addplot[only marks, mark=*, color=red2] coordinates {({\x},{-1.25*\x})};}
        \temp
    }
    \node[] at (2.9,2.9){$\mathbb R^2$};
    \node[colourE] at (2.7,-2.9) {$E$};
    \addplot[domain=-3.1:3.1,samples=2,color=colourE,thick] {2.5*x};
    \addplot[domain=-3.1:3.1,samples=2,color=colourE,thick] {-1.25*x};
    \node[color=red2] at (-2.8, 0.25) {$F_-$};
    \node[color=blue2] at (-2.8, -0.45) {$F_+$};
    \node[color=blue2] at (0.45, 2.8) {$F_+$};
    \addplot[domain=-3.1:-0.01,samples=100,thick,dashed,name path=A] {1/(5*x)};
    \addplot[domain=3.1:0.01,samples=100,thick,dashed,name path=C] {1/(5*x)};
    \addplot[name path=B,white] {-4};
    \addplot[name path=D,white] {4};
    \addplot[pattern color=blue2!5,pattern=crosshatch] fill between[of=A and B, soft clip={domain=-3.1:0}];
    \addplot[pattern color=red2!5,pattern=crosshatch] fill between[of=A and D, soft clip={domain=-3.1:0}];
    \addplot[pattern color=blue2!5,pattern=crosshatch] fill between[of=C and D, soft clip={domain=0:3.1}];
    \addplot[pattern color=red2!5,pattern=crosshatch] fill between[of=C and B, soft clip={domain=0:3.1}];
    \node[] (P) at (2.25,0.5){};
    \end{axis}
    \begin{axis}[
        unit vector ratio*=1 1 1,
        grid=none,
        ticks =none,
        axis lines=middle,
        xmin=-3.1,
        xmax=3.1,
        ymin=-3.1,
        ymax=3.1,
        xticklabels={,,},
        yticklabels={,,},
        at={(750,0)},
        axis on top,
    ]
    \node[] (Q) at (-2.25,0.5){};
    \node[] at (2.9,2.9){$\mathbb R^2$};
    \foreach \x in {-1,-0.4,-0.5,0.6,0.8}{
        \edef\temp{\noexpand\addplot[only marks, mark=*, color=blue2] coordinates {({1},{2.5*\x*\x})};}
        \temp
    }
    \foreach \x in {-0.4,-1.4,-1,1.25,0.8,0.6,1.5}{
        \edef\temp{\noexpand\addplot[only marks, mark=*, color=red2] coordinates {({1},{-1.25*\x*\x})};}
        \temp
    }
    \addplot[domain=-3.1:3.1,samples=2,color=black,thick,dashed, name path = H] {x/5};
    \draw[->,thick,gray,dashed]{} (0,0)--(-0.4,2);
    \node[gray] at (-0.5,1) {$w$};
    \node[color=red2] at (-2.8, -1) {$F_-$};
    \node[color=blue2] at (-2.8, -0.3) {$F_+$};
    \draw[thick,colourE]{} (1,-3.1)--(1,3.1);
    \node[colourE] at (1.6,-2.9) {$\Phi(E)$};
    \addplot[name path=B,white] {-4};
    \addplot[name path=D,white] {4};
    \addplot[pattern color=blue2!5,pattern=crosshatch] fill between[of=H and D, soft clip={domain=-3.1:3.1}];
    \addplot[pattern color=red2!5,pattern=crosshatch] fill between[of=H and B, soft clip={domain=-3.1:3.1}];
    \end{axis}
    \draw[->,ultra thick,color=orange!90!black] (P) to[bend left] node[midway,above,inner sep=1pt] {\footnotesize{$\begin{aligned}\Phi \colon \mathbb R^2 &\to \mathbb R^2 \\ f &\mapsto \begin{pmatrix} 1 \\ (f)_0(f)_1 \end{pmatrix}\end{aligned}$}} (Q);
\end{tikzpicture}
    \caption{Mapping a linearly inseparable dichotomy in pattern space to a homogeneously linearly separable dichotomy in feature space.}
    \label{fig:nonlinear_transf}
\end{figure}
The number of $\Phi$-separable dichotomies depends in general on $F$ and $\Phi$. However, if $F$ is ``typical'' with respect to $\Phi$ in the following sense, then the number of $\Phi$-separable dichotomies depends on $N$ and $M'$ only.    
\begin{definition}[$\Phi$-general position]\label{def:phi-gen-pos}
    For $M,M',N \in \mathbb N$, let $\Phi \colon E \to \mathbb R^{M'}$, where $E \subseteq \mathbb R^M$.
    The set $F \coloneqq \{f_1, \ldots, f_N\} \subseteq E$  is said to be in \emph{$\Phi$-general position} if every subset of $k$ elements of $\left\{\Phi(f_1), \ldots, \Phi(f_N)\right\} \subseteq \mathbb R^{M'}$ is linearly independent for all $k \leq \min\{M',N\}$. If this holds for $\Phi = \mathrm{Id} \colon E \to \mathbb R^M, f \mapsto f$, we simply say that $F$ is in \emph{general position}.  
\end{definition}
We are now ready to state the central result in function-counting theory, which provides a closed-form solution for the number of $\Phi$-separable dichotomies of $F$ under the assumption that $F$ is in $\Phi$-general position. Note, however, that as $E$ is an arbitrary subset of $\mathbb R^M$ and $\Phi$ can be any map $E \to \mathbb R^{M'}$, the $\Phi$-general position assumption for $F$ usually does not hold. For instance, if $E$ is as in \cref{fig:nonlinear_transf}, i.e., a union of two linear subspaces, and $\Phi = \mathrm{Id}$, every $N$-point set of $E$ is not in $\Phi$-general position. In \Cref{sec:sep-cap}, we extend the notion of $\Phi$-general position to make it applicable in the general case where both $E\subseteq \mathbb R^M$ and $\Phi$ are arbitrary. 
\begin{theorem}[Function-counting theorem, \cite{cover1965geometrical}]\label{thm:fct-counting}
     Fix $M,M',N \in \mathbb N$, and consider the set $F \coloneqq \{f_1,\ldots,f_N\} \subseteq E$, where $E \subseteq \mathbb R^M$. Furthermore, let $\Phi \colon E \to \mathbb R^{M'}$. The number of $\Phi$-separable dichotomies of $N$ points in $\Phi$-general position in $\mathbb R^M$ is
    \begin{align}\label{eq:counting-fct-cover}
        C(N,M') \coloneqq 2\sum_{k=0}^{M'-1} \binom{N-1}{k}.
    \end{align}
\end{theorem}
If $F$ is not in $\Phi$-general position, there are fewer $\Phi$-separable dichotomies of $F$ (see, e.g., \cite{mitchison1989bounds}), but determining the exact number becomes more challenging. To deal with this problem, consider a statement that is dual to \cref{def:phi-sep}. Concretely, we associate to each $f \in F$ the $(M'-1)$-dimensional hyperplane $\{\Phi(f)\}^\perp$, with goal of determining the number of regions into which these $N$ hyperplanes divide the space $\mathbb R^{M'}$. \cref{fig:partitions-by-hyperplanes} illustrates this for the case $\Phi = \mathrm{Id}$. Before the solution to this problem can be stated, let us introduce the following notion.
\begin{figure}
    \centering
    \definecolor{blue2}{RGB}{20,99,178}
\definecolor{green2}{RGB}{0,95,1}
\colorlet{red2}{red!70!black}
\begin{tikzpicture}[scale=1, every node/.style={scale=1}]
    \begin{axis}[
        unit vector ratio*=1 1 1,
        grid=none,
        ticks =none,
        axis lines=middle,
        xmin=-3.1,
        xmax=3.1,
        ymin=-3.1,
        ymax=3.1,
        xticklabels={,,},
        yticklabels={,,},
    ]
    \addplot [only marks, mark=*, color=green2, nodes near coords, nodes near coords align={east}, point meta=explicit symbolic,] coordinates {
        (-2, -1.5) [$f_3$]
    };
    \draw[densely dotted,thick,green2!50!white,->,shorten >=2pt]{} (0,0)--(-2,-1.5);
    \draw[thick,green2]{} (-3,4)--(3,-4);
    \addplot [only marks, mark=*, color=blue2, nodes near coords, nodes near coords align={west}, point meta=explicit symbolic,] coordinates {
        (2.5, 1) [$f_1$]
    };
    \draw[densely dotted,thick,blue2!50!white,->,shorten >=2pt]{} (0,0)--(2.5,1);
    \draw[thick,blue2]{} (4,-10)--(-4,10);
    \addplot [only marks, mark=*, color=red2, nodes near coords, nodes near coords align={east}, point meta=explicit symbolic,] coordinates {
        (-2, 1) [$f_2$]
    };
    \draw[densely dotted,thick,red2!50!white,->,shorten >=2pt]{} (0,0)--(-2,1);
    \draw[thick,red2]{} (-3,-6)--(3,6);
    \node[red2] at (-1.8,-2.93) {$H_2$};
    \node[blue2] at (0.8,-2.93) {$H_1$};
    \node[green2] at (2.6,-2.93) {$H_2$};
    \node[] at (2.87,2.92){$\mathbb R^2$};
    \fill[gray, opacity=0.2, pattern=north east lines] (0,0) -- (6,12) -- (6,-8) -- cycle;
    \node[gray] at (1.9,-1) {\textbf{$1$}};
    \fill[pattern color = gray, opacity=1, pattern=dots] (0,0) -- (6,12) -- (-4,10) -- cycle;
    \node[gray] at (0.5,2) {\textbf{$2$}};
    \fill[gray, opacity=0.2, pattern=horizontal lines] (0,0) -- (-6,8) -- (-4,10) -- cycle;
    \node[gray] at (-1.2,2) {\textbf{$3$}};
    \fill[pattern color = gray, opacity=0.2, pattern=north west lines] (0,0) -- (-6,8) -- (-6,-12) -- cycle;
    \node[gray] at (-2,-0.25) {\textbf{$4$}};
    \fill[gray, opacity=0.2, pattern=crosshatch] (0,0) -- (4,-10) -- (-6,-12) -- cycle;
    \node[gray] at (-0.2,-2.1) {\textbf{$5$}};
    \fill[pattern color = gray, opacity=0.2, pattern=vertical lines] (0,0) -- (4,-10) -- (6,-8) -- cycle;
    \node[gray] at (1.25,-2.1) {\textbf{$6$}};
    \end{axis}
\end{tikzpicture}
    \caption{Regions into which the $1$-dimensional hyperplanes $H_1$, $H_2$, and $H_3$ divide $\mathbb R^2$, where $H_k = \{f_k\}^\perp$, $k \in \{1,2,3\}$. Since the set $\{f_1,f_2,f_3\}$ is in general position, the number of such regions is given by $C(3,2)=6$.}
    \label{fig:partitions-by-hyperplanes}
\end{figure}
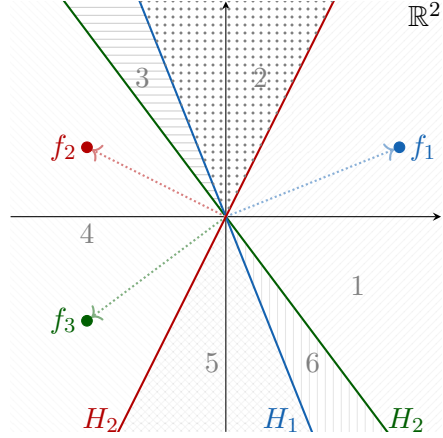
\begin{definition}[\cite{winder1966partitions}]\label{def:even-odd-deg}
    Fix $M' \in \mathbb N$, let $\mathcal K$ be a finite index set, and consider a set of $(M'-1)$-dimensional hyperplanes in $\mathbb R^{M'}$, denoted by $\{H_k\}_{k \in \mathcal K}$, i.e., $H_k \coloneqq \{\varphi_k\}^\perp$, $k \in \mathcal{K}$, where $\{\varphi_k\}_{k \in \mathcal K} \subseteq \mathbb R^{M'}\setminus\{0\}$. The set $\{H_k\}_{k \in \mathcal K}$ is said to be 
    \begin{enumerate}[label=(\roman*)]
        \item \emph{even-degenerate} if $\mathcal K = \emptyset$ or
        \begin{align*}
            \dimR{\bigcap_{k \in \mathcal K} H_k} \equiv M'- \lvert \mathcal K \rvert \Mod{2},
        \end{align*}
        \item \emph{odd-degenerate} if $\mathcal K \neq \emptyset$ and 
        \begin{align*}
            \dimR{\bigcap_{k \in \mathcal K} H_k} \not\equiv M'- \lvert \mathcal K \rvert \Mod{2}.
        \end{align*}
    \end{enumerate}
\end{definition}
We are now ready to present the solution to the general problem of counting the number of $\Phi$-separable dichotomies. We emphasize that here $F$ need not be in $\Phi$-general position.
\begin{theorem}[\cite{winder1966partitions}]\label{thm:winder}
    Fix $M',N \in \mathbb N$, and consider $N$ hyperplanes in $\mathbb R^{M'}$ each of dimension $M'-1$, denoted by $\{H_k\}_{k=1}^N$, i.e., $H_k \coloneqq \{\varphi_k\}^\perp$, $k \in \{1,\ldots,N\}$, where $\{\varphi_k\}_{k =1}^N \subseteq \mathbb R^{M'}\setminus\{0\}$.  
    The number of regions into which the hyperplanes $\{H_k\}_{k=1}^N$ divide $\mathbb R^{M'}$ is given by
    \begin{align*}
        \lvert \mathcal{E} \rvert - \lvert \mathcal{O}\rvert, 
    \end{align*}
    where
    \begin{align*}
        \mathcal{E} &\coloneqq \left\{\{H_k\}_{k \in \mathcal{K}} \colon \mathcal{K} \subseteq \{1,\ldots,N\} \text{ and } \{H_k\}_{k \in \mathcal{K}} \text{ is even-degenerate}\right\}
        \intertext{and}
        \mathcal{O} &\coloneqq \left\{\{H_k\}_{k \in \mathcal{K}} \colon \mathcal{K} \subseteq \{1,\ldots,N\} \text{ and } \{H_k\}_{k \in \mathcal{K}} \text{ is odd-degenerate}\right\}.
    \end{align*}
\end{theorem}
\begin{remark}\label{rem:number-even-odd-deg}
    Note that $\mathcal{E} \cup \mathcal{O}$ is the power set of $\{H_k\}_{k=1}^N$, and hence $\lvert \mathcal{E} \rvert + \lvert \mathcal{O}\rvert = 2^N$. In particular, $\emptyset \in \mathcal{E}$, i.e., the empty set is even-degenerate, see \cref{def:even-odd-deg}.
\end{remark}
As already indicated above, the classical function-counting theorem (\cref{thm:fct-counting}) can be deduced from \cref{thm:winder}. Indeed, we have the following remark.
\begin{remark}\label{rem:hyperplanes-gen-pos}
    Let $F \coloneqq\{f_1,\ldots,f_N\} \subseteq E \subseteq \mathbb R^M$ be in $\Phi$-general position, where $\Phi \colon E \to \mathbb R^{M'}$. Consider the associated $(M'-1)$-dimensional hyperplanes $H_k \coloneqq\{\Phi(f_k)\}^\perp$, $k \in \{1,\ldots,N\}$. To determine the number of regions into which these hyperplanes divide $\mathbb R^{M'}$ using \cref{thm:winder}, we need to count the sets of even- and odd-degenerate hyperplanes. By the assumption that $F$ is in $\Phi$-general position, all sets $\{H_k\}_{k \in \mathcal{K}}$ with $\mathcal K \subseteq \{1,\ldots,N\}$ and $0 \leq \lvert \mathcal K \rvert \leq M'$ are even-degenerate. Indeed, we have 
    \begin{align*}
        \bigcap_{k \in \mathcal K} H_k = \left(\sum_{k \in \mathcal K} \spanR{\{\Phi(f_k)\}}\right)^\perp = \left(\spanR{\{\Phi(f_k)\}_{k \in \mathcal K}}\right)^\perp,
    \end{align*}
    and hence $\dimR{\bigcap_{k \in \mathcal K} H_k} = M'- \lvert \mathcal K \rvert$, where we used that, by assumption, the set $\{\Phi(f_k)\}_{k \in \mathcal K}$ is linearly independent whenever $\lvert \mathcal K \rvert \leq M'$.
    For $M'+1\leq \lvert \mathcal{K} \rvert \leq N$, it holds that $\dimR{\bigcap_{k \in \mathcal K} H_k} = M'$. Consequently, the degeneracy of $\{H_k\}_{k \in \mathcal{K}}$ alternates with increasing $\lvert \mathcal{K} \rvert$ whenever $M'+1\leq \lvert \mathcal{K} \rvert \leq N$.      
    Thus, by \cref{thm:winder}, the number of regions is given by
    \begin{align}\label{eq:fct-count-hyperplanes}
        \binom{N}{0}+ \binom{N}{1} + \cdots + \binom{N}{M'} - \binom{N}{M'+1} + \binom{N}{M'+2} - \cdots \pm \binom{N}{N},
    \end{align}
    where the last term is positive if $M'-N$ is even and negative otherwise. As shown in \cite{winder1966partitions}, one can use the identity $\sum_{t=0}^N\binom{N}{t}(-1)^{M'+1+t} = 0$,
    a consequence of the binomial theorem, and the recurrence relation of binomial coefficients (Pascal's rule) to deduce that \cref{eq:fct-count-hyperplanes} is equal to $C(N,M')$. 
\end{remark}

\section{Hausdorff measure} \label[appendix]{app:Hausdorff}
In this section, we review the definition of the Hausdorff measure and recall some of its basic properties.
\begin{definition}[Hausdorff measure, \cite{simon2014introduction,federer2014geometric}]
Let $(X,\rho)$ be a metric space, $A\subseteq X$, and $0<\delta\leq \infty$. A collection of subsets $\mathcal{C} = \{C_i\}_{i \in \mathbb N}$ of $X$ is called a \emph{$\delta$-cover} of $A$ if $A \subseteq \bigcup_{i \in \mathbb N} C_i$ and if $\mathrm{diam}_\rho(C_i)\leq \delta$, for all $i \in \mathbb N$. Here, $\mathrm{diam}_\rho(C_i) \coloneqq \sup\{\rho(x,y) \colon x,y \in C_i\}$ denotes the diameter of the set $C_i$.   
For $s\geq 0$, $0<\delta\leq \infty$, and $A\subseteq X$, define 
    \begin{align*}
        \mathcal{H}^s_\delta(A) \coloneqq \inf \left\{\sum_{C \in \mathcal{C}} \alpha_s \left(\frac 12 \mathrm{diam}_{\rho}(C)\right)^s \colon \mathcal{C}\text{ is a }\delta\text{-cover of }A\right\},
    \end{align*}
    where $\alpha_s \coloneqq \frac{\pi^{s/2}}{\Gamma(s/2+1)}$ with $\Gamma(t) \coloneqq \int_0^\infty x^{t-1} e^{-x}\,\mathrm{d}x$, $t>0$, being the gamma function.
    The \emph{$s$-dimensional Hausdorff measure} of a set $A \subseteq X$ is defined to be
    \begin{align*}
        \mathcal{H}^s(A) \coloneqq \lim_{\delta \to 0^+} \mathcal H_\delta^s(A).
    \end{align*}
\end{definition}
\begin{proposition}[Properties of Hausdorff measure, \cite{simon2014introduction,federer2014geometric}]\label{prop:hausdorff-meas}
    Let $X$ be a metric space, and consider the $s$-dimensional Hausdorff measure $\mathcal{H}^s$ on $X$, where $s\geq 0$.
    \begin{enumerate}[label=(\roman*)]
        \item For every $s\geq 0$, $\mathcal{H}^s$ is a Borel-regular outer measure.
        \item On $X=\mathbb R^n$, the $n$-dimensional Hausdorff measure $\mathcal{H}^n$ coincides with the $n$-dimensional Lebesgue measure $\mathcal{L}^n$, $n \in \mathbb N$, i.e., 
        \begin{align*}
            \mathcal{H}^n=\mathcal{L}^n.
        \end{align*}
        \item If $\mathcal{U} \subseteq \mathbb R^n$ is an $s$-dimensional linear subspace $s \in \{0,\ldots,n\}$, then
        \begin{align*}
            \mathcal{H}^t(\mathcal U) = 0, \quad \text{for all $t>s$.}
        \end{align*}
        \item Let $A \subseteq X$, and let $Y$ be another metric space. If $\varphi \colon A \to Y$ is Lipschitz, then 
        \begin{align*}
            \mathcal{H}^s(\varphi(A)) \leq \mathrm{Lip}(\varphi)^s \mathcal{H}^s(A).
        \end{align*}
        Moreover, if $X = \mathbb R^n$, $n \in \mathbb N$, and if $A$ is $\mathcal{L}^n$-measurable, then $\varphi(A)$ is $\mathcal{H}^n$-measurable. 
    \end{enumerate}
\end{proposition}
\begin{proposition}[Theorem 1.15 in \cite{simon2014introduction}]\label{prop:regularity}
    Let $X$ be a metric space. Suppose $\mu$ is an open $\sigma$-finite Borel-regular measure on $X$. Then
    \begin{align*}
        \mu(A) = \inf\{\mu(U) \colon U \text{ open, }U \supset A\},
    \end{align*}
    for each subset $A \subset X$, and
    \begin{align*}
        \mu(A) = \sup \{\mu(C) \colon C \text{ closed, } C \subset A\},
    \end{align*}
    for each $\mu$-measurable subset $A \subset X$.
\end{proposition}

\bibliography{ref}
\end{document}